\documentclass[12pt]{article}

\usepackage{amsmath,amsfonts,mathrsfs,amsthm,bm} % Math packages
\usepackage{graphicx,psfrag,epsf}
\usepackage{caption, subcaption}
\usepackage{enumerate}
\usepackage{natbib}

\theoremstyle{plain}
\newtheorem{theorem}{Theorem}[section]
\newtheorem{proposition}[theorem]{Proposition}
\newtheorem{lemma}[theorem]{Lemma}
\newtheorem{corollary}[theorem]{Corollary}
\theoremstyle{definition}

\theoremstyle{remark}
\newtheorem{remark}[theorem]{Remark}

\usepackage{url}
\usepackage{chngcntr}
\usepackage{dsfont }
\usepackage[dvipsnames]{xcolor}
\usepackage[utf8]{inputenc}
\usepackage{relsize}
\usepackage{amssymb }
\usepackage{enumitem}
\usepackage[english]{babel}	
\usepackage{upgreek }
\usepackage{multirow}
\usepackage[toc,page]{appendix}
\usepackage{boldline}
\usepackage{hyperref} 
\usepackage{algorithm}
\usepackage{algorithmic}
\usepackage{microtype}
\usepackage{booktabs}
\usepackage{caption}
\usepackage{mathtools}

\newcommand{\RN}[1]{%
  \textup{\uppercase\expandafter{\romannumeral#1}}%
}

\newcommand{\mc}{\mathcal}
\newcommand{\wt}{\widetilde}
\newcommand{\mr}{\mathrm}

\usepackage{authblk}
\usepackage{lmodern} 

\newcommand{\corrsymbol}{\textasteriskcentered}
\newcommand{\equalsymbol}{\textdagger}

\title{Adaptively Truncated Signature-based Logistic Regression for Semi-parametric Functional Classification}

\author[1,\equalsymbol,\corrsymbol]{Pengcheng Zeng}
\author[1,\equalsymbol]{Siyuan Jiang}
\affil[1]{Institute of Mathematical Sciences, ShanghaiTech University, Shanghai, China}

\begin{document}
\maketitle

\footnotetext[1]{\equalsymbol\ These authors contributed equally.}
\footnotetext[2]{\corrsymbol\ Corresponding author: \texttt{zengpch@shanghaitech.edu.cn}}

\bigskip
\begin{abstract}
Modern sensing technologies routinely generate multi-dimensional functional data, often accompanied by scalar covariates, motivating statistical methods that can jointly accommodate nonlinear functional effects, complex cross-component dependencies, and irregular sampling patterns. Existing functional logistic regression approaches rely on basis expansions whose performance is critically sensitive to the choice of basis, smoothing parameters, and sampling design; meanwhile, current signature-based methods require heuristic selection of the signature truncation order. We propose Adaptively Truncated Signature-based Logistic Regression (ATSLR), a semi-parametric classification framework that integrates path-signature representations with a data-driven procedure for selecting the signature truncation order via penalized empirical risk minimization. Our method is supported by rigorous theoretical guarantees, including the existence of an optimal truncation order, consistency of its adaptive estimator, convergence rates for the classifier risk, a finite and computable search bound, and a general error-propagation analysis for irregular sampling. Experiments on synthetic and real-world datasets demonstrate that our framework consistently outperforms traditional functional classifiers and fixed-order signature baselines in accuracy, robustness, and interpretability. Our results highlight the practical and theoretical value of integrating rough path theory with adaptive model complexity control.
\end{abstract}

\noindent%

\section{Introduction}
Recent advances in sensing technologies have led to an explosion of multi-dimensional functional data, where each observation is a vector-valued function evolving over a continuous domain such as time, space, or frequency. Unlike univariate functional data, these high-dimensional trajectories capture rich interdependencies across dimensions, offering both opportunities and challenges for statistical learning \citep{Rams2005, Horvath2012}. Functional Data Analysis (FDA) provides a rigorous framework for modeling such data in infinite-dimensional spaces \citep{Ferr2006, Sir2016}. However, vector-valued functional data often exhibit complex structures—such as irregular sampling and inter-channel correlations—that demand specialized methodologies beyond conventional FDA tools \citep{Koner2023review}.

A paradigmatic example arises in the analysis of gait dynamics in Parkinson’s disease (PD), where vertical ground reaction forces (VGRFs) are recorded during walking. These signals, collected via force sensors under each foot, form multi-dimensional functional data that encapsulate critical motor control characteristics. However, such data present formidable modeling challenges due to stride-to-stride variability and inter-sensor correlations. Moreover, scalar covariates—such as age, walking speed, and clinical scores—are often available and play an essential role in predictive tasks, including disease classification and progression modeling. These complexities necessitate a modeling framework that can simultaneously incorporate functional and scalar predictors while being robust to noise and irregular sampling.

A standard statistical approach for handling such data is the functional logistic regression model \citep{Reiss2017sofr, Gerth2024review}. Given a binary outcome $y_i \in \{0,1\}$, a $d$-dimensional functional predictor $\bm{X}_i(t) = (X_i^1(t), \ldots, X_i^d(t))^\top$ defined on $[0, T]$, and a $q$-dimensional vector of scalar covariates $\bm{z}_i \in \mathbb{R}^q$, the model assumes
\begin{equation}\label{sig0}
\mathbb{P}(y_i = 1 \mid \bm{X}_i, \bm{z}_i) = \sigma\left( \alpha + \sum_{j=1}^{d} \int_0^T X_i^j(t) \beta_j(t) \, dt + \bm{z}_i^\top \bm{\gamma} \right),
\end{equation}
where $\sigma(u) = (1 + e^{-u})^{-1}$ denotes the logistic link function, 
the functional predictors $X_i^j(t)$ and coefficients $\beta_j(t)$ are 
expanded in a common basis $\{\phi_k(t)\}_{k=1}^K$ as:
\[
X_i^j(t) = \sum_{k=1}^K b^j_{ik} \phi_k(t), \quad 
\beta_j(t) = \sum_{k=1}^K c_{jk} \phi_k(t),
\]
and $\bm{\gamma}$ represents the vector of scalar coefficients.
Model~\eqref{sig0} extends the classical scalar-on-function regression framework~\citep{Ramsay1991tools}, but suffers from several important limitations. First, it imposes a \emph{linearity assumption} between each functional predictor $X_i^j(t)$ and the log-odds of the response $y_i$, which may lead to model misspecification when the true relationship is nonlinear. While generalized additive and index models~\citep{Muller2005GFLM, Mathew2014FGAM, Fan2015FAR} relax this assumption, they often inherit the next two issues. Second, the model exhibits \emph{basis selection sensitivity}: its reliance on basis expansions makes it vulnerable to mismatched basis choices (e.g., Fourier for non-periodic signals or poorly placed B-spline knots), and can obscure interpretability of coefficient functions $\beta_j(t)$~\citep{Rice1991fpca}. Third, it is \emph{fragile to irregular sampling}, a common feature in functional data. Basis projections assume complete, uniformly sampled trajectories, and their violation leads to biased coefficient estimates (\emph{representation bias}) and potential over-smoothing when interpolation is applied (\emph{imputation dependence}). Finally, Model~\eqref{sig0} \emph{neglects cross-component correlations} by modeling each functional input additively and independently. Although methods such as multivariate FPCA~\citep{Jeng2014MFPCA} aim to address this, they rely on joint decomposition and assume perfect temporal alignment, which may be impractical in real-world settings.

To overcome these challenges, we propose a semi-parametric model that replaces the linear term in model~\eqref{sig0} with a nonlinear transformation of the functional predictor. Specifically, we assume the existence of a smooth function $F$ such that
\begin{equation}\label{sig1}
\mathbb{P}(y_i = 1 \mid \bm{X}_i, \bm{z}_i) = \sigma\left(F(\bm{X}_i) + \bm{z}_i^\top \bm{\gamma} \right).
\end{equation}
To model $F(\bm{X}_i)$, we treat the time-augmented signal $\widetilde{\bm{X}}_i = (\bm{X}_i, t)$ as a continuous path and apply the theory of \emph{path signatures} \citep{chen1957integration, lyons1998differential, friz2010multidimensional}. The path signature is a sequence of algebraic features capturing the geometry of the path. Truncating at order $p$, we approximate $F(\widetilde{\bm{X}}_i)$ as
\begin{equation}\label{sig1_2}
F(\widetilde{\bm{X}}_i) \approx S_p(\widetilde{\bm{X}}_i)^\top \bm{\beta}_p,
\end{equation}
where $S_p(\cdot)$ denotes the truncated path signature and $\bm{\beta}_p$ is the associated coefficient vector. This representation enjoys several desirable properties. Unlike classical
functional logistic regression, it does not require basis expansions or
functional smoothing to construct predictors, thereby avoiding sensitivity to
basis selection and tuning parameters. Moreover, path signatures provide a
coordinate-free representation of the entire trajectory that naturally captures
nonlinear interactions among multiple functional components through iterated
integrals. Finally, because the signature transform is stable with respect to
perturbations of the underlying path in the rough-path topology, it provides a
principled foundation for learning from irregularly sampled trajectories after
simple linear interpolation, avoiding the structural assumptions imposed by
global smoothing procedures. These properties make the proposed representation
particularly attractive for modern multi-dimensional functional data, where
complex temporal dependence, heterogeneous sampling schemes, and missing
observations are commonplace.

The principal methodological challenge, however, lies in determining the
signature truncation order $p$. Although the signature of a path forms an
infinite sequence that uniquely characterizes the underlying trajectory under
mild conditions, practical implementation necessarily relies on a finite
truncation. The truncation order governs the complexity of the resulting model:
small values may fail to capture important nonlinear interactions, whereas large
values substantially increase the dimensionality of the feature space and the
risk of overfitting. Existing signature-based methods typically fix $p$
heuristically 
\citep{Yang2015Chinese,Yang2016DeepWriterID,Lai2017Online,Liu2017PSLSTM,Arribas2018Signature},
while recent deep-learning approaches incorporate truncation implicitly through
learnable augmentations \citep{Bonnier2019deep}. However, these approaches do
not address several fundamental statistical questions:
Does an optimal truncation order exist?
Can it be consistently estimated from finite samples?
Can its search be restricted to a finite, theoretically justified range?
To our knowledge, these questions remain largely open for signature-based
classification models. Although \citet{ferm2022linear} studied truncation in a
linear signature regression model, their framework neither considers logistic
classification nor accommodates scalar covariates.

Motivated by these questions, we propose
adaptively truncated signature-based logistic regression, a
semi-parametric classification framework that combines path signatures with a
data-driven model complexity selection procedure. Rather than treating the
signature order as a fixed hyperparameter, our framework selects the truncation order by
penalized empirical risk minimization, allowing model complexity to adapt
automatically to the observed data while jointly incorporating scalar
covariates. This adaptive formulation preserves the expressive power of path
signatures while providing statistical guarantees that are unavailable for
existing heuristic procedures.

The contributions of this paper are fourfold.

\begin{enumerate}[leftmargin=*,label=(\roman*)]

\item
We introduce a unified semi-parametric logistic regression framework for
multi-dimensional functional data that combines path-signature representations
with scalar covariates and adaptive complexity selection, eliminating the need
for basis expansions while retaining the ability to model nonlinear dependence
among functional components.

\item
We develop a data-driven procedure for selecting the signature truncation order
via penalized empirical risk minimization. The proposed criterion is supported
by rigorous statistical theory establishing the existence of an optimal
truncation order, consistency of its estimator, convergence guarantees for the
resulting classifier, and a finite search region that substantially reduces
computational complexity.

\item
We establish a general theoretical framework for irregularly sampled functional
data. Specifically, we show that the effect of reconstruction on classification
risk factors into two components: the stability of the feature representation
and the trajectory reconstruction error. For path signatures, reconstruction by
piecewise linear interpolation introduces an error that vanishes as the sampling
grid is refined, whereas basis-expansion methods generally retain an
irreducible approximation bias and, for adaptive bases such as FPCA, additional
basis-estimation variability.

\item
Extensive simulation studies and analyses of real-world functional datasets
demonstrate that our framework consistently outperforms existing functional
classification methods, particularly under irregular sampling and missing-data
settings, while simultaneously providing an adaptive and theoretically justified
choice of signature complexity.

\end{enumerate}

The remainder of the paper is organized as follows.
Section~\ref{sec:path} reviews the mathematical foundations of path signatures.
Section~\ref{sec:meth} presents the proposed methodology together with adaptive order selection procedure, its
theoretical properties, and robustness
analysis under irregular sampling.
Section~\ref{sec:expe} reports simulation studies and real-data applications.
Section~\ref{sec:disc} concludes with a discussion and future research
directions.

\section{A Brief Overview of Path Signatures}
\label{sec:path}

Path signatures provide a powerful and mathematically rigorous representation for modeling vector-valued functional data. Rooted in Chen’s seminal work on iterated integrals \citep{chen1957integration}, and further developed through rough path theory \citep{lyons1998differential, friz2010multidimensional}, the signature of a path captures essential geometric and temporal features of time-indexed trajectories. 

Let $\bm{X}\colon [0,T] \to \mathbb{R}^d$ denote a path of bounded variation, defined by the total variation norm:
\[
\|\bm{X}\|_{\mathrm{TV}} = \sup_{\mathcal{P}} \sum_{i=0}^{n-1} \|\bm{X}_{t_{i+1}} - \bm{X}_{t_i}\| < \infty,
\]
where the supremum is taken over all partitions $\mathcal{P} = \{0 = t_0 < t_1 < \cdots < t_n = T\}$ of $[0,T]$, and $\|\cdot\|$ denotes the Euclidean norm. We denote by $BV(\mathbb{R}^d)$ the space of $\mathbb{R}^d$-valued paths of bounded variation. This regularity condition guarantees the existence of iterated integrals, which constitute the core of the signature transform.

\paragraph{Definition.} For a multi-index $I = (i_1, \dots, i_k) \in \{1, \dots, d\}^k$, the $k$-th order signature term is given by:
\[
S^I(\bm{X}) = \int_{0 < t_1 < \cdots < t_k < T} dX_{t_1}^{i_1} \cdots dX_{t_k}^{i_k}.
\]
The full (infinite) signature of $\bm{X}$ is the sequence:
\[
S(\bm{X}) = \left(1, S^{(i)}(\bm{X}), S^{(i,j)}(\bm{X}), S^{(i,j,k)}(\bm{X}), \dots \right)_{i,j,k,\dots \in \{1,\dots,d\}}.
\]
We define the \emph{truncated signature} of order $p$ as the vector:
\[
S_p(\bm{X}) = \left(S^I(\bm{X}) \colon |I| \leq p \right),
\]
which contains all terms of order up to $p$. The dimension of the truncated signature is:
\[
s_d(p) = \sum_{k=0}^{p} d^k = \frac{d^{p+1} - 1}{d - 1} \quad \text{for } d \geq 2, \qquad s_1(p) = p+1.
\]
Hence, $S_p(\bm{X}) \in \mathbb{R}^{s_d(p)}$ grows exponentially with $p$ and polynomially with $d$. For instance, when $d = 3$ and $p = 4$, we obtain $s_3(4) = 121$ features.

\paragraph{Key Properties.} The signature transform possesses several desirable properties for learning from multi-dimensional functional data:

\begin{itemize}[leftmargin=*, itemsep=4pt]

\item \textbf{Geometric Interpretability.} Signature terms generalize classical moment-based features: (i) The first-order term $S^{(i)}(\bm{X}) = X_T^i - X_0^i$ corresponds to the net displacement along coordinate $i$. (ii) The second-order term $S^{(i,j)}(\bm{X}) = \int_{0<u<v<T} dX_u^i dX_v^j$ captures pairwise curvature, and the antisymmetric part
\[
A^{(i,j)} = S^{(i,j)}(\bm{X}) - S^{(j,i)}(\bm{X})
\]
approximates the signed area enclosed between the $i$-th and $j$-th coordinates.
(iii) Higher-order terms capture intricate interactions and directional geometry of the path \citep{chevyrev2016primer}.

\item \textbf{Uniqueness.} If $\bm{X}$ has one strictly monotonic coordinate, then the full signature $S(\bm{X})$ determines the path uniquely up to translation and reparametrization in time \citep{hambly2010uniqueness}. This property enables faithful path representation in statistical modeling.

\item \textbf{Linearity under Concatenation.} Let $\bm{X}_1$ and $\bm{X}_2$ be two paths concatenated in time. Then, the signature satisfies Chen’s identity:
\begin{equation}\label{eq:concat}
S(\bm{X}_1 \ast \bm{X}_2) = S(\bm{X}_1) \otimes S(\bm{X}_2),
\end{equation}
where $\otimes$ denotes the tensor (shuffle) product. This recursive structure facilitates efficient signature computation.

\item \textbf{Universality.} Let $F\colon \mathcal{X} \to \mathbb{R}$ be a continuous function defined on a compact subset $\mathcal{X} \subset BV(\mathbb{R}^{d-1})$. If each path $\bm{X}$ is time-augmented as $\widetilde{\bm{X}}(t) = (\bm{X}(t), t)$ and has fixed initial value, then for any $\varepsilon > 0$, there exists $p^* \in \mathbb{N}$ and a coefficient vector $\bm{\beta}^*_{p^*} \in \mathbb{R}^{s_d(p^*)}$ such that:
\[
\left|F(\bm{X}) - \left\langle \bm{\beta}^*_{p^*}, S_{p^*}(\widetilde{\bm{X}}) \right\rangle \right| < \varepsilon \quad \text{for all } \bm{X} \in \mathcal{X},
\]
establishing a Stone--Weierstrass-type approximation theorem for path signatures \citep{Levin:2016, ferm2022linear}.
\end{itemize}

These theoretical properties underlie our proposed methodology and make the signature transform particularly well-suited to learning tasks involving multi-dimensional functional data. For rigorous mathematical treatment (including proofs) of path signatures, we refer the reader to \cite{lyons2007rough} and \cite{friz2010multidimensional}.

\section{The Methodology}
\label{sec:meth}

\subsection{The Model}
\label{sec:mod}

Let $\{(\bm{X}_i, \bm{z}_i, y_i)\}_{i=1}^{n}$ be independent and identically distributed samples from a joint distribution over $(\bm{X}, \bm{z}, y)$, where $\bm{X}:[0,T]\to\mathbb{R}^{d-1}$ is a $(d-1)$-dimensional functional covariate ($d\geq2$), $\bm{z}\in\mathbb{R}^q$ is a vector of scalar covariates, and $y\in\{0,1\}$ is a binary response. We assume throughout that $\bm{X}\in BV(\mathbb{R}^{d-1})$ with fixed initial value. The fixed initial value condition is imposed solely to ensure identifiability and to facilitate the application of signature approximation theory (the same assumption is also used in \cite{ferm2022linear}); it is not a modeling restriction and entails no loss of generality. In practice, this assumption is mild and is typically satisfied after standard preprocessing, e.g., centering or baseline alignment of functional data such as gait signals or EEG/ECG traces.

Our objective is to construct a flexible classification framework that jointly incorporates the functional trajectory and scalar covariates, while circumventing the restrictive linearity and basis-expansion assumptions inherent in classical functional logistic regression. To that end, we posit that the conditional log-odds function admits a semi-parametric additive structure of the form
\begin{equation}
\label{eq:logit1}
\mathrm{Logit}\bigl(\mathbb{P}(y=1\mid \bm{X},\bm{z})\bigr) = F(\bm{X}) + \bm{z}^\top \bm{\gamma},
\end{equation}
where $F:C([0,T];\mathbb{R}^{d-1})\to\mathbb{R}$ is an unknown continuous nonlinear mapping defined on the space of multivariate trajectories, and $\bm{\gamma}\in\mathbb{R}^q$ is a finite-dimensional parameter vector. This formulation is particularly well suited to settings in which scalar predictors exert linear effects while the functional component contributes a nonparametric, yet continuous, influence. Such assumptions are common in biomedical applications---e.g., gait analysis, EEG/ECG trajectories, and longitudinal biomarkers---where small perturbations in $\bm{X}$ are expected to induce correspondingly small changes in outcome probabilities. Model~\eqref{eq:logit1} substantially enriches the class of admissible functional relationships relative to classical functional logistic regression. The remaining challenge is to construct a statistically tractable representation of the unknown functional $F$.

To this end, we augment time and define a $d$-dimensional path $\widetilde{\bm{X}} = (\bm{X}, t)$. This augmentation removes ambiguity due to different path parameterizations and explicitly incorporates temporal information into the signature. By the universality property of path signatures (see \textbf{Universality} in Section~\ref{sec:path}), every continuous functional on compact sets of bounded-variation paths can be approximated arbitrarily well by a linear functional of the truncated signature. Consequently, for any finite truncation order $p$, we may approximate
\begin{equation}
\label{eq:F}
F(\bm{X}) \approx S_p(\widetilde{\bm{X}})^\top \bm{\beta}_p,
\end{equation}
where $S_p(\widetilde{\bm{X}})\in\mathbb{R}^{s_d(p)}$ denotes the truncated signature of order $p$, $\bm{\beta}_p\in\mathbb{R}^{s_d(p)}$ is the corresponding coefficient vector, and $
s_d(p) = \sum_{k=0}^{p} d^k
$ (see \textbf{Definition} in Section~\ref{sec:path})
is the dimension after time augmentation. Notably, the signature transform of $\widetilde{\bm{X}}$ not only captures temporal variation but also uniquely determines the path~\citep{hambly2010uniqueness}, further justifying the time-augmentation.

For a fixed truncation order $p$, let the augmented design vector be $\widetilde{\bm{S}}_p = (S_p(\widetilde{\bm{X}})^\top, \bm{z}^\top)^\top \in \mathbb{R}^{s_d(p)+q}$ and the parameter vector be $\bm{\theta}_p = (\bm{\beta}_p^\top, \bm{\gamma}^\top)^\top$. Model~\eqref{eq:logit1} then reduces to a classical generalized linear model:
\begin{equation}
\label{eq:logit2}
\mathrm{Logit}\bigl(\mathbb{P}(y=1\mid \bm{X},\bm{z})\bigr) = \widetilde{\bm{S}}_p^\top \bm{\theta}_p.
\end{equation}
This formulation automatically includes an intercept term, since the zeroth-order signature component is identically equal to $1$. When $p=0$, the model reduces to standard logistic regression on scalar covariates only. Hence, the statistical problem shifts from estimating an infinite-dimensional functional $F$ to determining an appropriate model complexity through the choice of the truncation order $p$.

The truncation order governs the expressive power of the model. A value of $p$ that is too small precludes representation of higher-order interactions among the functional components, leading to approximation bias. Conversely, increasing $p$ rapidly inflates the signature dimension, $s_d(p)=\mathcal{O}(d^p)$---for instance, $s_4(5)=1365$ and $s_4(6)=5461$---thereby increasing estimation variance, computational cost, and the risk of overfitting. Existing signature-based methods typically treat $p$ as a tuning parameter chosen heuristically or via cross-validation. In contrast, we treat $p$ as an unknown model complexity parameter to be estimated from the data. This perspective recasts signature truncation as a statistical model-selection problem and constitutes the central methodological contribution of the paper.

In the following, we first rigorously establish the existence of a theoretically optimal truncation order $p^*\in\mathbb{N}$ and a corresponding parameter vector $\bm{\theta}_{p^*}^*\in\mathbb{R}^{s_d(p^*)+q}$ that jointly minimize the population risk of the approximated model~\eqref{eq:logit2}, while approaching the risk of the original semi-parametric model~\eqref{eq:logit1}. Building on this foundation, we then propose a principled, data-driven estimator for $p^*$ that balances model complexity and generalization error.

\subsection{Existence of Optimal Truncation Order}
\label{sec:existence}
We introduce the population risk
\[
\mathcal{R}(F, \bm{\gamma}) := \mathbb{E}_{(\bm{X},\bm{z},y)}\left[ \ell\left(y, F(\mathbf{X}) + \mathbf{z}^\top \bm{\gamma} \right) \right],
\]
where $\ell(y,\eta) = -y\eta + \log(1 + e^{\eta})$ denotes the logistic loss. The oracle risk is then defined as $\mathcal{R}^* := \min_{F,\bm{\gamma}} \mathcal{R}(F, \bm{\gamma})$, which is the minimal theoretical risk achievable by the original model in~\eqref{eq:logit1}. As we will show in the next theorem, $\mc{R}^*$ exists under relatively weak conditions. For any fixed truncation order $p$, the theoretical risk of the approximated model~\eqref{eq:logit2} is given by:
\begin{equation}\label{eq:risk}
\mathcal{R}_{p}(\bm{\theta}_p) = \mathbb{E}_{(\bm{X},\bm{z},y)}\left[-y\widetilde{\bm{S}}_p^\top \bm{\theta}_p + \log (1 + e^{\widetilde{\bm{S}}_p^\top \bm{\theta}_p})\right].
\end{equation}
We first establish that the oracle risk of the original semi-parametric model can be approximated arbitrarily well by some truncated signature model.

\begin{theorem}[$\varepsilon$-Approximation Guarantee]\label{thm1}
Suppose the following conditions hold:
\begin{enumerate}[label=(A.\arabic*),leftmargin=*,nosep]
\item There exist constants $C_{F}, C_{\bm{\gamma}} > 0$ such that $\|F\|_{\infty} < C_{F}$ and $\|\bm{\gamma}\|_{1} \leq C_{\bm{\gamma}}$.
\item There exist constants $C_{\bm{X}}, C_{\bm{z}} > 0$ such that $\|\bm{X}\|_{\mathrm{TV}} < C_{\bm{X}}$ and $\|\bm{z}\| < C_{\bm{z}}$ almost surely.
\end{enumerate}
Then the original model in~\eqref{eq:logit1} admits a minimal theoretical risk $\mathcal{R}^*$. Moreover, for any $\varepsilon > 0$, there exists a pair $(p^*, \bm{\theta}_{p^*}^*)$ such that
\begin{equation}\label{eq:thm1}
\left|\mathcal{R}_{p^*}(\bm{\theta}_{p^*}^*) - \mathcal{R}^*\right| < \varepsilon.
\end{equation}
\end{theorem}

\begin{remark}
Assumptions (A.1) and (A.2) are mild and practically motivated. Boundedness in (A.1) ensures the conditional log-odds remain well-behaved and aligns with common regularization practices in statistical learning. In applications such as Parkinson’s disease gait analysis (Section~\ref{sec:app}), the log-odds of disease status are naturally constrained by clinical considerations.
Assumption (A.2) accommodates the irregular, piecewise-smooth nature of real-world functional and scalar data. For instance, vertical ground reaction force (VGRF) signals—collected at high frequency and bounded by biomechanical limits—typically satisfy the total variation bound. Similarly, scalar covariates such as age and gait speed are physiologically constrained, making the boundedness assumption realistic. Overall, both assumptions reflect verifiable conditions in biomedical and engineering contexts involving functional and scalar predictors.
\end{remark}

Theorem~\ref{thm1} guarantees that for any desired precision $\varepsilon$, a finite truncation order $p^*$ and corresponding coefficients $\bm{\theta}_{p^*}^*$ suffice to approximate the oracle risk. However, the pair $(p^*, \bm{\theta}_{p^*}^*)$ needs not correspond to the minimizer of the approximated population risk $\mathcal{R}_p(\bm{\theta}_p)$ for that $p$. To define a meaningful optimal truncation order, we therefore restrict attention to coefficient vectors in the $\ell_1$-ball $B_{p,r} = \{\bm{\theta}_p \in \mathbb{R}^{s_{d}(p)+q} \mid \|\bm{\theta}_p\|_1 \leq r\}$, which corresponds to LASSO-type regularization. Within this constrained family, we seek the minimal truncation order that achieves the infimum of the approximated risk.

\begin{theorem}[Minimal Sufficient Truncation]\label{thm2}
Suppose Assumptions~(A.1)--(A.2) from Theorem~\ref{thm1} hold, along with:
\begin{enumerate}[label=(A.\arabic*),leftmargin=*,nosep,start=3]
\item There exists $r > 0$ such that the minimizer of the approximated risk lies in $B_{p,r}$;
\item $\mathcal{R}^* \leq \inf_{p,\, \bm{\theta}_p \in B_{p,r}} \mathcal{R}_p(\bm{\theta}_p)$.
\end{enumerate}
Then there exists a minimal truncation order $p^* \in \mathbb{N}$ and coefficients $\bm{\theta}_{p^*}^* \in B_{p^*,r}$ such that
\begin{equation}\label{eq:thm2}
\mathcal{R}_{p^*}(\bm{\theta}_{p^*}^*) = \inf_{p,\, \bm{\theta}_p \in B_{p,r}} \mathcal{R}_p(\bm{\theta}_p).
\end{equation}
\end{theorem}

\begin{remark}
Assumptions (A.3) and (A.4) impose natural constraints that promote sparse and well-approximated models. The $\ell_1$-boundedness in (A.3) controls the contribution of high-order signature terms, reflecting empirical sparsity observed in practice—where predictive information is often concentrated in lower-order interactions. Assumption (A.4) ensures that the minimal population risk $\mathcal{R}^*$ provides a valid lower bound for all truncated models. This is justified by the universal approximation capability of path signatures, which enables low-order truncations to achieve near-optimal accuracy. As seen in our empirical results (Figures~\ref{fig:park_coef}, \ref{fig:motion_coef}, \ref{fig:Gait_coef}), effective classification can be achieved with modest truncation orders. Together, these assumptions yield both theoretical tractability and practical relevance.
\end{remark}

The existence results established in Theorems~\ref{thm1} and~\ref{thm2} provide several key guarantees. First, for any desired precision $\varepsilon > 0$, a finite truncation order $p^*$ suffices to approximate the oracle risk. Second, the optimal truncation order adapts to the intrinsic complexity of the functional $F(\bm{X})$, rather than being fixed a priori. Third, no smoothness assumptions on $\bm{X}$ beyond finite variation are required. These theoretical properties explain the empirical success of our framework with rough, multi-dimensional functional data, where classical basis-expansion methods exhibit significant performance degradation (see Section~\ref{sec:expe}). Complete proofs of Theorems~\ref{thm1} and~\ref{thm2} are provided in Appendices~\ref{app:thm1} and~\ref{app:thm2}, respectively.

In all subsequent sections, we adopt the minimal sufficient truncation order $p^*$ as the optimal choice. Under this setting, the oracle version of our model is given by
\begin{equation}\label{eq:logit3}
\mathrm{Logit}\big(\mathbb{P}(y = 1 \mid \bm{X}, \bm{z})\big) = \widetilde{\bm{S}}_{p^*}^\top \bm{\theta}_{p^*}^*.
\end{equation}
We emphasize that the infinite-dimensional functional covariate $\bm{X}$ and the scalar covariate $\bm{z}$ serve as the inputs; the truncation order $p^*$ controls the complexity of the functional mapping $F$ in~\eqref{eq:F}, not the dimensionality of $\bm{X}$.

\subsection{Estimation of the Optimal Truncation Order}
\label{sec:estimation}

We now introduce a data-driven strategy for selecting the optimal signature truncation order $p^*$. Inspired by the penalized empirical risk framework of \citet{ferm2022linear}, we propose to choose $p$ by minimizing a regularized logistic loss over a constrained parameter class. 

For a sample of size $n$, the empirical risk associated with truncation order $p$ and coefficient vector $\bm{\theta}_p$ is defined as
\begin{equation}
\label{eq:risk_emp}
\widehat{\mathcal{R}}_{p,n}(\bm{\theta}_p) = 
\frac{1}{n} \sum_{i=1}^{n} \left[ -y_{i} \, \widetilde{\bm{S}}_p^\top(\widetilde{\bm{X}}_i, \bm{z}_i) \bm{\theta}_p 
+ \log\left(1 + e^{\widetilde{\bm{S}}_p^\top(\widetilde{\bm{X}}_i, \bm{z}_i)\bm{\theta}_p} \right) \right],
\end{equation}
where $\widetilde{\bm{S}}_p(\widetilde{\bm{X}}_i, \bm{z}_i)$ denotes the concatenation of the truncated signature features of $\widetilde{\bm{X}}_i$ and scalar covariates $\bm{z}_i$. We define the regularized empirical risk at order $p$ as
\begin{equation}
\label{eq:emp_minimizer}
\widehat{L}_n(p) \coloneqq \min_{\bm{\theta}_p \in B_{p,r}} \widehat{\mathcal{R}}_{p,n}(\bm{\theta}_p) 
= \widehat{\mathcal{R}}_{p,n}(\widehat{\bm{\theta}}_p),
\end{equation}
where $\widehat{\bm{\theta}}_p$ is the empirical risk minimizer over $B_{p,r}$. The existence and uniqueness of $\widehat{\bm{\theta}}_p$ follow from the strict convexity of $\bm{\theta}_p \mapsto \widehat{\mathcal{R}}_{p,n}(\bm{\theta}_p)$ and the compactness of $B_{p,r}$ (See the proof of Theorem~\ref{thm2} in Appendix \ref{app:thm2}). This formulation corresponds to a Lasso-type logistic regression, where the $\ell_1$-constraint plays the role of implicit regularization. Since the parameter spaces $\{B_{p,r}\}_{p \in \mathbb{N}}$ are nested and increasing in $p$, the sequence of empirical risks $\{\widehat{L}_n(p)\}_{p \in \mathbb{N}}$ is non-increasing. That is, richer function classes induced by higher $p$ yield improved data fit, albeit at the cost of increased variance and overfitting risk.

To balance this trade-off, we introduce a complexity penalty that grows with the model size. The estimated optimal truncation order $\widehat{p}$ is defined as the solution to a penalized empirical risk criterion:
\begin{equation}
\label{eq:order}
\widehat{p} \coloneqq \min\left\{\underset{p \in \mathbb{N}}{\arg\min} \left( \widehat{L}_n(p) + \mathrm{pen}_{n}(p, q) \right)\right\},
\end{equation}
where the penalty function takes the form
\begin{equation}
\label{eq:penalty}
\mathrm{pen}_n(p, q) = \frac{C_{\mathrm{pen}} \, \sqrt{s_d(p)\, e^{q}}}{n^{\rho}}.
\end{equation}
Here, $C_{\mathrm{pen}} > 0$ is a constant controlling the strength of regularization, $s_d(p)$ is the number of path signature terms up to order $p$, $q$ is the dimension of scalar covariates, and $\rho \in (0, \tfrac{1}{2})$ determines the convergence rate.

The penalization term $\sqrt{s_d(p)}$ accounts for the complexity of the functional signature representation, while the multiplicative factor $\sqrt{e^q}$ captures the contribution of the scalar covariates to the overall model class complexity. This is justified by the (worst-case) exponential growth of the Rademacher complexity and covering numbers in high-dimensional feature spaces \citep{bartlett2002rademacher}. The $\ell_1$-constraint mitigates this growth in practice, but the penalty ensures robustness. Our procedure selects the smallest truncation order $\widehat{p}$ that minimizes the penalized criterion in \eqref{eq:order}, thereby ensuring parsimony while achieving near-optimal predictive performance. 

\subsection{Consistency and Risk Convergence}
\label{sec:theory}
Note that the estimated truncation order $\hat{p}$, defined in \eqref{eq:order}, is determined indirectly by the $n$~i.i.d. samples. We now establish non-asymptotic concentration guarantees for the estimated $\hat{p}$ and corresponding risk under mild regularity conditions. Complete proofs of Theorems 3 and 4 are provided in Appendices \ref{app:thm3} and \ref{app:thm4}, respectively.

\begin{theorem}[Order Selection Consistency]\label{thm3}
Under assumptions (A.1)--(A.4) of Theorems~\ref{thm1}--\ref{thm2}, let $n^*$ be the smallest integer satisfying
\begin{align*}
(n^*)^{\tilde{\rho}} \geq 
\left(432\sqrt{\pi}rC + C_{\mathrm{pen}}\sqrt{e^{q}}\right) \left( 
\frac{2\sqrt{s_d(p^*+1)+q}}{L(p^*-1) - \widetilde{\mathcal{R}}^*} + 
\frac{\sqrt{2s_d(p^*+1)+2q}}{C_{\mathrm{pen}}\sqrt{e^{q}d^{p^*+1}}}
\right),
\end{align*}
where $\tilde{\rho} = \min(\rho, 1/2 - \rho)$, $C = 2(C_{\bm{z}} + e^{C_{\bm{X}} + T})$, $L(p) = \min_{\bm{\theta}_p \in B_{p,r}} \mathcal{R}_p(\bm{\theta}_p)$, and $\widetilde{\mathcal{R}}^* \triangleq \mathcal{R}_{p^*}(\bm{\theta}_{p^*}^*)$. Then for all $n \geq n^*$,
\begin{equation}\label{eq:order_concentration}
\mathbb{P}(\hat{p} \neq p^*) \leq c_1 e^{-c_2  n^{1 - 2\rho}},
\end{equation}
where the constants $c_1$ and $c_2$ are given by
\begin{align*}
c_1 &= 74 \sum_{p > 0} e^{-c_0 s_d(p)} + 148 p^*, \quad
c_0 = \frac{C_{\mathrm{pen}}^2 d^{p^* + 1}e^q }{256 r C (36 r C + 1) s_d(p^* + 1)}, \\
c_2 &= \frac{1}{256 r C (36 r C + 1)} \min \left\{ 
\frac{C_{\mathrm{pen}}^2 d^{p^* + 1}e^q}{s_d(p^* + 1)},\,
\frac{(L(p^* - 1) - \widetilde{\mathcal{R}}^*)^2}{4} \right\}.
\end{align*}
\end{theorem}

\begin{remark}
The behavior of the required sample size $n^*$ and constants $c_1,c_2$ reveals several insights as $r$, $q$, $d$, and $p^*$ vary:  
(i) As the parameter space radius $r$ increases, $n^*$, $c_1$ grows, while $c_2$ decreases, reflecting both increased data requirements and degraded estimator quality due to the enlarged space $B_{p,r}$.  
(ii) To ensure exponential convergence of $\hat{p}$ to the true $p^*$, $n^*$ must scale at least as $\mc{O}(e^{q/(2\tilde \rho)})$, due to the exponential increase in model complexity with the number of scalar covariates $q$.  
(iii) As the path dimension $d$ and the optimal truncation order $p^*$ increase, $n^*$ scales as $\mc{O}(d^{p^*/(2 \tilde \rho)})$, since the parameter size grows accordingly. %However, $c_1$ and $c_2$ remain stable when only $d$ increases, indicating that the bound in Eq.~\eqref{eq:order_concentration} remains robust in high dimensions.
\end{remark}

\paragraph{Proof Sketch of Theorem~\ref{thm3}.}
The proof establishes model selection consistency through three key mechanisms. First, the empirical process theory shows that the risk difference $Z_{p,n}(\bm{\theta}_p) = \widehat{\mathcal{R}}_{p,n}(\bm{\theta}_p) - \mathcal{R}_p(\bm{\theta}_p)$ concentrates uniformly over parameter spaces $B_{p,r}$, enabled by the logistic loss's Lipschitz properties and our regularity assumptions. Second, for overparameterized models ($p>p^*$), the growing structural penalty dominates any spurious fitting gains, forcing exponential decay in selection probability with both sample size and model complexity. Conversely, for underparameterized models ($p<p^*$), the fundamental risk gap provides sufficient separation to overcome diminishing penalty differences. Finally, a union bound combines these effects, with the overall error rate governed by the slower-decaying regime and weighted by the cumulative influence of all candidate models. The threshold sample size $n^*$ ensures these concentration effects become active simultaneously.
\vspace{1em}

\begin{theorem}[Risk Convergence]\label{thm4}
Under assumptions (A.1)--(A.4) of Theorems~\ref{thm1}--\ref{thm2}, let $n^*$ be as defined in Theorem~\ref{thm3}. Then for all $n \geq n^*$,
\begin{equation}\label{eq:risk_convergence}
\left| \mathbb{E}\left[ \mathcal{R}_{\widehat{p}}(\widehat{\bm{\theta}}_{\widehat{p}}) \right] - \mathcal{R}^* \right|
\leq \frac{c_3}{\sqrt{n}} + c_4 e^{ - c_2   n^{1 - 2\rho}},
\end{equation}
where the constants $c_3$ and $c_4$ are given by
\begin{equation*}
c_3 = 36 rC \sqrt{\pi} (p^* + 1) \sqrt{s_d(p^*) + q}, \hspace{1mm}
c_4 = rC \Big(2664 \sqrt{\pi} \sum_{p > p^*} \sqrt{s_d(p) + q} \, e^{-c_0  s_d(p)}+c_1\Big) + c_1 \log 2,
\end{equation*}
with constants $c_0$, $c_1$, and $c_2$ defined in Theorem~\ref{thm3}.
\end{theorem}

\begin{remark}
The risk bound in Eq.~\eqref{eq:risk_convergence} achieves the classical $\mathcal{O}(n^{-1/2})$ rate, typical in univariate functional logistic or linear models, but under much weaker assumptions on the functional predictors $\bm{X}$. 
This $n^{-1/2}$ rate is the standard parametric rate and is minimax optimal for fixed-dimensional parametric models, matching, for instance, the convergence of the maximum likelihood estimator in parametric logistic regression. Moreover, the bound consists of an estimation error term decaying as $n^{-1/2}$ (after correct model selection) and an exponentially fast model selection error; hence the overall rate is dominated by $n^{-1/2}$, which is optimal in the minimax sense for finite-dimensional parameters. As the number of scalar covariates $q$ grows, model complexity increases, amplifying estimation variance and overfitting risk—hallmarks of the curse of dimensionality. Consequently, more data and stronger regularization are required to maintain generalization. A large truncation order $p^*$ or increased data variability (i.e., larger $C_{\bm{z}}$ or $C_{\bm{X}}$) further slows convergence by inflating the constants in the bound.
\end{remark}

\paragraph{Proof Sketch of Theorem~\ref{thm4}.}
We decompose the excess risk into two components:
\[
\left| \mathbb{E}\left[ \mathcal{R}_{\widehat{p}}(\widehat{\bm{\theta}}_{\widehat{p}}) \right] - \mathcal{R}^* \right| 
\leq \underbrace{
\left| \mathbb{E}\left[ \mathcal{R}_{\widehat{p}}(\widehat{\bm{\theta}}_{\widehat{p}}) \right] 
- \mathcal{R}_{p^*}(\bm{\theta}_{p^*}^*) \right|
}_{\text{estimation and selection error}}
+ 
\underbrace{
\left| \mathcal{R}_{p^*}(\bm{\theta}_{p^*}^*) - \mathcal{R}^* \right|
}_{\text{approximation error}}.
\]
The second term is controlled via Theorem~\ref{thm1}, which ensures that truncating at $p^*$ yields risk close to the oracle. For the first term, we apply uniform entropy bounds to obtain an $\mathcal{O}(n^{-1/2})$ rate for estimation, and invoke Theorem~\ref{thm3} to ensure that the probability of incorrect order selection decays exponentially in $n$. The overall bound~\eqref{eq:risk_convergence} reflects this trade-off, with constants $c_3$ and $c_4$ capturing the complexity of the signature features and scalar covariates.

\subsection{Finite Search Bound}
We now establish an upper bound on the search range for the data-driven truncation order $\hat{p}$. Theorem \ref{thm:finite_search_bound} guarantees that the optimally selected truncation order $\hat{p}$ lies within a finite, computable range. A complete proof of Theorem \ref{thm:finite_search_bound} is provided in Appendix \ref{app:thm:finite_search_bound}.

\begin{theorem}[Empirical Monotonicity and Finite Search Bound]
\label{thm:finite_search_bound}
Let Assumptions (A.1)--(A.4) of Theorems~\ref{thm1}--\ref{thm2} hold.
For each truncation level $p \ge 1$, denote
\[
L(p) := \inf_{\theta_p \in B_{p,r}} R_p(\theta_p),
\qquad
C_n(p) := \widehat L_n(p) + \operatorname{pen}_n(p,q),
\]
where $\widehat L_n(p)$ is the empirical risk minimizer over $B_{p,r}$ (Eq.~(\ref{eq:emp_minimizer})) and $\operatorname{pen}_n(p,q)$ is the penalty function (Eq.~(\ref{eq:penalty})). Then there exists a finite integer $P$ such that:
\begin{enumerate}[label=(\arabic*),leftmargin=*,nosep,start=1]
\item $L(p) + \operatorname{pen}_n(p,q)$ is strictly increasing for all $p > P$;
\item $C_n(p)$ is strictly increasing for all $p > P$ with probability at least $
1 - \sum_{p>P} \exp\{-9\pi\,s_d(p)\}$;
\item consequently, with the same probability, $\widehat p = \min\left\{\arg\min_{p\ge 1} C_n(p) \right\} \le P$.
\end{enumerate}
Moreover, $P$ may be chosen as
\[
P = \left\lceil
\frac{
\rho \log n \;+\; \log\!\Bigl( \frac{2C_1}{C_{\mathrm{pen}}\sqrt{e^{q}}\,(d^{1/2}-1)} \Bigr)
}{
(\alpha+\tfrac12)\log d
}
\right\rceil,
\]
which is finite and satisfies $P = \mc O\!\left(\frac{\log n}{\log d}\right)$ as $n,d\to\infty$; for fixed $d$, $P = \mc O(\log n)$.
\end{theorem}

\begin{remark}
The explicit expression for $P$ reveals its dependence on the sample size $n$, the path dimension $d$, the number of scalar covariates $q$, and the penalty constant $C_{\mathrm{pen}}$, with $P=\mc O(\log n)$ for fixed $d$. This logarithmic growth ensures that even for very large samples (e.g., $n=10^6$), $P$ remains a moderate constant (e.g., $10$ or $12$ for $d\le 10$). For small $n$, the scaling constants yield $P$ in the range $3$--$5$ (e.g., $\log 100 \approx 4.6$), consistent with empirical selections $\widehat p=3$ or $4$ for sample sizes $24$--$260$ (Section~\ref{sec:app}). Thus, the search in Algorithm~\ref{alg:PSLR} (Section~\ref{sec:impl}) can be safely limited to a conservative bound $P$ (e.g., $P=12$), with the penalized criterion automatically selecting the optimal order within this range—making the method both statistically sound and computationally practical.
\end{remark}

\paragraph{Proof Sketch of Theorem~\ref{thm:finite_search_bound}.}
The proof establishes that the data-selected truncation order $\widehat{p}$ lies within a finite range by showing that the increment of the complexity penalty eventually dominates both the exponential decay of the population truncation risk and the empirical estimation fluctuations. Leveraging the universality of path signatures, the risk reduction between consecutive truncation orders decays at rate $d^{-\alpha p}$, while the penalty term grows as $d^{p/2}$ due to the dimensional scaling of truncated signatures. Combined with a uniform concentration bound on the empirical risk, this ensures that for sufficiently large $p$, the penalty increment fully outweighs both the population risk decrease and empirical noise, rendering the penalized risk strictly increasing. Solving the dominance condition yields an explicit finite bound $P = \mc O(\log n / \log d)$, such that the minimizer $\widehat{p}$ cannot exceed $P$ with high probability.

\subsection{Robustness under Irregular Sampling}
\label{sec:irregular}

Although the theoretical development above assumes fully observed continuous
trajectories, practical functional data are frequently observed on irregular or
incomplete time grids. Consequently, all existing methods must first reconstruct
a continuous trajectory before extracting features. A key observation is that
the statistical error introduced by this reconstruction is \emph{method
dependent}. Existing basis-expansion approaches typically recover trajectories
through global smoothing or basis projection (e.g., Fourier, B-splines or
FPCA), whereas PSLR only employs local piecewise linear interpolation before
computing path signatures (see Section~\ref{sec:impl}).

To understand this difference, let $\widehat {\mathbf X}$ denote an arbitrary
reconstruction of the true trajectory $\mathbf X$, and let $\Phi(\cdot)$ denote the
feature representation used by a classifier. In Appendix~\ref{app:error_propagation} (Theorem~\ref{thm:errorprop}), we establish a general error propagation bound: whenever the feature map $\Phi$ is Lipschitz continuous,
\[
\sup_{\bm\theta\in B_{p,r}}
\left|
\mathcal R(\bm\theta)
-
\widehat{\mathcal R}(\bm\theta)
\right|
\le
rL_\Phi \, \varepsilon(\widehat{\mathbf X}),
\]
where $L_\Phi$ denotes the Lipschitz constant of $\Phi$ and $\varepsilon(\widehat{\mathbf X}) = \mathbb{E}\|\mathbf X-\widehat{\mathbf X}\|_{1\text{-var}}$ is the reconstruction error. Equivalently,
\[
\text{Risk perturbation}
\le
(\text{feature stability } L_\Phi)
\times
(\text{reconstruction error } \varepsilon).
\]

For PSLR, both quantities admit favorable theoretical properties.
First, the signature transform satisfies a sampling-design-independent
stability theorem from rough path theory
\citep{lyons2007rough,friz2010multidimensional}, implying that small perturbations
of the reconstructed path produce proportionally small perturbations of the
signature features. Second, when piecewise linear interpolation is used,
Lemma~\ref{lem:interp} in Appendix~\ref{app:error_propagation} proves that the interpolation error converges to zero in the
$1$-variation norm $\|\cdot\|_{1\text{-var}}$ (i.e., the total variation with respect to the $\ell_1$ norm) as the maximum sampling gap
$\Delta\rightarrow0$ for absolutely continuous trajectories.
Hence the statistical error introduced by reconstruction vanishes as the
sampling grid becomes denser.

This behavior differs fundamentally from classical basis-expansion methods.
For a fixed basis dimension, basis reconstruction inevitably incurs an
approximation bias because the trajectory is projected onto a finite-dimensional
subspace. Data-adaptive approaches such as FPCA additionally require estimation
of eigenfunctions, introducing further variability that depends on the sampling
design. Consequently, increasing the sampling frequency alone cannot eliminate
the reconstruction error for these methods unless the basis complexity is also
allowed to increase. 

The above theoretical analysis (more details are given in Appendix~\ref{app:error_propagation}) explains the empirical observations in
Scenario~3 (Section~\ref{sec:expe}). 

\subsection{Implementation}
\label{sec:impl}

This section outlines the computational workflow for implementing the proposed PSLR model. The approach exploits the algebraic structure of piecewise linear paths to efficiently compute truncated signatures, which are then used in an $\ell_1$-penalized logistic regression framework to jointly estimate model parameters and select the optimal signature truncation order.

Functional inputs are typically observed as discrete multivariate time series $\bm{x}_i \in \mathbb{R}^{(d-1) \times m_i}$. We embed these into continuous paths $\bm{X}_i: [0,T] \to \mathbb{R}^{d-1}$ via linear interpolation. Each path is augmented with time as an additional channel, yielding augmented paths $\widetilde{\bm{X}}_i: [0,T] \to \mathbb{R}^{d}$ where the final coordinate is the identity function $t \mapsto t$. The truncated signature of a piecewise linear path can be computed recursively using a two-step procedure: (i) Compute the time-augmented signature for each linear segment. (ii) Concatenate the signatures of individual segments using Chen's identity \citep{chevyrev2016primer}, which follows from the multiplicative property of the signature under path concatenation (see Eq.~\eqref{eq:concat}).
These computations can be efficiently implemented using the \texttt{iisignature} Python package \citep{Reizen:2020}. 

For classification, we adopt an $\ell_1$-penalized logistic regression model \citep{Ped:2011}, and solve the optimization using the dual coordinate descent algorithm implemented in \texttt{liblinear} \citep{Fan:2008}. The full PSLR procedure is summarized in Algorithm~\ref{alg:PSLR}. Note that the LASSO objective in Eq.~\eqref{eq:risk_lasso} is equivalent to the constrained formulation
$
\hat{\bm{\theta}}_p = \underset{\bm{\theta}_p \in B_{p,r}}{\arg\min}\,\widehat{\mathcal{R}}_{p,n}(\bm{\theta}_p),
$
where $\widehat{\mathcal{R}}_{p,n}(\bm{\theta}_p)$ is defined in Eq.~\eqref{eq:risk_emp}, due to the bijective relationship between the penalty parameter $\lambda$ and the constraint radius $r$ in the $\ell_1$-ball $B_{p,r}$.

\begin{algorithm}[ht!]
\caption{Path Signatures Logistic Regression}
\label{alg:PSLR}
\textbf{Input:} Data $\{(\bm{x}_i, \bm{z}_i, y_i)\}_{i=1}^n$, regularization parameter $\lambda$, truncation bound $P$ \\
\textbf{Output:} Estimated truncation order $\hat{p}$ and coefficients $\hat{\bm{\theta}}_{\hat{p}}$
\begin{algorithmic}[1]
\FOR{$i = 1$ to $n$}
    \STATE Interpolate $\bm{x}_i$ to construct a continuous piecewise-linear path $\bm{X}_i : [0,T] \rightarrow \mathbb{R}^{d-1}$
    \STATE Time-augment: $\widetilde{\bm{X}}_i = (\bm{X}_i, t)$
\ENDFOR
\FOR{$p = 1, \dots, P$}
    \STATE Compute truncated signatures: $S_p(\widetilde{\bm{X}}_i)$ for $i = 1, \dots, n$
    \STATE Form combined feature vectors: $\widetilde{S}_p(\widetilde{\bm{X}}_i, \bm{z}_i) = \left[S_p(\widetilde{\bm{X}}_i)^\top, \bm{z}_i^\top \right]^\top$
    \STATE Solve the Lasso-regularized logistic regression:
    \begin{equation}\label{eq:risk_lasso}
    \hat{\bm{\theta}}_p = \arg\min_{\bm{\theta}_p} \left\{ \frac{1}{n} \sum_{i=1}^n \left[ -y_i \langle \widetilde{S}_p(\widetilde{\bm{X}}_i, \bm{z}_i), \bm{\theta}_p \rangle + \log\left(1 + e^{\langle \widetilde{S}_p(\widetilde{\bm{X}}_i, \bm{z}_i), \bm{\theta}_p \rangle} \right) \right] + \lambda \|\bm{\theta}_p\|_1 \right\}
    \end{equation}
    \STATE Record the minimal empirical loss: $\widehat{L}_n(p) = \widehat{\mathcal{R}}_{p,n}(\hat{\bm{\theta}}_p)$
    \STATE Compute complexity penalty: $\mathrm{pen}_{n}(p,q) = \frac{C_{\mathrm{pen}} \, \sqrt{s_d(p)\,e^{q}}}{n^{\rho}}$
\ENDFOR
\STATE Select optimal truncation order: $
\hat{p} = \arg\min_{1 \leq p \leq P} \left\{ \widehat{L}_n(p) + \mathrm{pen}_n(p,q) \right\}$
\STATE Return classifier coefficients $\hat{\bm{\theta}}_{\hat{p}}$
\end{algorithmic}
\end{algorithm}

\paragraph{Algorithmic Complexity.} For a $d$-dimensional path sampled at $m$ time points, the computational complexity of computing signatures up to order $p$ is $\mathcal{O}(m d^p)$, highlighting the exponential dependence on $p$ and underscoring the necessity of selecting an optimal truncation order.  In Algorithm~1, we compute signatures for all $p=1,\dots,P$, where $P$ is an a priori bound. The total computational cost therefore scales as
\[
\mathcal{O}\Bigl( n m \sum_{p=1}^{P} d^p \Bigr) = \mathcal{O}\bigl( n m (d^{P+1}-1)/(d-1) \bigr).
\]
Because the theoretically justified $P$ grows only logarithmically with $n$ (and is usually small, e.g., $P\le 12$ for practical sample sizes), the exponential dependence on $P$ is effectively bounded by a constant. Practically, we set $P$ to a conservative value (e.g., $P=12$ by default) and rely on the penalized criterion to select $\widehat p \le P$. The Lasso logistic regression step for each $p$ has complexity $\mathcal{O}( (s_d(p)+q)^3 )$ when using coordinate descent, but this is dominated by the signature computation for moderate $p$.

In practice, we first tune the regularization parameter $\lambda$ via cross-validation using the concatenated feature vectors $\widetilde{S}_1(\wt{\bm{X}}_i, \bm{z}_i)$.  The parameter $\rho$ is fixed at $0.4$ across all experiments. To determine the penalty constant $C_{\mathrm{pen}}$ in the model selection criterion, we employ the \emph{slope heuristics} method \citep{birg2007penal, BaudryJean-Patrick2012Shoa}. Specifically, we plot the estimated order $\hat{p}$ against $C_{\mathrm{pen}}$ and identify the first sharp drop in $\hat{p}$; we then set $C_{\mathrm{pen}}$ to twice the corresponding value. For instance, in the top-left panel of Figure~\ref{fig:order_vard}, the first drop occurs at $C_{\mathrm{pen}} = 0.008$, prompting us to choose $C_{\mathrm{pen}} = 0.016$, yielding $\hat{p} = 7$. The grid of $C_{\mathrm{pen}}$ values is chosen to ensure that $\hat{p}$ can drop to zero.
Scalar covariates $\bm{z}_i$ are standardized prior to model fitting and seamlessly integrated with standardized signature features, thus preserving their interpretability within the regression framework. The PSLR algorithm enables efficient and scalable classification for high-dimensional functional data enriched with scalar information. Moreover, it achieves a favorable trade-off between flexibility and parsimony through its rigorous model selection mechanism. Source code is available at: \url{https://github.com/Drivergo-93589/PSLR}. 

\section{Experiments}
\label{sec:expe}

In this section, we evaluate the performance of the proposed \textsc{PSLR} model through extensive experiments on both synthetic and real-world datasets. We benchmark against two reduced versions of \textsc{PSLR} and three classical functional classification methods. Performance is assessed using classification accuracy and F1 score.

The two ablated versions of \textsc{PSLR} are: (i) \textsc{Signature}: \textsc{PSLR} with only path signature input, (ii) \textsc{Scalar}: \textsc{PSLR} with only scalar covariates input. These serve as ablation studies to isolate the contribution of each component. For classical functional classification baselines, we transform each component of the functional predictor into coefficients using either B-spline, Fourier basis expansions, or functional principal component analysis (FPCA). The resulting features are concatenated across dimensions and used in a logistic regression classifier. The number of basis functions or components is selected via cross-validation. Scalar covariates are included as additional features for all baseline models. For simplicity, we refer to these methods as \textsc{B-spline}, \textsc{Fourier}, and \textsc{FPCA}, respectively. 

\subsection{Simulation}
We design three simulation scenarios to assess the effectiveness of \textsc{PSLR} under different data characteristics: (i) varying the number of functional components ($d$),
(ii) varying the number of scalar covariates ($q$), and (iii) irregular sampling with missing or unevenly spaced time points. 

Synthetic datasets $\mathcal{D}(d,q) = \{(\bm{X}_i(t), \bm{z}_i, y_i)\}_{i=1}^n$ are generated as follows. Functional observations $\bm{X}_i(t) = (X_i^1(t), \ldots, X_i^d(t))$ are constructed via
\[
X_i^j(t) = f_j(t) + N_{i,j}^{(\varepsilon)}(t), \quad 1 \leq j \leq d,\ t \in [0,1],
\]
where $f_j(t)$ is a base signal (distinct across two classes), and $N_{i,j}^{(\varepsilon)}(t)$ is a sample from a $\varepsilon$-smoothed zero-mean Gaussian process (GP) with an exponential kernel (length-scale 1). The smoothed Gaussian process $N_{i,j}^{(\varepsilon)}(t)$ is defined as the convolution of a raw GP sample path $N_{i,j}^{\text{raw}}(t)$ with a Gaussian smoothing kernel:
\[
N_{i,j}^{(\varepsilon)}(t) = \int_0^1 N_{i,j}^{\text{raw}}(s) \, \phi_\varepsilon(t-s) \, ds,
\]
where $\phi_\varepsilon$ denotes a Gaussian kernel with bandwidth $\varepsilon > 0$. Here the smoothed Gaussian process satisfies the bounded variation condition required by our theoretical framework. For all subsequent experiments, we use a smoothing bandwidth of $\varepsilon = 0.02$. Definitions of $f_j(t)$ for $j=1,\ldots,8$ are provided in Table~\ref{tab:sim_fun}. We also define a ramp function
\[
g(t;a) = 
\begin{cases}
0, & 0 \leq t \leq a, \\
\frac{t - a}{1 - a}, & a < t \leq 1,
\end{cases}
\]
and employ standard densities $f_{N(\mu,\sigma^2)}(\cdot)$ and $f_{\mathrm{Beta}(\alpha,\beta)}(\cdot)$ as components. Each functional trajectory is sampled at $T = 100$ uniformly spaced time points. Figure~\ref{fig:data_sim} (Appendix~\ref{app:expe}) displays sample curves for two classes.
\begin{table}[htpb]
  \centering
  \caption{Basis functions for the 8-dimensional functional predictor (two-class simulation)}
  \resizebox{\linewidth}{!}{
  \begin{tabular}{ccccc}
    \toprule
      Label & $f_1(t)$ & $f_2(t)$ & $f_3(t)$ & $f_4(t)$ \\
    \midrule
    $y = 0$ & $\exp(\cos 2\pi t)/3$ & $1.6 t^{1/3}$ & $\log(0.5+\cos(\frac{\pi t^4}{2})) + 1$ & $\exp(\sin 2\pi t)/3$ \\
    $y = 1$ & $\exp(\cos 2\pi t^{1.05})/3$ & $\sqrt{3} t^{1/2}$ & $0.9\log(0.5+\cos(\frac{\pi t^3}{2})) + 1$ & $\exp(\sin 2\pi t^{1.05})/3$ \\
    \midrule
    Label & $f_5(t)$ & $f_6(t)$ & $f_7(t)$ & $f_8(t)$ \\
    \midrule
    $y = 0$ & $0.6f_{N(0,1)}(5t) + 0.4 f_{\text{Beta}(2,3)}(5t)$ +0.5& $t^4-g(t;0.55) +1$ & $-0.15t-0.2 t^{2} + 0.95$ & $\text{sigmoid}(20t-10)/3+0.5$\\
    $y = 1$ & $0.3f_{N(0.5,0.5)}(5t) + 0.3 f_{\text{Beta}(3,4)}(5t)+0.5$ & $t^5-g(t;0.45) +1$ & $-0.5t+ 0.2 t^{2} + 0.95$ & $\tanh(12t-6.3)/3+0.5$ \\
    \bottomrule
  \end{tabular}
  }
  \label{tab:sim_fun}
\end{table}
Scalar covariates $\bm{z}_i = (z_i^1, \ldots, z_i^q)$ are independently sampled from distributions $\mr{D}_j$ (distinct across two classes) detailed in Table~\ref{tab:sim_scalar}. Notably, the true functional $F$ is left unspecified; labels are generated directly from the joint latent class of $\mathbf{X}_i$ and $\mathbf{z}_i$ without a logistic link or additive structure. Following \cite{Tang2022model} and \cite{Jiang2026deepfrc}, this design rigorously tests \textsc{ATSLR}'s ability to approximate unknown decision boundaries via truncated signatures while combining both data types.

\begin{table}[htpb]
  \centering
  \caption{Probability distributions used for generating simulated two-class scalar data}
  \resizebox{\linewidth}{!}{
  \begin{tabular}{ccccccccc}
    \toprule
      Label & $\mr{D}_1$ & $\mr{D}_2$ & $\mr{D}_3$ & $\mr{D}_4$ & $\mr{D}_5$ & $\mr{D}_6$ & $\mr{D}_7$ & $\mr{D}_8$ \\
    \midrule
    $y = 0$ & $U(1,2)$ & $N(0,1)$ & $\text{Exp}(2)$ & $\chi^2(0.1)$ & $ \text{logN}(0,1)$ & $\Gamma(2,2)$ & $\text{Beta}(2,3)$ & $\text{Bernoulli}(0.55)$\\
    $y = 1$ & $U(0.75,1.75)$ & $N(0.5,1)$ & $\text{Exp}(1)$ & $\chi^2(0.2)$ & $\text{logN}(0.25,1)$ & $\Gamma(3,2)$ & $\text{Beta}(3,2)$ & $\text{Bernoulli}(0.45)$ \\
    \bottomrule
  \end{tabular}
  }
  \label{tab:sim_scalar}
\end{table}
For all experiments, we simulate balanced datasets with $n = 1000$ and split into 80\% training and 20\% testing sets. Each configuration is repeated 50 times for statistical robustness.

\paragraph{Scenario 1: Varying Functional Dimensions.}
We consider $d \in \{1,2,4,8\}$ with $q = 3$ fixed, generating datasets $\mathcal{D}(d,3)$. Figure~\ref{fig:order_vard} shows the selected truncation orders for both \textsc{PSLR} and \textsc{Signature}. Two key observations emerge: (1) truncation orders decrease monotonically with dimension $d$ as the penalty term $\mathrm{pen}_n(p,3)$ increases; (2) owing to its reduced penalty term $\mathrm{pen}_n(p,0)$, the \textsc{Signature} method consistently achieves higher truncation orders than \textsc{PSLR} in all cases except at $d=2$ where they coincide.
\begin{figure}[htpb]
    \centering
    \includegraphics[width=1\linewidth]{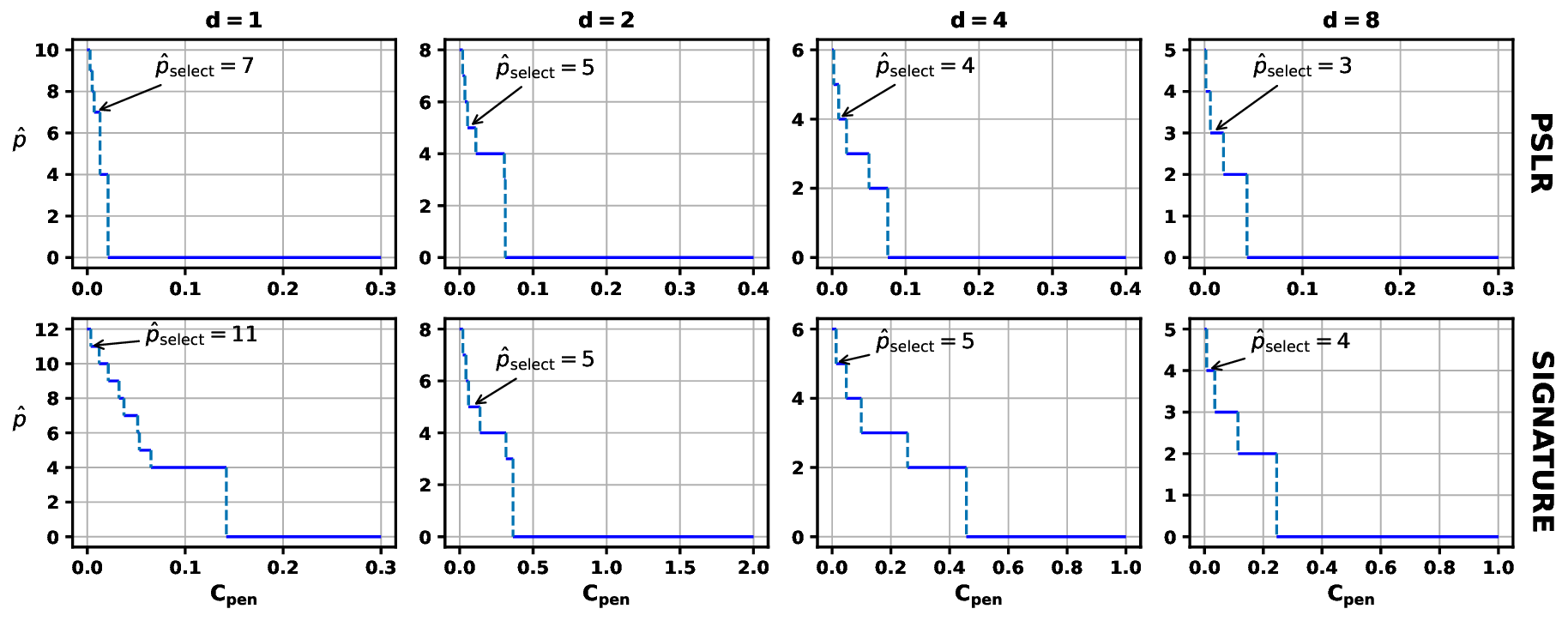}
    \caption{Truncation order selection for the PSLR (with fixed $q = 3$) and \textsc{Signature} across dimensions $d \in \{1,2,4,8\}$ in one representative dataset (Scenario 1).}%. Results are shown for one representative dataset per type (out of 50 instances).}
    \label{fig:order_vard}
    \vspace{-0.1in}
\end{figure}
\begin{figure}[htpb]
    \centering
    \includegraphics[width=0.85\linewidth]{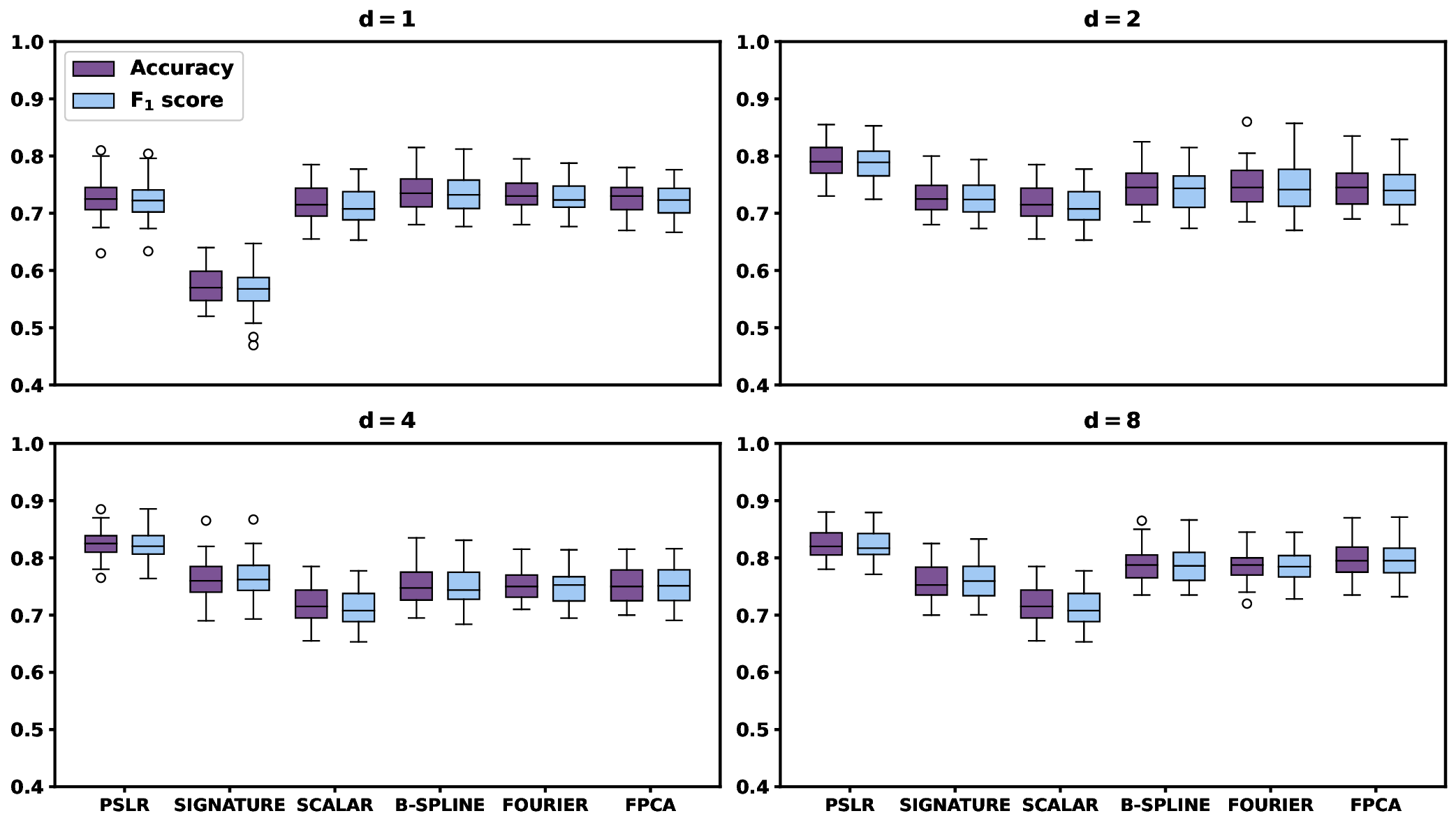}
    \caption{Classification performance for different models across dimensions $d \in \{1,2,4,8\}$ with fixed $q=3$ (Scenario 1). Boxplots summarize results from 50 simulated datasets.}
    \label{fig:result_vard}
\end{figure}

Figure~\ref{fig:result_vard} summarizes classification accuracy and F1 score across 50 replicates. The full \textsc{PSLR} model outperforms both \textsc{Signature} and \textsc{Scalar} in all settings, confirming the additive value of combining functional and scalar inputs. In high-dimensional cases ($d \geq 2$), \textsc{PSLR} significantly outperforms classical models due to its ability to capture inter-dimensional correlations without requiring subjective basis choices. In the univariate case ($d = 1$), \textsc{PSLR} remains competitive, highlighting its robustness.

\paragraph{Scenario 2: Varying Number of Scalar Covariates.}
We fix $d = 3$ and vary $q \in \{1,2,4,8\}$, generating datasets $\mathcal{D}(3,q)$. Figure~\ref{fig:order_varq} (Appendix~\ref{app:expe}) shows that the truncation order for \textsc{PSLR} decreases with increasing $q$, as the penalty $\mathrm{pen}_{n}(p,q)$ grows with $q$. In contrast, \textsc{Signature}'s truncation order remains constant since it ignores scalar features.

Classification results in Figure~\ref{fig:result_varq} indicate that \textsc{PSLR} consistently achieves the best performance, significantly outperforming both ablated variants and classical baselines. The inclusion of more scalar features improves performance across all models except \textsc{Signature}. This scenario illustrates that scalar covariates not only influence model complexity but also enhance predictive performance by regularizing the truncation order in \textsc{PSLR}.
\begin{figure}[htpb]
    \centering
    \includegraphics[width=0.85\linewidth]{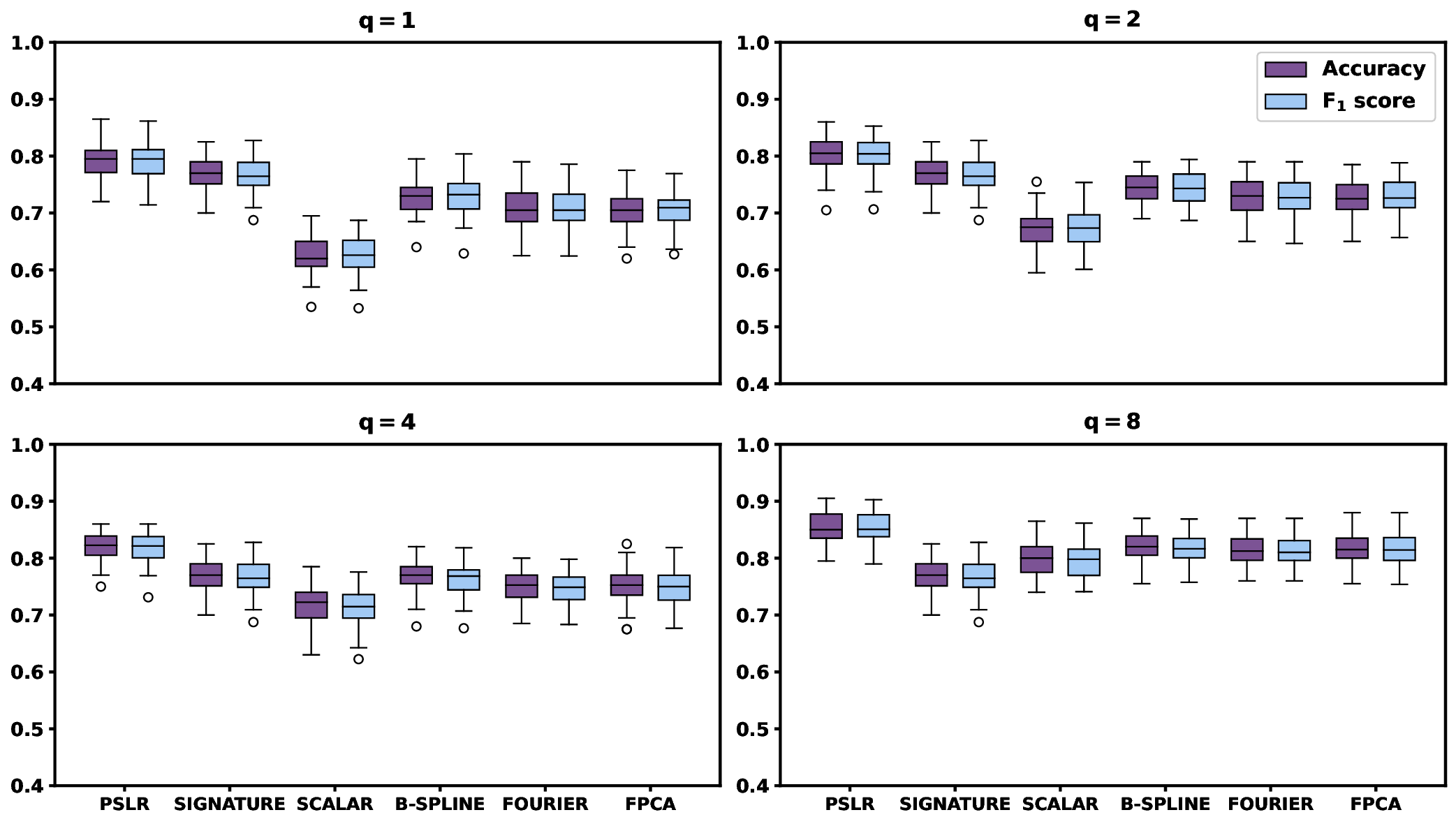}
    \caption{Classification performance for different models across numbers of scalar covariates $q \in \{1,2,4,8\}$ with fixed dimension $d = 3$ (Scenario 2). Boxplots summarize results from 50 simulated datasets.}
    \label{fig:result_varq}
\end{figure}

\paragraph{Scenario 3: Irregular Sampling.}
This scenario evaluates the robustness of \textsc{PSLR} and competing methods under irregular sampling of multi-dimensional functional covariates. We begin with the original datasets $\mathcal{D}(2,1)$, where the functional covariate dimension is $d = 2$ and the scalar covariate dimension is $q = 1$. Two types of irregular sampling are introduced: (a) randomly omitting observations with missing probabilities of $10\%$, $20\%$, and $30\%$ independently at each time point and for each functional dimension; and (b) perturbing the temporal grid to create unevenly spaced time points, defined by $t_k = \sum_{i=1}^k I_i / \sum_{i=1}^T I_i$, where $I_1 = 0$ and $I_k = 0.01 + |N_k|$ for $k = 2, \ldots, T$, with $N_k \sim \mathcal{N}(0.99, \sigma_T^2)$. We consider three levels of temporal scrambling by setting $\sigma_T \in \{0.1, 0.3, 0.5\}$. Figure~\ref{fig:result_timep} (Appendix~\ref{app:expe}) shows the estimated truncation order $\hat{p}$ across both irregular sampling schemes and the original data.

Figure~\ref{fig:result_robust} summarizes the classification performance of all methods (excluding ablated variants) across these datasets. As expected, \textsc{PSLR} consistently outperforms all baseline methods across data settings and maintains stable accuracy and F1 scores under both missing and uneven sampling conditions, due to the path signature’s ability to capture the global geometry of irregularly sampled trajectories. In contrast, classical approaches (\textsc{B-spline}, \textsc{Fourier}, and \textsc{FPCA}) show notable performance degradation, with decreasing mean accuracy as the missing rate grows (Figure~\ref{fig:result_robust}(a)) or as temporal distortion intensifies (Figure~\ref{fig:result_robust}(b)). This empirical finding is theoretically grounded (see Section~\ref{sec:irregular}): \textsc{PSLR}'s error is driven solely by the sampling grid and vanishes with refinement, whereas basis methods suffer from an irreducible structural bias that does not diminish with increased sampling frequency. These results highlight the sensitivity of basis-based models to irregular sampling and underscore the superior robustness and accuracy of \textsc{PSLR} in non-ideal sampling scenarios. For a detailed robustness analysis specific to this scenario, we refer the reader to Appendix~\ref{app:cc}.

\begin{figure}[htpb]
    \centering
    \begin{subfigure}[b]{0.95\textwidth}
        \centering
        \includegraphics[width=\textwidth]{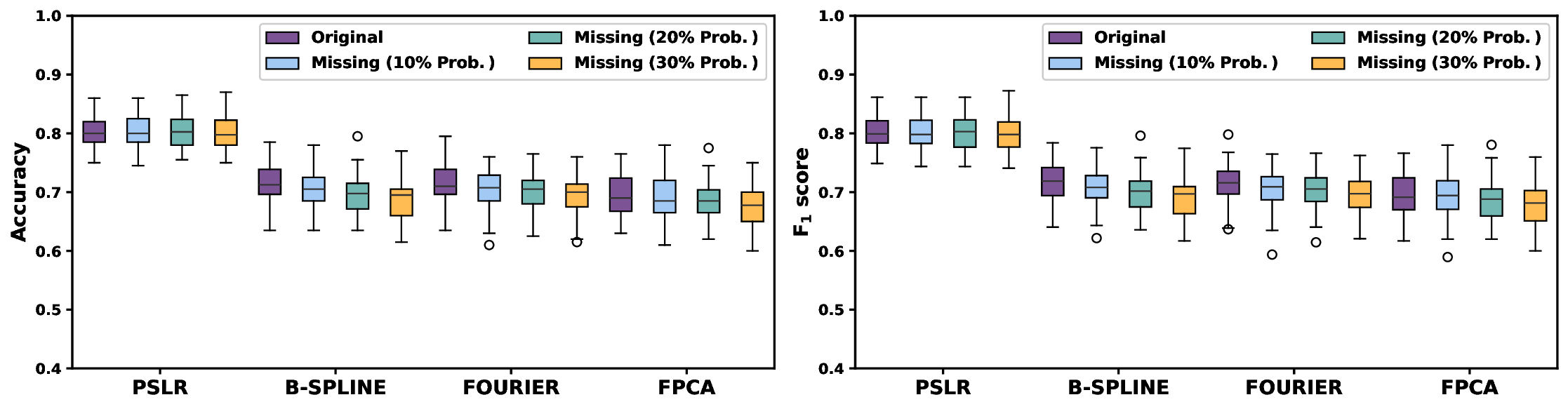}
        \vspace{-0.2in}
        \caption{Original vs. Missing Data (10\%, 20\%, 30\% Prob.)}
        %\label{fig:result_park}
    \end{subfigure}
    \hfill
    \begin{subfigure}[b]{0.95\textwidth}
        \centering
        \includegraphics[width=\textwidth]{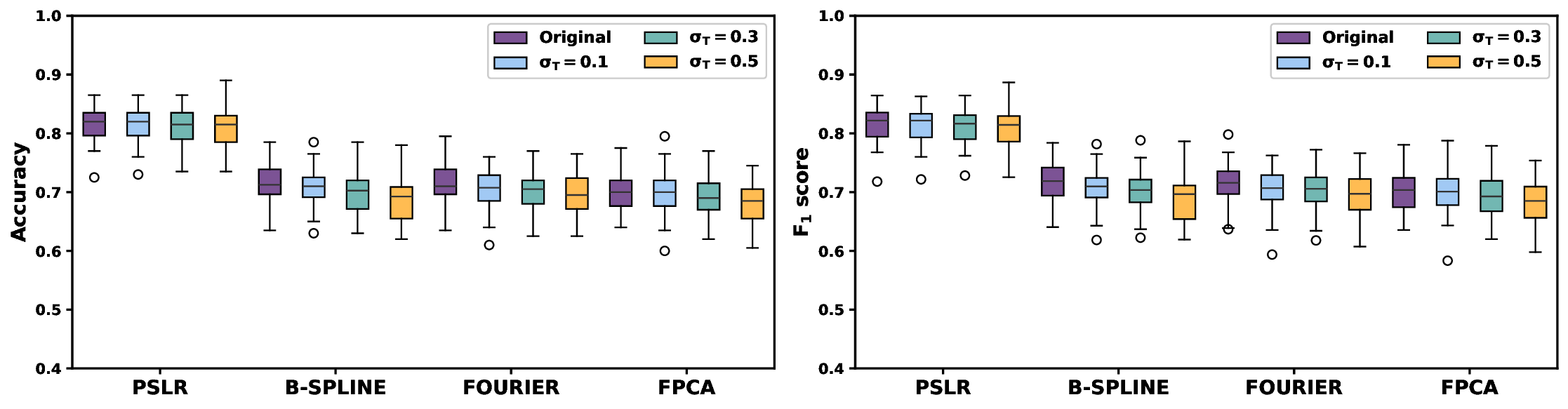}
        \vspace{-0.2in}
        \caption{Original vs. Unevenly Spaced Sampling ($\sigma_T = 0.1,0.3,0.5$)}
        %\label{fig:order_park}
    \end{subfigure}
    %\vspace{-0.2in}
    \caption{Classification performance of different models on both the original data and its irregularly sampled versions (Scenario 3). Boxplots summarize results from 50 simulated datasets.}
    \label{fig:result_robust}
\end{figure}

These results across the three scenarios collectively demonstrate that the proposed \textsc{PSLR} model achieves superior classification performance under a variety of data settings. Its advantages arise from: (i) the expressive, basis-free nature of path signatures, which capture nonlinear and cross-channel dependencies; (ii) the seamless integration of scalar and functional covariates within a unified framework; and (iii) robustness to moderate irregularities in functional data through the extraction of stable, geometry-aware features.

\subsection{Application}
\label{sec:app}

In this section, we evaluate the proposed PSLR model and its baseline counterparts on three publicly available real-world datasets: the \href{https://physionet.org/content/gaitpdb/1.0.0/}{Gait in Parkinson's Disease Database}~\citep{Gold:2000}, the \href{https://www.kaggle.com/datasets/malekzadeh/motionsense-dataset/data}{MotionSense Dataset: Sensor Based Human Activity and Attribute Recognition}~\citep{Male:2019}, and the \href{https://doi.org/10.6084/m9.figshare.28806086}{Dataset of Clinical Gait Signals with Wearable Sensors from Healthy, Neurological, and Orthopedic Cohorts}~\citep{voisard2025dataset}. Due to limited sample sizes in these datasets, we adopt 20 random train-test splits (with 80\% training and 20\% testing) to ensure statistical robustness.

\paragraph{Gait Analysis in Parkinson’s Disease Using VGRF.}
The dataset comprises vertical ground reaction force (VGRF) measurements from 93 Parkinson’s disease (PD) patients (mean age: 66.3 years, 63\% male) and 73 age-matched healthy controls (mean age: 66.3 years, 55\% male), recorded at 100 Hz across four batches. We focus on one batch (35 PD, 29 controls) and analyze four representative VGRF channels (L1, R1, R6, TL) from 16 foot-embedded sensors (plus aggregate TL/TR channels). Time-normalized gait cycles ($t \in [0,1]$) exhibit kinematic-dependent sampling irregularity due to heterogeneous gait speeds (see Figure~\ref{fig:data_park} in Appendix~\ref{app:expe}). The binary classification task ($y=1$ for PD vs. $y=0$ for health control) incorporates scalar covariates: Age, Height, Time Up and Go (TUAG), and Gait Speed.

Figure~\ref{fig:order_parkmotion}(a) (Appendix~\ref{app:expe}) shows the selection of truncated order $\hat{p}$ for the PSLR and \textsc{Signature} models. The PSLR model selects $\hat{p}=3$. Figure~\ref{fig:result_parkinson} presents classification performance comparisons across all models. We draw the following conclusions:  
(i)~the \textsc{PSLR} model substantially outperforms classical baselines (\textsc{B-spline}, \textsc{Fourier}, \textsc{FPCA}), highlighting the effectiveness of path signatures in capturing high-dimensional functional information and their capability to address irregular sampling; (ii)~\textsc{PSLR} also significantly surpasses the \textsc{signature} and \textsc{scalar} ablations, demonstrating the synergistic benefits of jointly modeling functional and scalar covariates, which suggests that both predictor types play crucial roles in functional classification tasks.

\begin{figure}[htpb]
    \centering
    \includegraphics[width=0.65\textwidth]{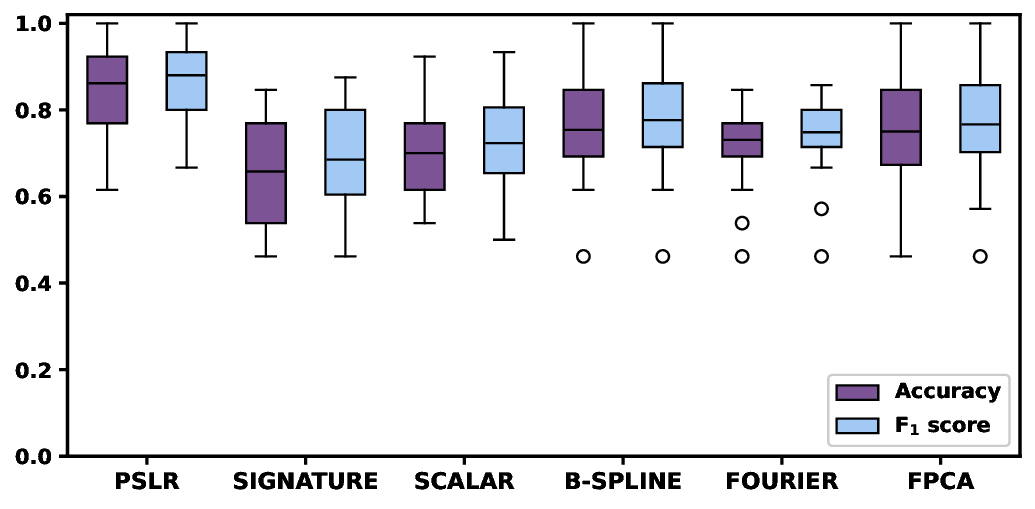}
    \caption{Boxplots shown classification peformance accross all the models from 20 random train-test splits for the Parkinson dataset.}
    \label{fig:result_parkinson}
    \vspace{-0.1in}
\end{figure}

Figure~\ref{fig:park_coef} presents the estimated coefficients $\hat{\bm{\theta}}_{\hat{p}}$ of the PSLR model and their interpretation. For scalar covariates, TUAG has the largest positive coefficient while Speed has the largest negative coefficient, aligning with clinical observations that longer TUAG and slower speed are symptomatic of PD. At order~1, the negative coefficients for $S^{(2)}(\wt{\bm{X}})$ and $S^{(4)}(\wt{\bm{X}})$ - which capture variation in the second and fourth channels (computed as last value minus initial value; see Section~\ref{sec:path} for geometric interpretation) - suggest that greater variability in right-foot (R1) and left-foot total (TL) vertical ground reaction forces is associated with reduced likelihood of Parkinson's disease (PD). This suggests reduced VGRF variability in specific foot regions (R1 and TL) likely reflects rigid and cautious gait characteristic of PD. At order~2, the dominant negative coefficient corresponds to $S^{(2,1)}(\wt{\bm{X}})$ (representing the interaction between R1 and L1 sensors), indicating that coordinated increases across both feet reduce PD probability. This suggests disrupted bilateral coordination (e.g., reduced R1–L1 synchrony) reflects the asymmetric motor control and instability in PD gait. Higher-order terms encode progressively more intricate interactions between foot dynamics. The PSLR framework leverages these subtle dynamics without requiring temporal alignment or handcrafted features, demonstrating both biomechanical plausibility and clinical relevance.

\begin{figure}[htpb]
    \centering
    \includegraphics[width=1\linewidth]{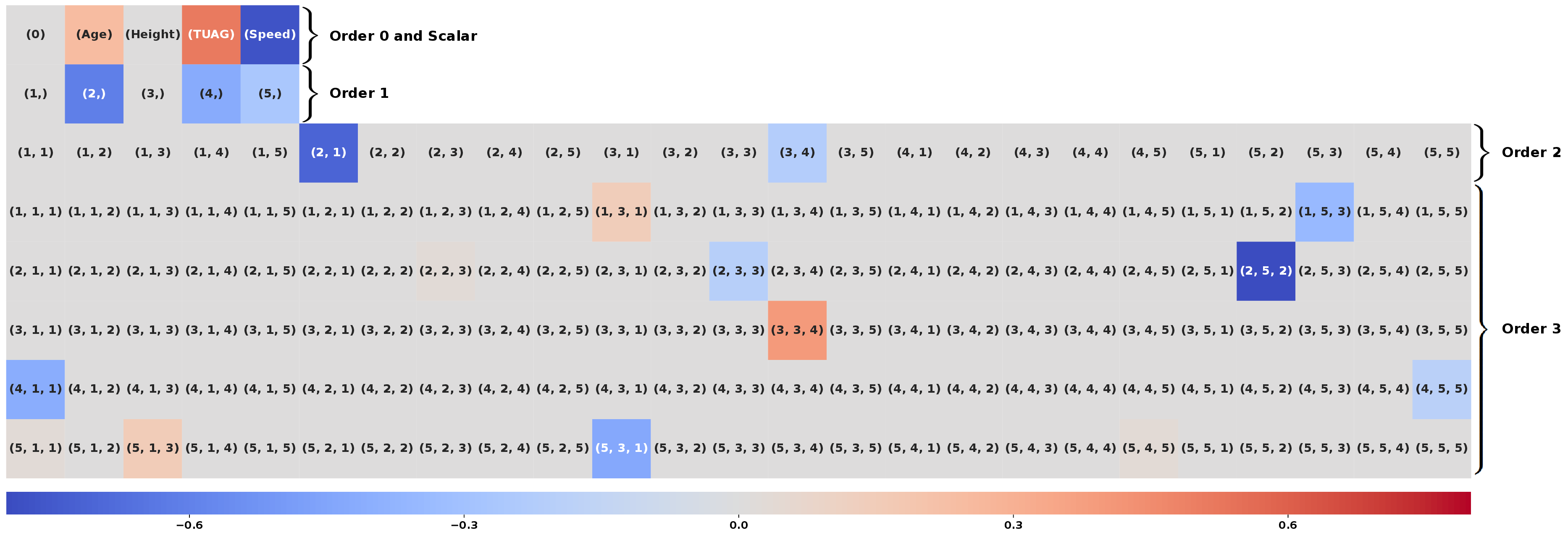}
\caption{Coefficient magnitudes from order-3 PSLR applied to Parkinson’s disease sensor data (L1/R1/R6/TL force sensors + time [channels 1–5]). Coefficients are organized hierarchically by signature order (vertical axis: Order 0 [intercept] = 1, Orders 1–3 = 5/25/125) with 4 scalar covariates aligned top-left.}
    \label{fig:park_coef}
\end{figure}

\paragraph{Human Activity Recognition Using Smartphone Motion Sensors.}
We further analyze the MotionSense dataset, comprising multivariate time-series signals recorded from smartphone sensors (iPhone 6s in front pocket) during six daily activities (walking, jogging, sitting, standing, upstairs, downstairs) performed by 24 participants. Data were collected via the iOS Core Motion API, capturing four motion modalities: attitude, gravity, user acceleration, and rotation rate. For our binary classification task ($y=1$ for walking vs. $y=0$ for jogging), we select subjects performing only one activity to ensure independence, resulting in a balanced dataset (16 training and 8 testing samples). We use gravity signals (Gx, Gy, Gz) as functional predictors, preprocessing the data by extracting one periodic cycle per subject (see Figure~\ref{fig:data_sensor} in Appendix~\ref{app:expe}). Time-normalized cycles ($t \in [0,1]$) exhibit sampling irregularity due to gait-speed variability. Scalar covariates include Age, Height, Weight, and Gender.

Figure~\ref{fig:order_parkmotion}(b) in Appendix~\ref{app:expe} shows that the selected truncation order for PSLR is $\hat{p} = 4$. Classification comparisons across all models are reported in Figure~\ref{fig:result_motionsense}. Our results demonstrate that: (i)~{\textsc{PSLR}} achieves superior classification performance (both in accuracy and F1 score) compared to classical baselines (\textsc{B-spline}, \textsc{Fourier}, and \textsc{FPCA}), confirming the expressive power of path signatures for irregularly sampled functional data; and (ii)~while outperforming the \textsc{Scalar} baseline (highlighting the value of functional information), {\textsc{PSLR}} also exhibits significantly higher performance mean and lower performance variance across splits than the \textsc{Signature} model, demonstrating enhanced both accuracy and stability from scalar covariates - collectively underscoring the complementary importance of both predictor types in functional classification tasks.

\begin{figure}[htpb]
    \centering
    \includegraphics[width=0.65\textwidth]{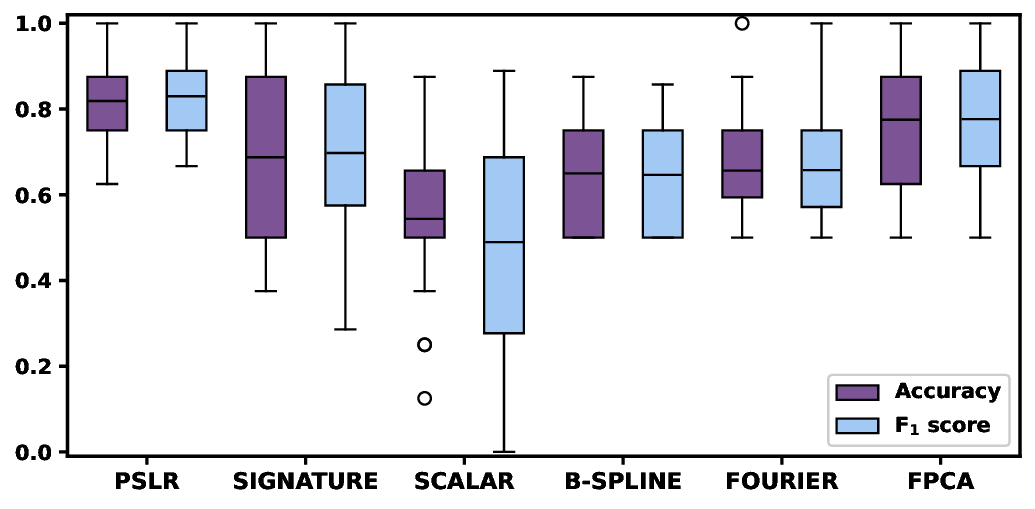}
    \caption{Boxplots shown classification performance accross all the models from 20 random train-test splits for the MotionSense dataset.}
    \label{fig:result_motionsense}
    \vspace{-0.1in}
\end{figure}

Figure~\ref{fig:motion_coef} illustrates the estimated coefficients $\hat{\bm{\theta}}_{\hat{p}}$ of PSLR. For scalar covariates, Weight has the largest positive coefficient, suggesting that heavier individuals are more likely to be predicted as walking rather than jogging — possibly due to differences in exertion. At order~2, the negative coefficient for $S^{(3,4)}(\wt{\bm{X}})$ (capturing the cumulative vertical gravity component) implies that increased Gz reduces the likelihood of walking, consistent with less vertical motion in walking than in jogging. Higher-order coefficients, such as $S^{(3,1,4,1)}(\wt{\bm{X}})$ (negative) and $S^{(4,3,2,3)}(\wt{\bm{X}})$ (positive), represent more complex multivariate dependencies and interactions among gravity axes.

\begin{figure}[htpb]
    \centering
    \includegraphics[width=1\linewidth]{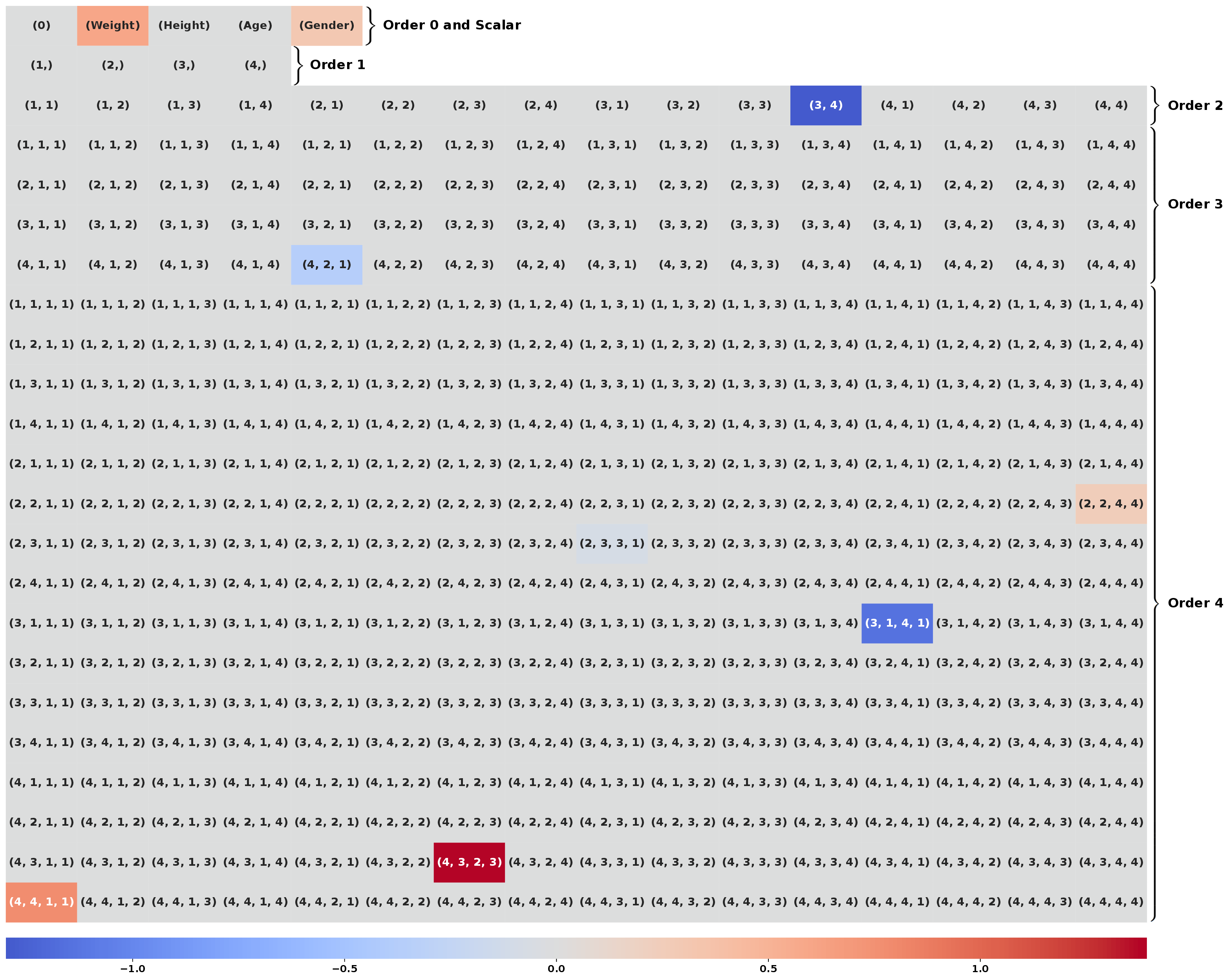}
\caption{Coefficient magnitudes from order-4 PSLR applied to MotionSense dataset (Gx/Gy/Gz sensor signals + time [channels 1–4]). Coefficients are organized hierarchically by signature order (vertical axis: Order 0 [intercept] = 1, Orders 1–4 = 4/16/24/256) with 4 scalar covariates aligned top-left.}
    \label{fig:motion_coef}
    \vspace{-0.1in}
\end{figure}

\paragraph{Multi-Cohort Clinical Gait Dataset from Wearable Sensors.}
The last dataset contains over 11 hours of gait time series data, collected from 1,356 gait trials involving 260 participants. Recordings were performed at 100 Hz using four inertial measurement units placed on the head (HE), lower back (LB), and the dorsal side of each foot (LF and RF). Each sensor location captures 9 signals: triaxial acceleration, triaxial free acceleration, and triaxial angular velocity (X, Y, Z axes). The participants are divided into three groups: 73 healthy subjects, 44 orthopedic patients (diagnosed with hip osteoarthritis, knee osteoarthritis, or anterior cruciate ligament injury), and 143 neurological patients (diagnosed with cerebrovascular accident, Parkinson's disease, chemotherapy‑induced peripheral neuropathy, or radiation‑induced leukoencephalopathy). For the present work, we consider a binary classification task involving the neurological (y = 1) and healthy groups (y = 0). The input features consist of the X‑axis acceleration signals from the four sensor placements (HE\_Acc\_X, LB\_Acc\_X, LF\_Acc\_X, RF\_Acc\_X) (see Figure~\ref{fig:data_gait} in Appendix~\ref{app:expe}) and three scalar covariates: Height, Weight, and Gender.

Figure~\ref{fig:order_gait}(Appendix~\ref{app:expe}) shows the selection of truncated order $\hat{p}$ for the PSLR and \textsc{Signature} models. The PSLR model selects $\hat{p}=4$. Figure~\ref{fig:result_gait} presents classification performance comparisons across all models. Figure~\ref{fig:Gait_coef} illustrates the estimated coefficients $\hat{\bm{\theta}}_{\hat{p}}$ of PSLR. We draw consistent conclusions regarding the effectiveness of path signatures in capturing high-dimensional functional information, the synergistic benefit of jointly modeling functional and scalar covariates, and the crucial roles that both predictor types play in functional classification tasks. Similar interpretability of the signature coefficients is also observed.

\begin{figure}[htpb]
    \centering
    \includegraphics[width=0.65\textwidth]{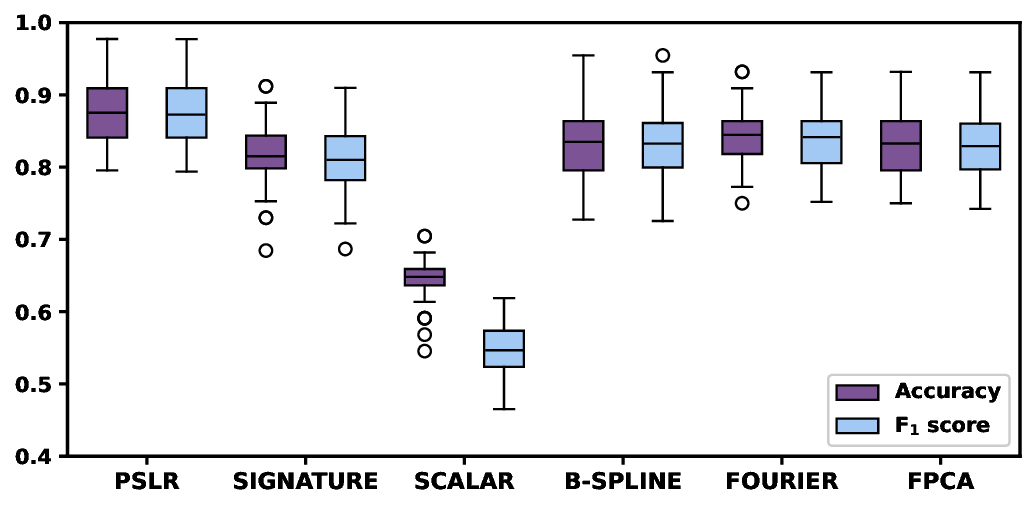}
    \caption{Boxplots shown classification performance across all the models from 50 random train-test splits for the Multi-Cohort Clinical Gait Dataset.}
    \label{fig:result_gait}
    \vspace{-0.1in}
\end{figure}

\paragraph{Interpretability of Signature Coefficients.} 
Unlike conventional functional regression approaches that rely on pointwise time effects, signature-based coefficients in PSLR capture global, geometric summaries of input trajectories. This feature enables robust modeling of inter-variable dependencies and irregular sampling, which are particularly useful in human activity analysis and biomechanics. For deeper interpretability of iterated integrals in dynamic systems, we refer the reader to \citet{giusti2020iterated} and the recent interpretability framework by \citet{ferm2022linear}.

\section{Discussion}
\label{sec:disc}

\paragraph{Advantages of ATSLR.}
The proposed \emph{Adaptively Truncated Signature-based Logistic Regression (ATSLR)} framework builds on the well-established benefits of path signatures---basis-free representation, cross-channel interaction capture, and stability under moderate sampling perturbations. We emphasize the distinctive contributions that extend beyond standard signature-based feature maps.

First, ATSLR introduces a semi-parametric additive structure
\[
\mathrm{Logit}\big(\mathbb{P}(y=1 \mid \bm{X}, \bm{z})\big) = F(\bm{X}) + \bm{z}^\top \bm{\gamma},
\]
which preserves a direct linear interpretation for the scalar covariates $\bm{\gamma}$. This is a deliberate and principled choice: simply augmenting the path with scalar covariates as constant channels is mathematically inadequate because their differentials vanish, yielding no contribution to the signature. Encoding them as time-varying channels would force them to enter only through nonlinear cross-terms, obscuring direct covariate effects. In biomedical applications such as Parkinson's disease analysis, where age, gait speed, and clinical scores have interpretable linear effects while VGRF trajectories exert a nonlinear functional influence, the additive structure is clinically motivated and essential for interpretability.

Second, ATSLR provides the first data-driven framework for selecting the signature truncation order with rigorous non-asymptotic guarantees. Existing signature-based methods fix $p \in \{2,3,4,5,8\}$ heuristically, lacking theoretical justification. In contrast, ATSLR selects $p$ via penalized empirical risk minimization, with theorems establishing the existence of an optimal truncation order (Theorems~\ref{thm1} and \ref{thm2}), its non-asymptotic selection consistency (Theorem~\ref{thm3}), non-asymptotic risk bounds (Theorem~\ref{thm4}), and a finite computable search range (Theorem~\ref{thm:finite_search_bound}). This theoretical foundation distinguishes ATSLR from a standard signature-based classifier and ensures that model complexity adapts to the intrinsic structure of the data without manual tuning.

Third, beyond the qualitative stability of signatures, we provide a general error propagation framework showing that reconstruction error vanishes with grid refinement, formally quantifying robustness under irregular sampling. This stands in contrast to classical basis-expansion methods, whose representation bias and imputation dependence do not diminish with increased sampling frequency. Collectively, these features make ATSLR a principled and theoretically grounded framework for semi-parametric functional classification, particularly suited to complex, high-dimensional, and temporally heterogeneous datasets where interpretability and adaptive complexity control are essential.

\paragraph{Signature Order Selection.}
Selecting the signature truncation order $p$ is pivotal to the performance of ATSLR, as it governs the trade-off between approximation accuracy, model complexity, and computational cost. While fixed-order heuristics (e.g., $p \in \{2,\ldots,8\}$) are commonly used in practice, they lack theoretical justification and often lead to underfitting or overfitting. Standard alternatives such as (i) \emph{Information Criteria} (e.g., AIC/BIC) offer a model-based penalization scheme, but are ill-suited to our setting due to the exponential growth of the feature space with $p$, instability in estimating degrees of freedom under $\ell_1$-regularization, and the absence of finite-sample guarantees. (ii) \emph{Cross-Validation}, though empirically flexible, is computationally burdensome and statistically unstable for nested, high-dimensional signature spaces.

In contrast, our proposed approach selects $p$ via a data-driven minimization of a \emph{penalized empirical risk} criterion, where the penalty $\mathrm{pen}_n(p,q)$ is carefully constructed to scale with the model's functional complexity ($\sqrt{s_d(p)}$) and scalar covariate contribution ($\sqrt{e^q}$). This regularization-based method admits non-asymptotic theoretical guarantees for consistency and risk convergence, scales efficiently in high dimensions, and requires no manual tuning of $p$. Empirically, Figure~\ref{fig:order_significance} further demonstrates that the ATSLR-selected truncation order consistently achieves better classification performance than its immediate neighbors ($\hat{p}-1$ and $\hat{p}+1$) across all three real-world datasets, confirming that ATSLR reliably identifies the optimal order and that fixed-order heuristics are indeed suboptimal.

\begin{figure}[htbp]
    \centering
    \begin{subfigure}[b]{0.32\textwidth}
        \centering
        \includegraphics[width=\textwidth]{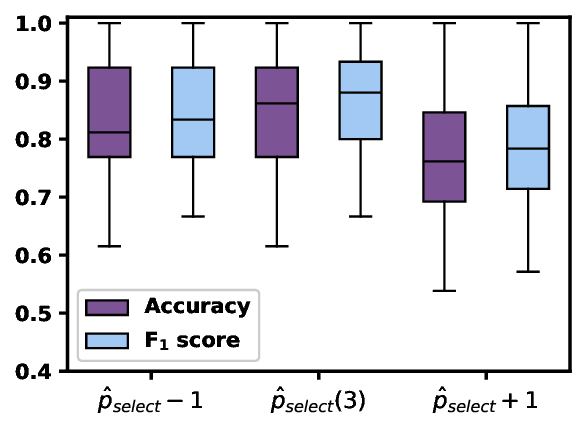}
        \vspace{-0.25in}
        \caption{Parkinson}
    \end{subfigure}
    \hfill
    \begin{subfigure}[b]{0.32\textwidth}
        \centering
        \includegraphics[width=\textwidth]{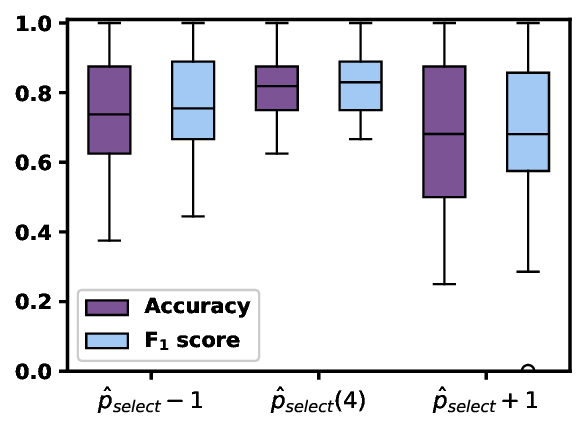}
        \vspace{-0.25in}
        \caption{MotionSense}
    \end{subfigure}
    \hfill
    \begin{subfigure}[b]{0.32\textwidth}
        \centering
        \includegraphics[width=\textwidth]{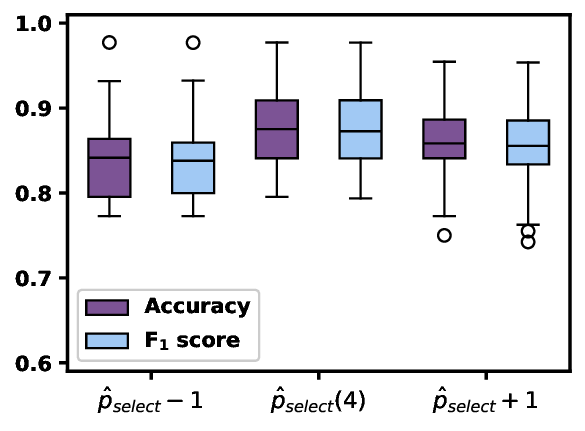}
        \vspace{-0.25in}
        \caption{Multi-Cohort}
    \end{subfigure}
    \caption{Classification performance for the ATSLR-selected order versus its immediate neighbors ($\hat{p}-1$ and $\hat{p}+1$) across three real-world datasets. Results are based on 20 random train-test splits.}
    \label{fig:order_significance}
\end{figure}

\paragraph{Limitations and Future Work.}
While ATSLR provides a principled framework grounded in rough path theory, with an emphasis on adaptive order selection and semi-parametric structure, and offers strong theoretical guarantees and competitive empirical performance, several limitations suggest directions for future research. First, the computational cost of truncated signatures grows rapidly with path dimension $d$ and order $p$. More efficient strategies---such as sparse approximations, randomized projections, or kernelized representations---deserve further exploration, particularly in large-scale or streaming contexts. Second, the interpretability of higher-order signature terms remains limited, which is especially critical in biomedical applications where model transparency is essential. Advancing visualization techniques, domain-informed feature grouping, or attribution methods may help bridge this gap. Third, while this work focuses on binary classification, ATSLR naturally extends to multi-class, ordinal, and survival outcomes, broadening its utility for longitudinal modeling and risk stratification. Incorporating prior knowledge, such as temporal alignment or anatomical structure, could further improve model fidelity. Integration with deep architectures---e.g., neural controlled differential equations (CDEs) or attention-based signature networks---may enhance flexibility and scalability in high-dimensional or noisy settings while preserving theoretical structure. Finally, adapting ATSLR to non-Euclidean functional data (e.g., trajectories on manifolds or graphs) would further extend its applicability to complex structured domains.

\clearpage
%\vskip 0.2in
\bibliography{PSLR}
\bibliographystyle{plainnat}

%%%%%%%%%%%%%%%%%%%%%%%%%%%%%%%%%%%%%%%%%%%%%%%%%%%%%%%%%%%%%%%%%%%%%%%%%%%%%%%
%%%%%%%%%%%%%%%%%%%%%%%%%%%%%%%%%%%%%%%%%%%%%%%%%%%%%%%%%%%%%%%%%%%%%%%%%%%%%%%
% APPENDIX
%%%%%%%%%%%%%%%%%%%%%%%%%%%%%%%%%%%%%%%%%%%%%%%%%%%%%%%%%%%%%%%%%%%%%%%%%%%%%%%
%%%%%%%%%%%%%%%%%%%%%%%%%%%%%%%%%%%%%%%%%%%%%%%%%%%%%%%%%%%%%%%%%%%%%%%%%%%%%%%
%%%%%%%%%%%%%%%%%%%%%%%%%%%%%%%%%%%%%%%%%%%%%%%%%%%%%%%%%%%%
\newpage
\appendix
\onecolumn

\renewcommand{\thefigure}{A\arabic{figure}}
\setcounter{figure}{0}
\renewcommand{\theequation}{A.\arabic{equation}}
\setcounter{equation}{0}

\section{Proof of Theorem \ref{thm1}}
\label{app:thm1}
%{\bf Proof of Theorems \ref{thm1} and \ref{thm2}}. 
\begin{proof} 
We begin by proving the existence of a continuous function $F^* : \mathcal{X} \to \mathbb{R}$ and a bounded coefficient vector $\gamma^* \in \mathbb{R}^q$ that jointly minimize the population risk associated with the semi-parametric logistic model~(\ref{eq:logit1}).
%\paragraph{Compact Hypothesis Space.}

Let $D \subset \mathbb{R}^{d-1}$ be a compact domain and denote by $C(D)$ the space of continuous functions on $D$ endowed with the uniform norm $\|F\|_\infty = \sup_{x \in D} |F(x)|$. For fixed constants $C_F, C_{\bm{\gamma}} > 0$, we define the hypothesis space:
\[
  \Theta := \left\{(F, \bm{\gamma}) : F \in C(D),\, \|F\|_\infty \leq C_F,\ \bm{\gamma} \in \mathbb{R}^q,\ \|\bm{\gamma}\|_1 \leq C_{\bm{\gamma}} \right\}.
\]
We equip $\Theta$ with the product metric:
\[
  d\big((F,\bm{\gamma}), (F',\bm{\gamma}')\big) = \|F - F'\|_\infty + \|\bm{\gamma} - \bm{\gamma}'\|_1.
\]

\paragraph{Step 1: Compactness of $\Theta$.}
\begin{itemize}
  \item The $\ell_1$-ball $\{\bm{\gamma} \in \mathbb{R}^q : \|\bm{\gamma}\|_1 \leq C_{\bm{\gamma}\}}$ is compact as it is closed and bounded in finite-dimensional Euclidean space.
  \item The set $\{F \in C(D) : \|F\|_\infty \leq C_F\}$ is closed, uniformly bounded, and equicontinuous on the compact domain $D$. By the Arzelà–Ascoli theorem, this set is compact in the uniform topology.
  \item Hence, $\Theta$ is compact as a product of two compact metric spaces.
\end{itemize}

\paragraph{Step 2: Continuity of the Risk Function.}
We define the population risk as
\[
  \mathcal{R}(F, \bm{\gamma}) := \mathbb{E}_{(\bm{X},\bm{z},y)}\left[ \ell\left(y, F(\mathbf{X}) + \mathbf{z}^\top \bm{\gamma} \right) \right],
\]
where $\ell(y,\eta) = -y\eta + \log(1 + e^{\eta})$ is the logistic loss.

Let $(F,\bm{\gamma}), (F',\bm{\gamma}') \in \Theta$. For any realization $(\mathbf{X}, \mathbf{z}, y)$, and assuming $\|\mathbf{z}\|_1 \leq C_{\mathbf{z}}$ almost surely, we have
\[
\left|F(\mathbf{X}) + \mathbf{z}^\top \bm{\gamma} - F'(\mathbf{X}) - \mathbf{z}^\top \bm{\gamma}'\right| 
\leq \|F - F'\|_\infty + C_{\mathbf{z}} \|\bm{\gamma} - \bm{\gamma}'\|_1 =: \Delta + C_{\mathbf{z}} \delta.
\]
The logistic loss $\ell(y, \cdot)$ is 1-Lipschitz in $\eta$, so
\[
\left| \ell(y, F(\mathbf{X}) + \mathbf{z}^\top \bm{\gamma}) - \ell(y, F'(\mathbf{X}) + \mathbf{z}^\top \bm{\gamma}') \right| 
\leq \Delta + C_{\mathbf{z}} \delta.
\]
Taking expectations, we obtain
\[
\left|\mathcal{R}(F,\bm{\gamma}) - \mathcal{R}(F',\bm{\gamma}')\right| \leq (1 \vee C_{\mathbf{z}}) \cdot d((F,\bm{\gamma}), (F',\bm{\gamma}')),
\]
i.e., $\mathcal{R}$ is Lipschitz continuous on $\Theta$.

\paragraph{Step 3: Existence of a Minimizer.}
Since $\mathcal{R}$ is continuous on the compact set $\Theta$, the extreme value theorem implies the existence of a minimizer:
\[
(F^*, \bm{\gamma}^*) := \arg\min_{(F,\bm{\gamma}) \in \Theta} \mathcal{R}(F,\bm{\gamma}),
\quad \text{with } \mathcal{R}^* := \mathcal{R}(F^*, \bm{\gamma}^*).
\]

\paragraph{Step 4: Approximation by Truncated Signatures.}
By the universality property of truncated path signatures (see the last key property in Section \ref{sec:path}), for any $\varepsilon > 0$, there exists $p^* \in \mathbb{N}$ and $\bm{\beta}_{p^*}^* \in \mathbb{R}^{s_d(p^*)}$ such that
\[
\left|F^*(\mathbf{X}) - \langle \bm{\beta}_{p^*}^*, S_{p^*}(\wt{\mathbf{X}})\rangle \right| < \varepsilon \quad \text{a.s.}
\]
Let $\bm{\theta}_{p^*}^* := (\bm{\beta}_{p^*}^{*\top}, \bm{\gamma}^{*\top})^\top\in \mathbb{R}^{s_d(p^*)+q}$. Define the population risk of oracle path-signature model:
\[
  \mathcal{R}_{p^*}(\bm{\theta}_{p^*}^*) := \mathbb{E}_{(\bm{X},\bm{z},y)} \left[ \ell\left(y, \widetilde{\bm{S}}_{p^*}^\top \bm{\theta}_{p^*}^* \right) \right],
\]
where $\widetilde{\bm{S}}_{p^*} := (S_{p^*}(\wt{\mathbf{X}})^\top, \mathbf{z}^\top)^\top$. Then:
\begin{align}
\left| \mathcal{R}_{p^*}(\bm{\theta}_{p^*}^*) - \mathcal{R}^* \right|
&= \left| \mathbb{E}_{(\bm{X},\bm{z},y)} \left[ \ell\left( y, \widetilde{\bm{S}}_{p^*}^\top \bm{\theta}_{p^*}^* \right) - \ell\left( y, F^*(\mathbf{X}) + \mathbf{z}^\top \bm{\gamma}^* \right) \right] \right| \notag \\
&\leq \mathbb{E}_{(\bm{X},\bm{z},y)} \left| \ell\left( y, \widetilde{\bm{S}}_{p^*}^\top \bm{\theta}_{p^*}^* \right) - \ell\left( y, F^*(\mathbf{X}) + \mathbf{z}^\top \bm{\gamma}^* \right) \right| \notag \\
&\leq \mathbb{E}_{(\bm{X},\bm{z},y)} \left| \widetilde{\bm{S}}_{p^*}^\top \bm{\theta}_{p^*}^* - \left(F^*(\mathbf{X}) + \mathbf{z}^\top \bm{\gamma}^* \right) \right| \notag \\
&< \varepsilon.
\label{eq:thm1_approx}
\end{align}
This establishes Theorem~\ref{thm1}.
\end{proof}

\section{Proof of Theorem \ref{thm2}}
\label{app:thm2}
\begin{proof}
%\paragraph{Proof of Theorem~\ref{thm2}.}
For a fixed truncation order $p$, define the minimal population risk:
\[
  L(p) := \inf_{\bm{\theta}_p \in B_{p,r}} \mathcal{R}_p(\bm{\theta}_p) = \mathcal{R}_p(\bm{\theta}_p^*),
\]
where $B_{p,r}$ is a compact convex parameter space, and $\bm{\theta}_p^*$ exists since $\mathcal{R}_p$ is convex in $\bm{\theta}_p$ and continuous.

We note that the Hessian of the logistic risk is given by:
\[
  \nabla^2 \mathcal{R}_p(\bm{\theta}_p) 
  = \mathbb{E}_{(\bm{X},\bm{z},y)} \left[ \sigma'(\widetilde{\bm{S}}_p^\top \bm{\theta}_p) \cdot \widetilde{\bm{S}}_p \widetilde{\bm{S}}_p^\top \right] \succeq 0,
\]
where $\sigma(\eta) = (1 + e^{-\eta})^{-1}$ and $\sigma'(\eta) = \sigma(\eta)(1 - \sigma(\eta)) \in (0, 1/4]$. Hence, $\mathcal{R}_p$ is convex, ensuring the existence of $\bm{\theta}_p^*$.

Moreover, the nested structure $B_{0,r} \subset B_{1,r} \subset \cdots \subset B_{p,r} \subset \cdots$ implies that $L(p)$ is non-increasing in $p$. By the approximation argument in Theorem~\ref{thm1}, we know that for any $\varepsilon^* >0$ there exist some $p^*$ and $\bm{\theta}_{p^*}^*$ such that $\left|\mathcal{R}_{p^*}(\bm{\theta}_{p^*}^*) - \mathcal{R}^*\right| < \varepsilon^* $. If $L(p)$ are strictly smaller for some $p > p^*$, i.e.,
\[
  \mathcal{R}_{p^*}(\bm{\theta}_{p^*}^*) - L(p) \geq \varepsilon^* > 0,
\]
then we would obtain
\[
  \left| \mathcal{R}_{p^*}(\bm{\theta}_{p^*}^*) - \mathcal{R}^* \right|= \mathcal{R}_{p^*}(\bm{\theta}_{p^*}^*)-L(p) + L(p) - \mathcal{R}^*
  \geq \varepsilon^* + L(p)-\mathcal{R}^*  \geq \varepsilon^*,
\]
which is a contradiction. Hence, $L(p)$ attains its minimum at $p^*$ and remains constant for all $p \geq p^*$. This concludes the proof of Theorem~\ref{thm2}.
\end{proof}

\section{Proof of Theorem \ref{thm3}}
\label{app:thm3}
%{\bf Proof of Theorems \ref{thm3} and \ref{thm4}}. 
Here we will make extensive use of the concentration results developed by \citet{vanHandel2014}, as well as key analytical techniques from \citet{ferm2022linear}. We focus on the centered empirical risk associated with path signatures truncated at order $p$. Specifically, for any $\bm{\theta}_p \in B_{p,r}$, we define
\begin{align}
Z_{p,n}(\bm{\theta}_p) \coloneqq \widehat{R}_{p,n}(\bm{\theta}_p) - R_p(\bm{\theta}_p),
\label{eq:Zmn}
\end{align}
where $\widehat{R}_{p,n}(\bm{\theta}_p)$ denotes the empirical risk computed from $n$ samples, and $R_p(\bm{\theta}_p)$ is its population analogue.

The next lemma shows that the process $\big(Z_{p,n}(\bm{\theta}_p)\big)_{\bm{\theta}_p \in B_{p,r}}$ is sub-Gaussian with respect to a suitably defined metric. This property allows us to apply a chaining tail inequality from \citet{vanHandel2014}, yielding a uniform high-probability bound on the deviations of $Z_{p,n}(\bm{\theta}_p)$ over the parameter set $B_{p,r}$. This concentration result serves as a central component in the proof of Theorem~\ref{thm3}.

\begin{lemma}\label{lem:A3} %% new lemma 1 
Under assumptions (A.2)–(A.3), for any $p \in \mathbb{N}$, the stochastic process $\big(Z_{p,n}(\bm{\theta}_p)\big)_{\bm{\theta}_p \in B_{p,r}}$ is subgaussian with respect to the semimetric
\begin{equation}\label{eq:D}
    D(\bm{\theta}_p, \bm{\eta}_p) = \frac{C}{\sqrt{n}} \|\bm{\theta}_p - \bm{\eta}_p \|, \quad \bm{\theta}_p, \bm{\eta}_p \in B_{p,r},
\end{equation}
where the constant $C$ is defined as
\begin{equation}\label{eq:C}
    C = 2 \left(C_{\bm{z}} + e^{C_{\bm{X}}+T} \right).
\end{equation}
\end{lemma}

\begin{proof}
By definition, for any $\bm{\theta}_p \in B_{p,r}$, we have $\mathbb{E}[Z_{p,n}(\bm{\theta}_p)] = 0$. Let the loss function $\ell_{(\wt{\bm{X}}, \bm{z}, y)}: B_{p,r} \to \mathbb{R}$ be defined as
\[
\ell_{(\wt{\bm{X}}, \bm{z}, y)}(\bm{\theta}_p) = -y \langle \bm{\theta}_p, \widetilde{\bm{S}}_p(\widetilde{\bm{X}}, \bm{z}) \rangle + \log\left(1 + e^{\langle \bm{\theta}_p, \widetilde{\bm{S}}_p(\widetilde{\bm{X}}, \bm{z}) \rangle} \right).
\]
We first show that $\ell_{(\wt{\bm{X}}, \bm{z}, y)}$ is $C$-Lipschitz. For any $\bm{\theta}_p, \bm{\eta}_p \in B_{p,r}$,
\begin{align}
\big|\ell_{(\wt{\bm{X}}, \bm{z}, y)}(\bm{\theta}_p) - \ell_{(\wt{\bm{X}}, \bm{z}, y)}(\bm{\eta}_p)\big|
&= \left| -y \langle \bm{\theta}_p, \widetilde{\bm{S}}_p(\widetilde{\bm{X}}, \bm{z}) \rangle + \log\left(1 + e^{\langle \bm{\theta}_p, \wt{\bm{S}}_p(\widetilde{\bm{X}}, \bm{z}) \rangle}\right) \right. \notag\\
&\quad + \left. y \langle \bm{\eta}_p, \widetilde{\bm{S}}_p(\widetilde{\bm{X}}, \bm{z}) \rangle - \log\left(1 + e^{\langle \bm{\eta}_p, \widetilde{\bm{S}}_p(\widetilde{\bm{X}}, \bm{z}) \rangle}\right) \right| \notag \\
&\leq |y \langle \bm{\theta}_p - \bm{\eta}_p, \widetilde{\bm{S}}_p(\widetilde{\bm{X}}, \bm{z}) \rangle| + \left|\log\left(1 + e^{\langle \bm{\theta}_p, \widetilde{\bm{S}}_p(\widetilde{\bm{X}}, \bm{z}) \rangle}\right) - \log\left(1 + e^{\langle \bm{\eta}_p, \widetilde{\bm{S}}_p(\widetilde{\bm{X}}, \bm{z}) \rangle}\right) \right| \notag \\
&\leq (|y| + 1) \| \widetilde{\bm{S}}_p(\widetilde{\bm{X}}, \bm{z}) \| \cdot \|\bm{\theta}_p - \bm{\eta}_p\| \notag \\
&\leq 2(\| \bm{z} \| + \| \widetilde{\bm{S}}_p(\widetilde{\bm{X}}) \|) \|\bm{\theta}_p - \bm{\eta}_p \| \notag \\
&\leq 2(C_{\bm{z}} + e^{C_{\bm{X}}+T}) \|\bm{\theta}_p - \bm{\eta}_p\| := C \|\bm{\theta}_p - \bm{\eta}_p \|. \label{eq:lipschitz}
\end{align}
The bound on the exponential term uses the fact that the function $f(t) = \log(1 + e^t)$ is 1-Lipschitz, since its derivative $f'(t) = \frac{e^t}{1 + e^t} \in (0,1)$. Applying Lemma 5.1 of \cite{lyons2014rough}, the norm of the truncated signature $\| \widetilde{\bm{S}}_p(\widetilde{\bm{X}}) \|)$ has the following bound:
\[
\| \wt{\bm{S}}_p(\wt{\bm{X}}) \|\leq \sum_{k=0}^p \frac{\|\wt{\bm{X}}\|_{TV}^k}{k!} \leq e^{\|\wt{\bm{X}}\|_{TV}} = e^{\|\wt{\bm{X}}\|_{TV}+\|t\|_{TV}} \leq e^{C_{\bm{X}}+T}.
\]
Hence, for any $\bm{\theta}_p, \bm{\eta}_p \in B_{p,r}$, the random variable
\[
Z_{\ell} := \ell_{(\wt{\bm{X}}, \bm{z}, y)}(\bm{\theta}_p) - \ell_{(\wt{\bm{X}}, \bm{z}, y)}(\bm{\eta}_p)
\]
is $C \|\bm{\theta}_p - \bm{\eta}_p\|$-subgaussian. By Hoeffding's lemma \citep{Levin:2016}, for all $\lambda \in \mathbb{R}$,
\[
\mathbb{E}\left[ \exp\left( \lambda (Z_{\ell} - \mathbb{E}[Z_{\ell}]) \right) \right] \leq \exp\left( \frac{\lambda^2 (2C \|\bm{\theta}_p - \bm{\eta}_p\|)^2}{8} \right).
\]
Define
\[
Z_{p,n}(\bm{\theta}_p) = \frac{1}{n} \sum_{i=1}^n \left(\ell_{(\wt{\bm{X}}_i, \bm{z}_i, y_i)}(\bm{\theta}_p) - \mathbb{E}[\ell_{(\wt{\bm{X}}_i, \bm{z}_i, y_i)}(\bm{\theta}_p)] \right).
\]
Then, by independence and applying the above subgaussianity to each summand,
\begin{align*}
\mathbb{E}\left[ \exp\left( \lambda (Z_{p,n}(\bm{\theta}_p) - Z_{p,n}(\bm{\eta}_p)) \right) \right]
&= \prod_{i=1}^n \mathbb{E}\left[ \exp\left( \frac{\lambda}{n} \left(Z_{\ell}^{(i)} - \mathbb{E}[Z_{\ell}^{(i)}] \right) \right) \right] \\
&\leq \exp\left( \frac{\lambda^2 C^2 \|\bm{\theta}_p - \bm{\eta}_p\|^2}{2n} \right) \\
&= \exp\left( \frac{\lambda^2 D(\bm{\theta}_p, \bm{\eta}_p)^2}{2} \right),
\end{align*}
where $D(\bm{\theta}_p, \bm{\eta}_p) = \frac{C}{\sqrt{n}} \|\bm{\theta}_p - \bm{\eta}_p\|$. Thus, the process $\big(Z_{p,n}(\bm{\theta}_p)\big)_{\bm{\theta}_p \in B_{p,r}}$ is subgaussian with respect to $D$.
\end{proof}

Now we derive a maximal tail inequality for the process $Z_{p,n}(\bm{\theta}_p)$.

\begin{proposition}\label{prop:A1} %% New prop 1
Under assumptions (A.2)–(A.3), for any $p \in \mathbb{N}$, $x > 0$, and any fixed $\bm{\theta}^0_p \in B_{p,r}$, the following bound holds:
\[
\mathbb{P}\left(\sup_{\bm{\theta}_p \in B_{p,r}} Z_{p,n}(\bm{\theta}_p) \geq 108 Cr \sqrt{\frac{s_d(p) + q}{n}} \sqrt{\pi} + Z_{p,n}(\bm{\theta}^0_p) + x \right) \leq 36 \exp\left(- \frac{x^2 n}{144 C^2 r^2} \right),
\]
where the constant $C$ is defined in equation~\emph{(\ref{eq:C})}.
\end{proposition}

\begin{proof}
By Lemma~\ref{lem:A3}, the process $\left(Z_{p,n}(\bm{\theta}_p)\right)_{\bm{\theta}_p \in B_{p,r}}$ is subgaussian with respect to the metric
\[
D(\bm{\theta}_p, \bm{\eta}_p) = \frac{C}{\sqrt{n}} \|\bm{\theta}_p - \bm{\eta}_p\|.
\]
Applying Theorem 5.29 of \citet{Levin:2016} to the process $Z_{p,n}$ over the metric space $(B_{p,r}; D)$ yields:
\[
\mathbb{P}\left( \sup_{\bm{\theta}_p \in B_{p,r}} Z_{p,n}(\bm{\theta}_p) - Z_{p,n}(\bm{\theta}^0_p) \geq C_0 \int_0^\infty \sqrt{\log N(\varepsilon, B_{p,r}, D)} \, \mathrm{d}\varepsilon + x \right) \leq C_0 \exp\left(- \frac{x^2}{C_0 \, \mathrm{diam}(B_{p,r})^2} \right),
\]
where we take the universal constant $C_0 = 36$ here, and $N(\varepsilon, B_{p,r}, D)$ denotes the $\varepsilon$-covering number of $B_{p,r}$ under the metric $D$. The diameter of $B_{p,r}$ with respect to $D$ satisfies:
\[
\mathrm{diam}(B_{p,r}) = \sup_{\bm{\theta}_p, \bm{\eta}_p \in B_{p,r}} D(\bm{\theta}_p, \bm{\eta}_p) = \frac{2Cr}{\sqrt{n}}.
\]
Using Lemma 5.13 of \citet{Levin:2016}, we relate the covering number under $D$ to that under the Euclidean norm:
\[
N(\varepsilon, B_{p,r}, D) = N\left( \frac{\sqrt{n}}{C} \varepsilon, B_{p,r}, \|\cdot\| \right),
\]
which implies
\[
N(\varepsilon, B_{p,r}, D) \leq \left( \frac{3Cr}{\sqrt{n} \varepsilon} \right)^{s_d(p) + q}, \quad \text{for } \varepsilon < \frac{Cr}{\sqrt{n}},
\]
and $N(\varepsilon, B_{p,r}, D) = 1$ otherwise. Consequently, the entropy integral can be bounded as follows:
\begin{align*}
\int_0^\infty \sqrt{\log N(\varepsilon, B_{p,r}, D)} \, \mathrm{d}\varepsilon 
&= \int_0^{\frac{Cr}{\sqrt{n}}} \sqrt{(s_d(p) + q) \log \left( \frac{3Cr}{\sqrt{n} \varepsilon} \right)} \, \mathrm{d}\varepsilon \\
&\leq 3Cr \sqrt{\frac{s_d(p) + q}{n}} \int_0^\infty 2x^2 e^{-x^2} \, \mathrm{d}x \\
&= 3Cr \sqrt{\frac{s_d(p) + q}{n}} \sqrt{\pi},
\end{align*}
where the second inequality follows from the change of variable $x = \sqrt{\log \left( \frac{2Cr}{\sqrt{n} \varepsilon} \right)}$.
Substituting back into the concentration inequality completes the proof.
\end{proof}

We now divide the proof of Theorem \ref{thm3} into two cases, as $\mathbb{P}(\widehat{p} \neq p^*) = \mathbb{P}(\widehat{p} > p^*) + \mathbb{P}(\widehat{p} < p^*)$. We first consider the case when $\widehat{p} > p^*$ in the following proposition.

\begin{proposition}\label{prop:A2}
    Let $0 < \rho < \frac{1}{2}$, and let $\mathrm{pen}_n(p, q)$ be defined as in Eq. (\ref{eq:penalty}). Let $n_1$ be the smallest integer satisfying
    \begin{equation}\label{eq:n1}
        n_1 \geq \left(\frac{432 \sqrt{\pi} C r \sqrt{s_d(p^* + 1) + q}}{C_{\mathrm{pen}} \sqrt{e^q} (\sqrt{s_d(p^* + 1)} - \sqrt{s_d(p^*)})}\right)^{\frac{1}{\frac{1}{2} - \rho}},
    \end{equation}
    Then, under assumptions (A.1)-(A.4), for any $p > p^*$ and $n \geq n_1$, we have
    \begin{equation}\label{eq:prob}
        \mathbb{P}(\widehat{p} = p) \leq 74 \exp\left(-c_0  \left(n^{1 - 2\rho} + s_d(p)\right)\right),
    \end{equation}
    where
    \begin{equation}\label{eq:C3}
        c_0 = \frac{C_{\mathrm{pen}}^2 d^{p^* + 1}e^q}{256 r C \left(36 r C + 1\right) s_d(p^* + 1)}.
    \end{equation}
\end{proposition}

\begin{proof}
    Theorems \ref{thm1} and \ref{thm2} guarantee the existence of \( p^* \). We now define
    \[
    u_{p,n} = \frac{1}{2} \left( \mathrm{pen}_n(p,q) - \mathrm{pen}_n(p^*,q) \right) = \frac{C_{\mathrm{pen}}}{2} \sqrt{e^q} n^{-\rho} \left( \sqrt{s_d(p)} - \sqrt{s_d(p^*)} \right).
    \]
    Since \( p \mapsto \mathrm{pen}_n(p,q) \) is increasing in \( p \), it is clear that \( u_{p,n} > 0 \) for any \( p > p^* \). From Lemma 2 of \cite{ferm2022linear}, we have the following bound for any \( p > p^* \):
    \begin{equation}\label{eq:lemma2}
    \mathbb{P}(\widehat{p} = p) \leq \mathbb{P} \left( 2 \sup_{\bm{\theta}_p \in B_{p,r}} \left| \widehat{R}_{p,n}(\bm{\theta}_p) - R_p(\bm{\theta}_p) \right| \geq \mathrm{pen}_n(p,q) - \mathrm{pen}_n(p^*,q) \right).
    \end{equation}
    We now proceed with the following decomposition:
    \begin{align}
    \mathbb{P}(\widehat{p} = p) \leq & \mathbb{P} \left( \sup_{\bm{\theta}_p \in B_{p,r}} \left| Z_{p,n}(\bm{\theta}_p) \right| > u_{p,n} \right) \notag \\
    = & \mathbb{P} \left( \sup_{\bm{\theta}_p \in B_{p,r}} Z_{p,n}(\bm{\theta}_p) > u_{p,n} \right) + \mathbb{P} \left( \sup_{\bm{\theta}_p \in B_{p,r}} \left( - Z_{p,n}(\bm{\theta}_p) \right) > u_{p,n} \right), \label{eq:prop2_0}
    \end{align}
    where we focus on the first term of the inequality. The second term can be handled analogously, as Proposition \ref{prop:A1} remains valid when \( Z_{p,n}(\bm{\theta}_p) \) is replaced by \( -Z_{p,n}(\bm{\theta}_p) \). Let \( \bm{\theta}^0_p \) denote a fixed point within \( B_{p,r} \), to be specified later. Then, we have
    \begin{align}
    \mathbb{P} \left( \sup_{\bm{\theta}_p \in B_{p,r}} Z_{p,n}(\bm{\theta}_p) > u_{p,n} \right) 
    = & \mathbb{P} \left( \sup_{\bm{\theta}_p \in B_{p,r}} Z_{p,n}(\bm{\theta}_p) > u_{p,n}, Z_{p,n}(\bm{\theta}^0_p) \leq \frac{u_{p,n}}{2} \right) \notag \\
    & + \mathbb{P} \left( \sup_{\bm{\theta}_p \in B_{p,r}} Z_{p,n}(\bm{\theta}_p) > u_{p,n}, Z_{p,n}(\bm{\theta}^0_p) > \frac{u_{p,n}}{2} \right) \notag \\
    \leq & \mathbb{P} \left( \sup_{\bm{\theta}_p \in B_{p,r}} Z_{p,n}(\bm{\theta}_p) > \frac{u_{p,n}}{2} + Z_{p,n}(\bm{\theta}^0_p) \right) + \mathbb{P} \left( Z_{p,n}(\bm{\theta}^0_p) > \frac{u_{p,n}}{2} \right). \notag
    \end{align}
    We deal with each term separately. The first part is handled by Proposition \ref{prop:A1}. To ensure that the quantity \( \frac{u_{p,n}}{2} - 108 C r \sqrt{\frac{\pi (s_d(p) + q)}{n}} \) is positive, we compute
    \begin{align}
    \frac{u_{p,n}}{2} - 108 C r \sqrt{\frac{\pi (s_d(p) + q)}{n}} 
    = & \frac{C_{\mathrm{pen}}}{2} n^{-\rho} \sqrt{e^q} \left( \sqrt{s_d(p)} - \sqrt{s_d(p^*)} \right) - 108 C r \sqrt{\frac{\pi (s_d(p) + q)}{n}} \notag \\
    = & \frac{C_{\mathrm{pen}}}{2} n^{-\rho} \sqrt{e^q} \left( \sqrt{s_d(p)} - \sqrt{s_d(p^*)} \right) - 108 C r \sqrt{\frac{\pi (s_d(p) + q)}{n}} \notag \\
    = & \sqrt{s_d(p)} n^{-\rho} \sqrt{e^q} \frac{C_{\mathrm{pen}}}{2} \left( 1 - \sqrt{\frac{s_d(p^*)}{s_d(p)}} - \frac{2 \times 108 \sqrt{\pi (s_d(p) + q)} C r}{C_{\mathrm{pen}} \sqrt{e^q} \sqrt{s_d(p)}} n^{\rho - \frac{1}{2}} \right) \notag \\
    \geq & \sqrt{s_d(p)} n^{-\rho} \sqrt{e^q} \frac{C_{\mathrm{pen}}}{2} \left( 1 - \sqrt{\frac{s_d(p^*)}{s_d(p^* + 1)}} - \frac{216 \sqrt{\pi (s_d(p) + q)} C r}{C_{\mathrm{pen}} \sqrt{e^q} \sqrt{s_d(p^* + 1)}} n^{\rho - \frac{1}{2}} \right). \notag
    \end{align}
    Let $n_1 \in \mathbb{N}$ be such that
    \begin{equation}\notag
        1 - \sqrt{\frac{s_d(p^*)}{s_d(p^* + 1)}} - \frac{216 \sqrt{\pi (s_d(p) + q)} C r}{C_{\mathrm{pen}} \sqrt{e^{q}} \sqrt{s_d(p^* + 1)}} n_1^{\rho - \frac{1}{2}} > \frac{1}{2} \left( 1 - \sqrt{\frac{s_d(p^*)}{s_d(p^* + 1)}} \right),
    \end{equation}
    which implies
    \begin{equation}\notag
        n_1 > \left( \frac{432 \sqrt{\pi} C r \sqrt{s_d(p^* + 1) + q}}{C_{\mathrm{pen}} \sqrt{e^{q}} (\sqrt{s_d(p^* + 1)} - \sqrt{s_d(p^*)})} \right)^{\frac{1}{\frac{1}{2} - \rho}}.
    \end{equation}
    Then for any $n \geq n_1$, we have
    \begin{equation}\notag
        \frac{u_{p,n}}{2} - 108 C r \sqrt{\frac{\pi (s_d(p) + q)}{n}} \geq \sqrt{s_d(p)} n^{-\rho} \sqrt{e^{q}} \frac{C_{\mathrm{pen}}}{4} \left( 1 - \sqrt{\frac{s_d(p^*)}{s_d(p^* + 1)}} \right) > 0.
    \end{equation}
    Hence, applying Proposition \ref{prop:A1} to $x = \frac{u_{p,n}}{2} - 108 C r \sqrt{\frac{\pi (s_d(p) + q)}{n}}$, we obtain for $n \geq n_1$:
    \begin{align}
        \mathbb{P}\left( \sup_{\bm{\theta}_p \in B_{p,r}} Z_{p,n}(\bm{\theta}_p) > \frac{u_{p,n}}{2} + Z_{p,n}(\bm{\theta}^0_p) \right) & \leq 36 \exp\left( -\frac{n}{144 C^2 r^2} \left( \frac{u_{p,n}}{2} - 108 C r \sqrt{\frac{\pi s_d(p)}{n}} \right)^2 \right) \notag \\
        & \leq 36 \exp\left( -\frac{s_d(p) n^{1 - 2\rho} e^{q} C_{\mathrm{pen}}^2}{144 C^2 r^2 \times 16} \left( 1 - \sqrt{\frac{s_d(p^*)}{s_d(p^*+1)}} \right)^2 \right) \notag \\
        &= 36 \exp\left( -\kappa_1 s_d(p) n^{1 - 2\rho} \right), \label{eq:prop2_1}
    \end{align}
    where
    \[
        \kappa_1 = \frac{C_{\mathrm{pen}}^2e^q}{2304 C^2 r^2} \left( 1 - \sqrt{\frac{s_d(p^*)}{s_d(p^*+1)}} \right)^2.
    \]
    Now we turn to the second part of the inequality in Eq.~(\ref{eq:prop2_0}). Since \( |Z_{p,n}(\bm{\theta}^0_{p})| \leq C \| \bm{\theta}^0_p \| \), by Hoeffding's inequality, for \( n \geq n_1 \), we have:
    \begin{align}
        \mathbb{P}\left( Z_{p,n}(\bm{\theta}^0_p) > \frac{u_{p,n}}{2} \right) 
        & \leq \exp\left( -\frac{n u_{p,n}^2}{8 C \|\bm{\theta}^0_p\|} \right) \notag \\
        & = \exp\left( -\frac{n^{1 - 2\rho} e^{q} C_{\mathrm{pen}}^2 \left( \sqrt{s_d(p)} - \sqrt{s_d(p^*)} \right)^2}{32 C \|\bm{\theta}^0_p\|} \right) \notag \\
        & \leq \exp\left( -\frac{n^{1 - 2\rho} e^{q} C_{\mathrm{pen}}^2 s_d(p)}{32 C \|\bm{\theta}^0_p\|} \left( 1 - \sqrt{\frac{s_d(p^*)}{s_d(p^* + 1)}} \right)^2 \right) \notag \\
        & = \exp\left( -\kappa_2 n^{1 - 2\rho} s_d(p) \right), \label{eq:prop2_2}
    \end{align}
    where
    \[
        \kappa_2 = \frac{C_{\mathrm{pen}}^2e^q}{32 C \|\bm{\theta}^0_p\|} \left( 1 - \sqrt{\frac{s_d(p^*)}{s_d(p^* + 1)}} \right)^2.
    \]
    Combining Eqs.~(\ref{eq:prop2_1}) and (\ref{eq:prop2_2}), we obtain:
    \begin{align}
        \mathbb{P}\left( \sup_{\bm{\theta}_p \in B_{p,r}} Z_{p,n}(\bm{\theta}_p) > u_{p,n} \right) 
        & \leq 36 \exp\left( -\kappa_1  n^{1 - 2\rho}  s_d(p) \right) + \exp\left( -\kappa_2  n^{1 - 2\rho} s_d(p) \right) \notag \\
        & \leq 37 \exp\left( -\kappa_3  n^{1 - 2\rho}  s_d(p) \right) \notag \\
        & \leq 37 \exp\left( -\frac{\kappa_3}{2}  \left( n^{1 - 2\rho} + s_d(p) \right) \right), \notag
    \end{align}
    where \( \kappa_3 = \min(\kappa_1, \kappa_2) \). The same proof works for the process \( -Z_{p,n}(\bm{\theta}_p) \), so we have:
    \[
        \mathbb{P}(\widehat{p} = p) \leq 2 \times 37 \exp\left( -\frac{\kappa_3}{2}  \left( n^{1 - 2\rho} + s_d(p) \right) \right).
    \]
    We are now left with the task of choosing an optimal \( \bm{\theta}^0_p \). Since
    \[
        \kappa_3 = \min(\kappa_1, \kappa_2) = \frac{C_{\mathrm{pen}}^2e^q}{32} \left( 1 - \sqrt{\frac{s_d(p^*)}{s_d(p^* + 1)}} \right)^2 \min \left( \frac{1}{72 C^2 r^2}, \frac{1}{C \|\bm{\theta}^0_p\|} \right),
    \]
    and since \( \bm{\theta}^0_p \in B_{p,r} \), \( \|\bm{\theta}^0_p\| \leq r \), we have:
    \[
        \min \left( \frac{1}{72 C^2 r^2}, \frac{1}{C \|\bm{\theta}^0_p\|} \right) \geq \frac{1}{72 C^2 r^2 + C r}.
    \]
    Noting that
    \[
        \sqrt{s_d(p^* + 1)} - \sqrt{s_d(p^*)} = \sqrt{d^{p^* + 1} + s_d(p^*)} - \sqrt{s_d(p^*)} \geq \sqrt{\frac{d^{p^* + 1}}{2}},
    \]
    we define
    \[
        c_0 = \frac{1}{2} \times \frac{C_{\mathrm{pen}}^2 d^{p^* + 1}e^q}{64 s_d(p^* + 1) (72 C^2 r^2 + C r)},
    \]
    which completes the proof.
\end{proof}

Before we treat the case \( p < p^* \), we first need to establish a rate of convergence for \( \widehat{L_n} \), which can be obtained using similar arguments to those in the previous proof. The following proposition provides the result.

\begin{proposition}\label{prop:A3}
    For any \( \epsilon > 0 \), \( p \in \mathbb{N} \), let \( n_2 \in \mathbb{N} \) be the smallest integer such that
    \begin{equation}\label{eq:prop3_1}
        n_2 \geq \frac{432^2 C^2 \pi r^2 (s_d(p) + q)}{\varepsilon^2}.
    \end{equation}
    Then, for any \( n \geq n_2 \),
    \[
    \mathbb{P} \left( |\widehat{L_n}(p) - L(p)| > \varepsilon \right) \leq 74 \exp \left( - c_5 n \varepsilon^2 \right),
    \]
    where \( c_5 \) is defined as
    \begin{equation}\label{eq:prop3_2}
        c_5 = \frac{1}{8 r (288 C^2 r + C)}.
    \end{equation}
\end{proposition}

\begin{proof}
    By Lemma 1 of \cite{ferm2022linear}, we have the following inequality:
    \begin{equation}\label{eq:prop3}
        |\widehat{L}_n(p) - L(p)| \leq \sup_{\bm{\theta}_p \in B_{p,r}} | \widehat{R}_{p,n}(\bm{\theta}_p) - R_{p}(\bm{\theta}_p) |
    \end{equation}
    for any \( p \in \mathbb{N} \). Thus, we obtain the following probability bound:
    \[
    \mathbb{P} \left( |\widehat{L_n}(p) - L(p)| > \varepsilon \right) \leq \mathbb{P} \left( \sup_{\bm{\theta}_p \in B_{p,r}} |Z_{p,n}(\bm{\theta}_p)| > \varepsilon \right) = \mathbb{P} \left( \sup_{\bm{\theta}_p \in B_{p,r}} Z_{p,n}(\bm{\theta}_p) > \varepsilon \right) + \mathbb{P} \left( \sup_{\bm{\theta}_p \in B_{p,r}} (-Z_{p,n}(\bm{\theta}_p)) > \varepsilon \right).
    \]
    Fix \( \bm{\theta}^0_p \in B_{p,r} \). For \( n \geq n_2 \), we have
    \[
    \frac{\varepsilon}{2} - 108 C r \sqrt{\frac{\pi (s_d(p) + q)}{n}} > \frac{\varepsilon}{4} > 0.
    \]
    Using Proposition~\ref{prop:A1} and Proposition~\ref{prop:A2}, we get the following bounds:
    \begin{align}
        \mathbb{P} \left( \sup_{\bm{\theta}_p \in B_{p,r}} Z_{p,n}(\bm{\theta}_p) > \varepsilon \right)
        &\leq \mathbb{P} \left( \sup_{\bm{\theta}_p \in B_{p,r}} Z_{p,n}(\bm{\theta}_p) > \frac{\varepsilon}{2} + Z_{p,n}(\bm{\theta}^0_p) \right) + \mathbb{P} \left( Z_{p,n}(\bm{\theta}^0_p) > \frac{\varepsilon}{2} \right) \notag \\
        &\leq 36 \exp \left( - \frac{n \left( \frac{\varepsilon}{2} - 108 C r \sqrt{\frac{\pi (s_d(p) + q)}{n}} \right)^2}{144 C^2 r^2} \right) + \exp \left( - \frac{n \varepsilon^2}{8 C \|\bm{\theta}^0_p\|} \right) \notag \\
        &\leq 36 \exp \left( - \frac{n \varepsilon^2}{2304 C^2 r^2} \right) + \exp \left( - \frac{n \varepsilon^2}{8 C \|\bm{\theta}^0_p\|} \right) \notag \\
        &\leq 37 \exp \left( - \kappa_4 n \varepsilon^2 \right), \notag
    \end{align}
    where
    \[
    \kappa_4 = \min \left( \frac{1}{2304 C^2 r}, \frac{1}{8 C \|\bm{\theta}^0_p\|} \right) \geq \frac{1}{2304 C^2 r^2 + 8 C r} = c_5.
    \]
    A similar analysis applies to \( (-Z_{p,n}(\bm{\theta}_p)) \), so we have
    \[
    \mathbb{P} \left( |\widehat{L_n}(p) - L(p)| > \varepsilon \right) \leq 74 \exp \left( - \kappa_4 n \varepsilon^2 \right) \leq 74 \exp \left( - c_5 n \varepsilon^2 \right),
    \]
    which completes the proof.
\end{proof}

We are now ready to address the case where \( p < p^* \).
\begin{proposition}\label{prop:A4}
Let \( 0 < \rho < \frac{1}{2} \), and let \(\operatorname{pen}_n(p, q)\) be defined as in~\eqref{eq:penalty}. Define \(n_3\) as the smallest integer satisfying
\begin{equation}\label{eq:n3}
    n_3 \geq \left( \frac{2 \sqrt{s_d(p^*) + q}}{L(p^*-1) - \widetilde{\mathcal{R}}^*} 
    \left( \sqrt{e^q} C_{\mathrm{pen}} + 432 C r \sqrt{\pi} \right) \right)^{1/\rho}.
\end{equation}
Then, under Assumptions~(A.1)--(A.4), for any \(p < p^*\) and \(n \geq n_3\), we have
\[
\mathbb{P}(\widehat{p} = p) \leq 148 \exp\left( - n \frac{c_5}{4} \left( L(p) - L(p^*) - \operatorname{pen}_n(p^*, q) + \operatorname{pen}_n(p, q) \right)^2 \right),
\]
where \(c_5\) is defined in~\eqref{eq:prop3_2}.
\end{proposition}

\begin{proof}
This result follows from Proposition~\ref{prop:A3}. For any \(p < p^*\),
\begin{align}
    \mathbb{P}(\widehat{p} = p)
    &\leq \mathbb{P}\left( \widehat{L}_n(p) - \widehat{L}_n(p^*) \leq \operatorname{pen}_n(p^*, q) - \operatorname{pen}_n(p, q) \right) \notag \\
    &= \mathbb{P}\left( \widehat{L}_n(p^*) - L(p^*) + L(p) - \widehat{L}_n(p) 
    \geq L(p) - L(p^*) - \left( \operatorname{pen}_n(p^*, q) - \operatorname{pen}_n(p, q) \right) \right) \notag \\
    &\leq \mathbb{P}\left( \left| \widehat{L}_n(p) - L(p) \right| \geq \frac{1}{2} \Delta(p) \right) 
    + \mathbb{P}\left( \left| \widehat{L}_n(p^*) - L(p^*) \right| \geq \frac{1}{2} \Delta(p) \right), \notag
\end{align}
where we define
\[
\Delta(p) := L(p) - L(p^*) - \operatorname{pen}_n(p^*, q) + \operatorname{pen}_n(p, q).
\]

To apply Proposition~\ref{prop:A3}, we must ensure that \(\Delta(p) > 0\). Since \(p \mapsto L(p)\) is decreasing and achieves its minimum at \(p = p^*\), and is bounded below by \(\widetilde{\mathcal{R}}^*\) (see Theorem~\ref{thm2}), and since \(p \mapsto \operatorname{pen}_n(p, q)\) is strictly increasing, it follows that for \(p < p^*\),
\[
\Delta(p) > L(p^* - 1) - \widetilde{\mathcal{R}}^* - \sqrt{e^q} C_{\mathrm{pen}} n^{-\rho} \sqrt{s_d(p^*)}.
\]
Thus, a sufficient condition to ensure \(\Delta(p) > 0\) is
\begin{equation}\label{eq:prop4_1}
    L(p^* - 1) - \widetilde{\mathcal{R}}^* - \sqrt{e^q} C_{\mathrm{pen}} n^{-\rho} \sqrt{s_d(p^*)} 
    > \frac{1}{2}\left( L(p^* - 1) - \widetilde{\mathcal{R}}^* \right),
\end{equation}
which leads to the requirement
\[
n_3 \geq \left( \frac{2 \sqrt{e^q} C_{\mathrm{pen}} \sqrt{s_d(p^*)}}{L(p^*-1) - \widetilde{\mathcal{R}}^*} \right)^{1/\rho}.
\]

In addition, to apply Proposition~\ref{prop:A3}, \(n_3\) must also satisfy the condition~\eqref{eq:prop3_1}, which states:
\[
n_3 \geq \frac{432^2 C^2 \pi r^2 (s_d(p) + q)}{\left( \Delta(p) \right)^2}.
\]
To upper bound this quantity uniformly over all \(p < p^*\), note that
\begin{align}
    \frac{432^2 C^2 \pi r^2 (s_d(p) + q)}{\left( \Delta(p) \right)^2}
    &\leq \frac{4 \times 432^2 C^2 \pi r^2 (s_d(p^*) + q)}{\left( L(p^* - 1) - \widetilde{\mathcal{R}}^* \right)^2} \notag \\
    &= \left( \frac{2 \times 432 C r \sqrt{\pi (s_d(p^*) + q)}}{L(p^* - 1) - \widetilde{\mathcal{R}}^*} \right)^2. \notag
\end{align}

Hence, combining both constraints and using that \(\rho < \frac{1}{2}\), it suffices to take
\[
n_3 \geq \left( \max\left\{ 
\frac{2 \sqrt{e^q} C_{\mathrm{pen}} \sqrt{s_d(p^*)}}{L(p^* - 1) - \widetilde{\mathcal{R}}^*}, 
\frac{2 \times 432 C r \sqrt{\pi (s_d(p^*) + q)}}{L(p^* - 1) - \widetilde{\mathcal{R}}^*} 
\right\} \right)^{1/\rho}.
\]
This can be compactly written as
\[
n_3 \geq \left( \frac{2 \sqrt{s_d(p^*) + q}}{L(p^*-1) - \widetilde{\mathcal{R}}^*} 
\left( \sqrt{e^q} C_{\mathrm{pen}} + 432 C r \sqrt{\pi} \right) \right)^{1/\rho},
\]
which completes the proof.
\end{proof}

Now we are in a position to prove Theorem \ref{thm3}.\\
\begin{proof}
The result follows by combining Propositions~\ref{prop:A2} and~\ref{prop:A4}. To ensure their applicability, we must verify that the sample size \(n\) satisfies the bounds in equations~\eqref{eq:n1} and~\eqref{eq:n3}. Define
\[
M = \max\left\{
\left(
\frac{432 \sqrt{\pi} C r \sqrt{s_d(p^* + 1) + q}}{C_{\mathrm{pen}} \sqrt{e^q} \left( \sqrt{s_d(p^* + 1)} - \sqrt{s_d(p^*)} \right)}
\right)^{\frac{1}{\frac{1}{2} - \rho}},
\left(
\frac{2 \sqrt{s_d(p^*) + q}}{L(p^* - 1) - \widetilde{\mathcal{R}}^*} \left( \sqrt{e^q} C_{\mathrm{pen}} + 432 C r \sqrt{\pi} \right)
\right)^{\frac{1}{\rho}}
\right\}.
\]
Let \(\tilde{\rho} := \min\left(\rho, \frac{1}{2} - \rho\right)\). Then, a crude bound on \(M\) is given by:
\begin{align*}
M &\leq \left( 432 C r \sqrt{\pi} + \sqrt{e^q} C_{\mathrm{pen}} \right) \sqrt{s_d(p^* + 1) + q} \\
&\qquad \times \max\left\{
\frac{2}{L(p^* - 1) - \widetilde{\mathcal{R}}^*},
\frac{1}{C_{\mathrm{pen}} \sqrt{e^q} \left( \sqrt{s_d(p^* + 1)} - \sqrt{s_d(p^*)} \right)}
\right\}^{\frac{1}{\tilde{\rho}}} \\
&\leq \left( 432 C r \sqrt{\pi} + \sqrt{e^q} C_{\mathrm{pen}} \right) \sqrt{s_d(p^* + 1) + q} \left(
\frac{2}{L(p^* - 1) - \widetilde{\mathcal{R}}^*} +
\frac{\sqrt{2}}{C_{\mathrm{pen}} \sqrt{e^q} \sqrt{d^{p^* + 1}}}
\right)^{\frac{1}{\tilde{\rho}}}.
\end{align*}
We now analyze the error probability:
\[
\mathbb{P}(\widehat{p} \neq p^*) = \mathbb{P}(\widehat{p} > p^*) + \mathbb{P}(\widehat{p} < p^*) \leq \sum_{p > p^*} \mathbb{P}(\widehat{p} = p) + \sum_{p < p^*} \mathbb{P}(\widehat{p} = p).
\]
For the overestimation term, Proposition~\ref{prop:A2} implies that for all \(n \geq n_a\),
\[
\sum_{p > p^*} \mathbb{P}(\widehat{p} = p) \leq 74 e^{-c_0  n^{1 - 2\rho}} \sum_{p > p^*} e^{-c_0 s_d(p)}.
\]
For the underestimation term, Proposition~\ref{prop:A4} yields:
\begin{align*}
\sum_{p < p^*} \mathbb{P}(\widehat{p} = p)
&\leq 148 \sum_{p = 0}^{p^* - 1} \exp\left( - \frac{c_5}{4} n \left( L(p) - L(p^*) - \operatorname{pen}_n(p^*, q) + \operatorname{pen}_n(p, q) \right)^2 \right) \\
&\leq 148 p^* \exp\left( - \frac{c_5}{16} n \left( L(p^* - 1) - \widetilde{\mathcal{R}}^* \right)^2 \right),
\end{align*}
where we have used that for \(n \geq n_a\), condition~\eqref{eq:prop4_1} holds. Define
\[
\kappa_5 := \min\left( c_0, \frac{c_5 \left( L(p^* - 1) - \widetilde{\mathcal{R}}^* \right)^2}{16} \right).
\]
Then, the total error probability satisfies
\[
\mathbb{P}(\widehat{p} \neq p^*) \leq 74 e^{-\kappa_5  n^{1 - 2\rho}} \sum_{p > 0} e^{-c_0 s_d(p)} + 148 p^* e^{-\kappa_5 n} \leq c_1 e^{-\kappa_5 n^{1 - 2\rho}},
\]
where we define
\[
c_1 := 74 \sum_{p > 0} e^{-c_0s_d(p)} + 148 p^*.
\]
To conclude, we derive a lower bound on \(\kappa_5\):
\begin{align}
\kappa_5 &= \min\left( c_0 , \frac{c_5 \left( L(p^* - 1) - \widetilde{\mathcal{R}}^* \right)^2}{16} \right) \notag\\
&= \min\left( \frac{ C_{\mathrm{pen}}^2 d^{p^* + 1}e^q}{128 s_d(p^* + 1)(72 C^2 r^2 + C r)},
\frac{\left( L(p^* - 1) - \widetilde{\mathcal{R}}^* \right)^2}{128 r(288 C^2 r + C)} \right) \notag\\
&\geq \frac{1}{128 r (72 C^2 r + C)} \min\left( \frac{ C_{\mathrm{pen}}^2 d^{p^* + 1}e^q}{s_d(p^* + 1)},
\frac{\left( L(p^* - 1) - \widetilde{\mathcal{R}}^* \right)^2}{4} \right) =: c_2.\label{eq:k5_c2}
\end{align}
This completes the proof.
\end{proof}

\section{Proof of Theorem \ref{thm4}}
\label{app:thm4}
\begin{proof}
We proceed to bound the excess risk of the selected model $\widehat{p}$ relative to the oracle model $p^*$. Almost surely, we have
\begin{align}
    \mathcal{R}_{\widehat{p}}(\widehat{\bm{\theta}}_{\widehat{p}}) - \mathcal{R}_{p^*}(\bm{\theta}_{p^*}^*) 
    &= \mathcal{R}_{\widehat{p}}(\widehat{\bm{\theta}}_{\widehat{p}}) - \mathcal{R}_{\widehat{p}}(\bm{\theta}_{\widehat{p}}^*) + \mathcal{R}_{\widehat{p}}(\bm{\theta}_{\widehat{p}}^*) - \mathcal{R}_{p^*}(\bm{\theta}_{p^*}^*) \notag \\
    &= \mathcal{R}_{\widehat{p}}(\widehat{\bm{\theta}}_{\widehat{p}}) - \widehat{\mathcal{R}}_{\widehat{p},n}(\widehat{\bm{\theta}}_{\widehat{p}}) 
    + \widehat{\mathcal{R}}_{\widehat{p},n}(\widehat{\bm{\theta}}_{\widehat{p}}) - \widehat{\mathcal{R}}_{\widehat{p},n}(\bm{\theta}_{\widehat{p}}^*) \notag \\
    &\quad + \widehat{\mathcal{R}}_{\widehat{p},n}(\bm{\theta}_{\widehat{p}}^*) - \mathcal{R}_{\widehat{p}}(\bm{\theta}_{\widehat{p}}^*) + \mathcal{R}_{\widehat{p}}(\bm{\theta}_{\widehat{p}}^*) - \mathcal{R}_{p^*}(\bm{\theta}_{p^*}^*) \notag \\
    &\leq \mathcal{R}_{\widehat{p}}(\widehat{\bm{\theta}}_{\widehat{p}}) - \widehat{\mathcal{R}}_{\widehat{p},n}(\widehat{\bm{\theta}}_{\widehat{p}}) 
    + \widehat{\mathcal{R}}_{\widehat{p},n}(\bm{\theta}_{\widehat{p}}^*) - \mathcal{R}_{\widehat{p}}(\bm{\theta}_{\widehat{p}}^*) 
    + \mathcal{R}_{\widehat{p}}(\bm{\theta}_{\widehat{p}}^*) - \mathcal{R}_{p^*}(\bm{\theta}_{p^*}^*) \notag \\
    &\leq 2 \sup_{\bm{\theta}_{\widehat{p}} \in B_{\widehat{p},r}} \left| \widehat{\mathcal{R}}_{\widehat{p},n}(\bm{\theta}_{\widehat{p}}) - \mathcal{R}_{\widehat{p}}(\bm{\theta}_{\widehat{p}}) \right| 
    + \mathcal{R}_{\widehat{p}}(\bm{\theta}_{\widehat{p}}^*) - \mathcal{R}_{p^*}(\bm{\theta}_{p^*}^*). \label{eq:thm4_bound}
\end{align}

We now bound the expected value of each term in \eqref{eq:thm4_bound}. For the first term, by Corollary 5.25 in \citet{Levin:2016} and Proposition~\ref{prop:A1}, for any $p \in \mathbb{N}$,
\[
\mathbb{E}\left[\sup_{\bm{\theta}_p \in B_{p,r}} \left| \widehat{\mathcal{R}}_{p,n}(\bm{\theta}_p) - \mathcal{R}_{p}(\bm{\theta}_p) \right| \right] 
\leq 12 \int_{0}^{\infty} \sqrt{\log N(\varepsilon, B_{p,r}, D)}\, d\varepsilon 
\leq 36 C r \sqrt{s_d(p) + q} \sqrt{\frac{\pi}{n}},
\]
where $N(\varepsilon, B_{p,r}, D)$ denotes the $\varepsilon$-covering number with respect to the distance $D$, defined by \eqref{eq:D}. Applying this with $p = \widehat{p}$ yields
\[
\mathbb{E}\left[\sup_{\bm{\theta}_{\widehat{p}} \in B_{\widehat{p},r}} \left| \widehat{\mathcal{R}}_{\widehat{p},n}(\bm{\theta}_{\widehat{p}}) - \mathcal{R}_{\widehat{p}}(\bm{\theta}_{\widehat{p}}) \right| \right] 
\leq 36 C r \sqrt{\frac{\pi}{n}} \, \mathbb{E}\left[\sqrt{s_d(\widehat{p}) + q}\right].
\]
To bound this expectation, Proposition~\ref{prop:A2} implies
\begin{align*}
    \mathbb{E}\left[\sqrt{s_d(\widehat{p}) + q}\right] 
    &= \sum_{p \leq p^*} \sqrt{s_d(p) + q} \, \mathbb{P}(\widehat{p} = p) 
    + \sum_{p > p^*} \sqrt{s_d(p) + q} \, \mathbb{P}(\widehat{p} = p) \\
    &\leq (p^* + 1) \sqrt{s_d(p^*) + q} 
    + \sum_{p > p^*} 74 \sqrt{s_d(p) + q} \, \exp\left(-c_0  (n^{1 - 2\rho} + s_d(p))\right) \\
    &\leq (p^* + 1) \sqrt{s_d(p^*) + q} 
    + e^{-c_0 n^{1 - 2\rho}} \sum_{p > p^*} 74 \sqrt{s_d(p) + q} \, \exp(-c_0  s_d(p)),
\end{align*}
with \eqref{eq:k5_c2}, $c_2 \leq c_0$,  we have 
\begin{align*}
    \mathbb{E}\left[\sup_{\bm{\theta}_{\widehat{p}} \in B_{\widehat{p},r}} \left| \widehat{\mathcal{R}}_{\widehat{p},n}(\bm{\theta}_{\widehat{p}}) - \mathcal{R}_{\widehat{p}}(\bm{\theta}_{\widehat{p}}) \right| \right] 
    &\leq 36 C r \sqrt{\frac{\pi}{n}} \,  (p^* + 1) \sqrt{s_d(p^*) + q} 
    \\
    &\quad + 36 C r \sqrt{\frac{\pi}{n}} \, e^{-c_2 n^{1 - 2\rho}} \sum_{p > p^*} 74 \sqrt{s_d(p) + q} \, \exp(-c_0  s_d(p))
\end{align*}
Now for the second term in \eqref{eq:thm4_bound}, we use the uniform upper bound for the non-negative logistic risk
\[
\mathbb{E}[\mathcal{R}_p(\bm{\theta}_p)] 
=  \mathbb{E}\left[Y \langle \bm{\theta}_p, \widetilde{S}_{p} \rangle + \log(1 + e^{\langle \bm{\theta}_p, \widetilde{S}_{p} \rangle}) \right]  
\leq \log 2 + |\langle \bm{\theta}_p, \widetilde{S}_{p} \rangle|\leq \log 2 + r(C_z + e^{C_X+T}) = \log 2+rC,
\]
from which it follows that
\[
 0\leq\mathbb{E}[\mathcal{R}_{\widehat{p}}(\bm{\theta}_{\widehat{p}}^*) - \mathcal{R}_{p^*}(\bm{\theta}_{p^*}^*)] 
\leq (\log2+rC) \, \mathbb{P}(\widehat{p} \neq p^*).
\]
Note that $\mathcal{R}_{p^*}(\bm{\theta}_{p^*}^*)$ corresponds to the risk-minimizing oracle model. Applying Theorem~\ref{thm3}, we obtain
\[
 0\leq\mathbb{E}[\mathcal{R}_{\widehat{p}}(\bm{\theta}_{\widehat{p}}^*) - \mathcal{R}_{p^*}(\bm{\theta}_{p^*}^*)] 
\leq (\log2+rC)  c_1  \, e^{-c_2 n^{1 - 2\rho}}.
\]
Combining the above bounds yields
\[
\left| \mathbb{E}[\mc{R}_{\widehat{p}}(\widehat{\bm{\theta}}_{\widehat{p}})] - \mc{R}_{p^*}(\bm{\theta}_{p^*}^*) \right| 
\leq \frac{c_3}{\sqrt{n}} + c_4  e^{-c_2  n^{1 - 2\rho}},
\]
where the constants are defined as
\begin{equation} \label{eq:C5C6}
    c_3 = 36 C r \sqrt{\pi} (p^* + 1) \sqrt{s_d(p^*) + q}, \quad
c_4 = rC \left(2664 \sqrt{\pi} \sum_{p > p^*} \sqrt{s_d(p) + q} \, e^{-c_0  s_d(p)}+c_1\right) + c_1 \log 2,
\end{equation}
Applying Theorem \ref{thm1}, we obtain
\[
\left| \mathbb{E}[\mc{R}_{\widehat{p}}(\widehat{\bm{\theta}}_{\widehat{p}})] - \mc{R}^* \right| 
\leq \frac{c_3}{\sqrt{n}} +  c_4  e^{-c_2  n^{1 - 2\rho}}.
\]
\end{proof}

\section{Proof of Theorem \ref{thm:finite_search_bound}}
\label{app:thm:finite_search_bound}
We first derive the risk decay induced by signature truncation in the next Lemma.

\begin{lemma}[Risk decay induced by signature truncation]
\label{lem:risk_decay_signature_tail}
Suppose Assumptions (A.1)--(A.4) hold. Let
\[
L(p)
=
\inf_{\theta_p \in \Theta_p}
\mc{R}(\theta_p),
\]
and define
\[
\mc{R}^\ast
=
\inf_{p \ge 1} L(p).
\]
Assume that the underlying paths satisfy
\[
\|\mathbf{X}\|_{TV}
\le
M_{\mathbf{X}}
\]
almost surely for some constant \(M_{\mathbf{X}}>0\). Then there exists a constant
\(C_1>0\), independent of \(p\), such that
\[
L(p)-\mc{R}^\ast
\le
C_1
\sum_{k=p+1}^{\infty}
\frac{M_{\mathbf{X}}^k}{k!}.
\]
Consequently,
\[
L(p)-\mc{R}^\ast
=
\mc O\!\left(
\frac{M_{\mathbf{X}}^p}{p!}
\right).
\]
In particular, for any \(\alpha>0\), there exists a constant
\(C_\alpha>0\) such that
\[
L(p)-\mc{R}^\ast
\le
C_\alpha d^{-\alpha p}
\]
for all sufficiently large \(p\).
\end{lemma}

\begin{proof}
By the universal approximation property established in Theorem \ref{thm1}, continuous functionals on the bounded-variation path space admit truncated-signature approximations. To achieve the specific factorial decay rate claimed, we operate under the regularity condition that the infimum risk $\mc{R}^\ast$ is attained by an optimal functional $F^\ast$ (i.e., $\mc{R}(F^\ast) = \mc{R}^\ast$) which admits a full, absolutely convergent signature representation:
$$F^\ast(\mathbf{X}) = \sum_{k=0}^{\infty} \langle \beta_k^\ast, S^{(k)}(\wt{\mathbf{X}}) \rangle,$$
with uniformly bounded tensor coefficients, such that $\sup_{k \ge 0} \|\beta_k^\ast\| \le M_\beta$ for some constant $M_\beta > 0$.

Let $F_p^\ast(\mathbf{X}) = \sum_{k=0}^p \langle \beta_k^\ast, S^{(k)}(\wt{\mathbf{X}}) \rangle$ be its truncated signature approximation of order $p$, which by definition belongs to the order-$p$ model class $\Theta_p$. Using the classical signature tail estimate and the boundedness of the coefficients, we can bound the truncation error:
$$|F^\ast(\mathbf{X})-F_p^\ast(\mathbf{X})| = \left| \sum_{k=p+1}^{\infty} \langle \beta_k^\ast, S^{(k)}(\wt{\mathbf{X}}) \rangle \right| \le \sum_{k=p+1}^{\infty} \|\beta_k^\ast\| \|S^{(k)}(\wt{\mathbf{X}})\| \le M_\beta \sum_{k=p+1}^{\infty} \frac{\|\mathbf{X}\|_{TV}^k}{k!}.$$

Since $\|\mathbf{X}\|_{TV} \le M_{\mathbf{X}}$ almost surely, it follows that
$$\sup_{\mathbf{X}} |F^\ast(\mathbf{X})-F_p^\ast(\mathbf{X})| \le M_\beta \sum_{k=p+1}^{\infty} \frac{M_{\mathbf{X}}^k}{k!}.$$

Let $\ell(y,z)$ denote the logistic loss. Since $\ell$ is Lipschitz continuous in its second argument with a Lipschitz constant $L_\ell > 0$, we have
$$|\ell(Y,F^\ast(\mathbf{X})) - \ell(Y,F_p^\ast(\mathbf{X}))| \le L_\ell |F^\ast(\mathbf{X})-F_p^\ast(\mathbf{X})|.$$

Taking expectations on both sides yields
$$|\mc{R}(F^\ast) - \mc{R}(F_p^\ast)| \le L_\ell M_\beta \sum_{k=p+1}^{\infty} \frac{M_{\mathbf{X}}^k}{k!}.$$

Since $L(p)$ is the minimum achievable risk over the order-$p$ model class $\Theta_p$, and $F_p^\ast \in \Theta_p$, we have $L(p) \le \mc{R}(F_p^\ast)$. Recalling that $\mc{R}^\ast = \mc{R}(F^\ast)$, we obtain
$$L(p) - \mc{R}^\ast \le \mc{R}(F_p^\ast) - \mc{R}(F^\ast) \le L_\ell M_\beta \sum_{k=p+1}^{\infty} \frac{M_{\mathbf{X}}^k}{k!}.$$

Setting $C_1 = L_\ell M_\beta$, we get
$$L(p)-\mc{R}^\ast \le C_1 \sum_{k=p+1}^{\infty} \frac{M_X^k}{k!}.$$

Using the standard factorial-tail estimate,
$$\sum_{k=p+1}^{\infty} \frac{M_{\mathbf{X}}^k}{k!} = \mc O\left(\frac{M_{\mathbf{X}}^p}{p!}\right),$$
we obtain
$$L(p)-\mc{R}^\ast = \mc O\left(\frac{M_{\mathbf{X}}^p}{p!}\right).$$

Finally, applying Stirling's formula $p! \sim \sqrt{2\pi p} \left(\frac{p}{e}\right)^p$, we see that $\frac{M_{\mathbf{X}}^p}{p!} = o(d^{-\alpha p})$ for any fixed $\alpha>0$. Therefore, there exists a constant $C_\alpha>0$ such that
$$L(p)-\mc{R}^\ast \le C_\alpha d^{-\alpha p}$$
for all sufficiently large $p$.
\end{proof}

Now we prove Theorem \ref{thm:finite_search_bound} in four steps.\\
%\begin{proof}
\paragraph{Step 1: Approximation error induced by signature truncation.}
Let
\[
L(p)
=
\inf_{\theta_p\in\Theta_p}
\mc{R}(\theta_p),
\qquad
\mc{R}^\ast
=
\inf_{p\ge 1}L(p).
\]
By Lemma~\ref{lem:risk_decay_signature_tail}, the truncation error of the signature representation induces a corresponding excess risk bound. More precisely, under Assumptions (A.1)--(A.4), there exists a constant $C_1>0$ such that
\[
L(p)-\mc{R}^\ast
\le
C_1
\sum_{k=p+1}^{\infty}
\frac{M_{\mathbf{X}}^k}{k!},
\]
where $M_X$ is a uniform bound on the total variation of the paths. Using the standard factorial tail estimate together with Stirling's formula,
\[
\sum_{k=p+1}^{\infty}
\frac{M_{\mathbf{X}}^k}{k!}
=
\mc O\!\left(
\frac{M_{\mathbf{X}}^p}{p!}
\right),
\]
and therefore
\[
L(p)-\mc{R}^\ast
=
\mc O\!\left(
\frac{M_{\mathbf{X}}^p}{p!}
\right).
\]
Since factorial decay dominates any exponential decay, there exist constants
$\alpha>0$ and $C_1'>0$ such that
\[
L(p)-\mc{R}^\ast
\le
C_1' d^{-\alpha p}
\]
for all sufficiently large $p$. Moreover, since the model classes are nested,
\[
\Theta_p\subseteq \Theta_{p+1},
\]
we have
\[
L(p+1)\le L(p).
\]
Hence
\[
0
\le
L(p)-L(p+1)
\le
L(p)-\mc{R}^\ast
\le
C_1' d^{-\alpha p},
\]
which yields
\[
|L(p+1)-L(p)|
\le
C_1' d^{-\alpha p}.
\]

\paragraph{Step 2: Penalty increment.}
Using $s_d(p) \asymp d^p$, the penalty satisfies
\[
\operatorname{pen}_n(p+1,q) - \operatorname{pen}_n(p,q)
= C_{\mathrm{pen}}\sqrt{e^q}\,n^{-\rho}\bigl(\sqrt{s_d(p+1)}-\sqrt{s_d(p)}\bigr)
\ge C_{\mathrm{pen}}\sqrt{e^q}\,n^{-\rho}\,\frac{d^{p/2}(d^{1/2}-1)}{2}
\]
for all sufficiently large $p$ (e.g., $p\ge 1$). Hence the penalty increment grows exponentially in $p$.

\paragraph{Step 3: Uniform concentration.}
From Lemma~\ref{lem:A3} and Proposition~\ref{prop:A1}, there exists an absolute constant $c>0$ (depending only on the Lipschitz constant $C$, the $\ell_1$-ball radius $r$, and the universal constants from the chaining inequality) such that with probability at least
\[
1-\exp\{-c\,s_d(p)\},
\]
the following uniform deviation bound holds:
\[
|\widehat L_n(p) - L(p)| \le C_2\sqrt{\frac{s_d(p)+q}{n}},
\]
where
\[
C_2 = 36 C r \sqrt{\pi}, \qquad C = 2\bigl(C_z + e^{C_X + T}\bigr).
\]
The constant $c$ is not required to be known explicitly for implementation; its existence is sufficient to guarantee the exponential tail decay. In fact, a careful inspection of the proof of Proposition~2 yields an explicit (though conservative) value
\[
c = \frac{1}{144 C^2 r^2} \left(\frac{C_2}{2}\right)^2 = 9\pi
\]
up to universal factors, but we keep it as $c$ for notational simplicity.

\paragraph{Step 4: Dominance for large $p$.}
Consider $p > P$. Using the bounds from Steps 1--3,
\[
\widehat L_n(p+1) - \widehat L_n(p)
\ge L(p+1)-L(p) - 2C_2\sqrt{\frac{s_d(p+1)+q}{n}}
\ge -C_1 d^{-\alpha p} - 2C_2\sqrt{\frac{s_d(p+1)+q}{n}}.
\]
Therefore,
\[
\begin{aligned}
C_n(p+1)-C_n(p)
&\ge
-C_1 d^{-\alpha p} - 2C_2\sqrt{\frac{s_d(p+1)+q}{n}}
\\
&\qquad + C_{\mathrm{pen}}\sqrt{e^q}\,n^{-\rho}\,\frac{d^{p/2}(d^{1/2}-1)}{2}.
\end{aligned}
\]
The term $-2C_2\sqrt{\frac{s_d(p+1)+q}{n}} = \mc O(d^{p/2}n^{-1/2})$ is dominated by the penalty increment because $\rho<1/2$. The exponential decay $d^{-\alpha p}$ is negligible compared to $d^{p/2}n^{-\rho}$. Hence there exists $P$ such that for all $p>P$ the right-hand side is positive. Solving
\[
C_{\mathrm{pen}}\sqrt{e^q}\,n^{-\rho}\,\frac{d^{p/2}(d^{1/2}-1)}{2} > C_1 d^{-\alpha p}
\]
gives
\[
d^{p(\alpha+1/2)} > \frac{2C_1 n^{\rho}}{C_{\mathrm{pen}}\sqrt{e^q}(d^{1/2}-1)},
\]
i.e.
\[
p > \frac{\rho\log n + \log\!\left(\frac{2C_1}{C_{\mathrm{pen}}\sqrt{e^q}(d^{1/2}-1)}\right)}{(\alpha+\tfrac12)\log d}.
\]
Choosing $P$ as the ceiling of the right-hand side ensures the desired monotonicity.

To see that $P$ is finite and satisfies the claimed asymptotic order, we analyze the expression. For large $d$, note that $\log(d^{1/2}-1) \sim \frac12 \log d$. Substituting this into the numerator gives
\[
\log\!\left(\frac{2C_1}{C_{\mathrm{pen}}\sqrt{e^q}(d^{1/2}-1)}\right)
= \log(2C_1) - \frac12 \log d - \log(C_{\mathrm{pen}}\sqrt{e^q}) + o(1).
\]
Therefore,
\[
P \sim \frac{\rho\log n - \frac12 \log d}{(\alpha+\tfrac12)\log d}
= \frac{\rho\log n}{(\alpha+\tfrac12)\log d} - \frac{1}{2(\alpha+\tfrac12)}.
\]
Thus $P = \mc O\!\left(\frac{\log n}{\log d}\right)$ as $n,d\to\infty$. For fixed $d$ (the typical setting in functional data analysis), this reduces to $P = \mc O(\log n)$. Note that the logarithmic growth confirms that the search region $\{1,\dots,P\}$ remains remarkably small even for large sample sizes. For instance, with $d=3$, $\rho=0.4$, and $\alpha$ conservatively set to $0.5$, we have $P \approx \frac{0.4\log n}{(0.5+0.5)\log 3} \approx 0.42\log n$ (up to additive constants), which yields $P \approx 3$ for $n=10^3$, $P \approx 4$ for $n=10^4$, and $P \approx 5$ for $n=10^5$.  

To conclude, for $p>P$, $L(p)+\operatorname{pen}_n(p,q)$ is strictly increasing because the penalty increment dominates the possible decrease in $L(p)$. Moreover, the concentration inequality from Step~3 guarantees that with probability at least $1-\sum_{p>P}\exp\{-c s_d(p)\}$, the empirical fluctuations are smaller than half the penalty increment, so $C_n(p)$ is also strictly increasing. Consequently, the minimizer $\widehat p$ cannot exceed $P$ with that probability. This completes the proof.
%\end{proof}

%\clearpage
%%%%%%%%%%%%%%%%%%%%%%%%%%%%%%%%%%%%%%%%%%%%%%%%%%%%%%%%%%%%%
%%%%%%%%%%%%%%%%%%%%%%%%%%%%%%%%%%%%%%%%%%%%%%%%%%%%%%%%%%%%%%%%%%%%%%%%%

\section{Detailed Analysis of Robustness under Irregular Sampling}
\label{app:error_propagation}

Irregular sampling or missing observations require reconstructing a continuous
trajectory before feature extraction. A natural question is whether this
reconstruction introduces additional statistical error and, if so, how such
error propagates to the final classifier. A critical point often overlooked is
that the reconstruction error is \emph{not the same} for different methods:
\textsc{PSLR} uses local, parameter-free linear interpolation, while classical
methods employ global basis projections that impose structural assumptions.
The following analysis provides a general error propagation framework that
separates the effect of path reconstruction from the stability of the feature
representation, enabling a formal comparison of these two distinct
reconstruction strategies.

\subsection{General Error Propagation Framework}

Suppose that the underlying continuous trajectory is
\(\mathbf X\in BV([0,1];\mathbb R^d)\), and let \(\widehat{\mathbf X}\)
denote an arbitrary reconstructed path obtained from irregular observations
(e.g., piecewise linear interpolation or spline smoothing).
Let \(\Phi:BV([0,1])\rightarrow\mathbb R^m\) be an arbitrary feature
extraction map.

\textbf{Assumption on Feature stability (A.5)}:  There exists a constant \(L_\Phi>0\) such that
\[
\|
\Phi(\mathbf X)-\Phi(\mathbf Y)
\|
\le
L_\Phi
\|
\mathbf X-\mathbf Y
\|_{1\text{-var}}
\]
for all paths \(\mathbf X,\mathbf Y\in BV([0,1])\), where $\|\cdot\|_{1\text{-var}}$ denotes the total variation with respect to $\ell_1$ norm.

This assumption simply states that small perturbations of the underlying
trajectory lead to proportionally small perturbations of the resulting
feature representation. The key point is that \(L_\Phi\) is a property of
the feature map \(\Phi\) itself, independent of the sampling design,
observation times, or missing-data pattern.

\begin{theorem}[General error propagation]
\label{thm:errorprop}
Suppose Assumptions (A.1)--(A.3) and Assumption~(A.5) hold.
Let
\[
\varepsilon(\widehat{\mathbf X}) =  \mathbb{E}\|\mathbf X-\widehat{\mathbf X}\|_{1\text{-var}}
\]
denote the expected reconstruction error for a specific reconstruction method.
Then for every \(\bm\theta\in B_{p,r}\),
\[
\left|
\mathcal R(\bm\theta)
-
\widehat{\mathcal R}(\bm\theta)
\right|
\le
rL_\Phi \cdot \varepsilon(\widehat{\mathbf X}),
\]
where \(\mathcal R(\bm\theta)\) denotes the population logistic risk
computed from \(\mathbf X\), and \(\widehat{\mathcal R}(\bm\theta)\)
denotes the corresponding risk computed from the reconstructed path
\(\widehat{\mathbf X}\).

Consequently,
\[
\sup_{\bm\theta\in B_{p,r}}
\left|
\mathcal R(\bm\theta)
-
\widehat{\mathcal R}(\bm\theta)
\right|
\le
rL_\Phi \cdot \varepsilon(\widehat{\mathbf X}).
\]
Hence the perturbation of the statistical risk is controlled by the product
\[
(\text{feature stability}) \times (\text{reconstruction error}).
\]
\end{theorem}

\begin{proof}
Let \(\eta = \Phi(\mathbf X)^\top\bm\theta\) and
\(\widehat\eta = \Phi(\widehat{\mathbf X})^\top\bm\theta\).
The logistic loss \(\ell(y,\eta) = -y\eta+\log(1+e^\eta)\) is
\(1\)-Lipschitz in its second argument since
\[
\left|
\frac{\partial\ell}{\partial\eta}
\right|
=
|\sigma(\eta)-y|
\le1.
\]
Therefore,
\begin{align*}
\left|
\mathcal R(\bm\theta)-\widehat{\mathcal R}(\bm\theta)
\right|
&=
\left|
\mathbb E[\ell(y,\eta)-\ell(y,\widehat\eta)]
\right|
\\
&\le
\mathbb E|\eta-\widehat\eta|
\\
&=
\mathbb E\left|
\bm\theta^\top(\Phi(\mathbf X)-\Phi(\widehat{\mathbf X}))
\right|
\\
&\le
\|\bm\theta\|
\,
\mathbb E\|
\Phi(\mathbf X)-\Phi(\widehat{\mathbf X})
\|
\\
&\le
rL_\Phi
\,
\mathbb E\|
\mathbf X-\widehat{\mathbf X}
\|_{1\text{-var}}
\\
&=
rL_\Phi\ \! \varepsilon(\widehat{\mathbf X}).
\end{align*}
Taking the supremum over \(B_{p,r}\) completes the proof.
\end{proof}

\subsection{Feature Stability for Different Methods}

Theorem~\ref{thm:errorprop} shows that the effect of irregular sampling is
determined by two quantities:
\[
\boxed{
\text{Risk perturbation}
=
(\text{feature stability } L_\Phi)
\times
(\text{reconstruction error } \varepsilon).
}
\]

We now compare \(L_\Phi\) for different feature extraction methods.

\begin{corollary}[Application to \textsc{PSLR}]
\label{cor:signature}
For \textsc{PSLR}, \(\Phi(\mathbf X)=S_p(\mathbf X)\), the truncated path
signature. The signature stability theorem
\citep{lyons2007rough,friz2010multidimensional} implies that
\[
\|
S_p(\mathbf X)-S_p(\mathbf Y)
\|
\le
C_p
\|
\mathbf X-\mathbf Y
\|_{1\text{-var}},
\]
where \(C_p\) depends only on the truncation order, the path dimension,
and the variation bound. Therefore, Theorem~\ref{thm:errorprop} holds with
\(L_\Phi=C_p\). In particular,
\[
\sup_{\bm\theta\in B_{p,r}}
\left|
\mathcal R_p(\bm\theta)
-
\widehat{\mathcal R}_p(\bm\theta)
\right|
\le
rC_p \cdot \varepsilon(\widehat{\mathbf X}).
\]
\end{corollary}

\begin{remark}[Feature stability of classical methods]
\label{rem:classical_stability}
For fixed basis-expansion methods, the feature map takes the form
\[
\Phi_{\text{basis}}(\mathbf X)
=
\left(
\int \mathbf X(t)\phi_1(t)\,dt,
\ldots,
\int \mathbf X(t)\phi_K(t)\,dt
\right)^\top,
\]
for Fourier or B-spline bases. In this case,
\(\Phi_{\text{basis}}\) is linear and its Lipschitz constant is given by
\( L_{\Phi_{\text{basis}}} = \|\Phi_{\text{basis}}\|_{\mathrm{op}}.\) For data-adaptive methods such as FPCA, however, the feature map
depends on estimated eigenfunctions and therefore must itself be learned
from the data, introducing an additional source of variability. Consequently,
the effective stability constant depends not only on the truncation level
\(K\), but also on the estimation accuracy of the empirical eigenfunctions
and hence on the sampling design. Consequently, the stability properties of basis-expansion methods are generally tied to the chosen basis, truncation level, and estimation procedure, rather than admitting a universal
sampling-design-independent stability bound. This stands in contrast to the signature transform, whose stability bound is guaranteed by rough path theory and depends only on intrinsic properties of the transform, rather than on the sampling design or the estimation procedure.
\end{remark}

\subsection{Reconstruction Error Analysis}

While the feature stability constant \(L_\Phi\) differs across methods,
an equally important distinction lies in the reconstruction error
\(\varepsilon(\widehat{\mathbf X})\). We now analyze this error for the
two reconstruction strategies.

\subsubsection{Reconstruction Error for \textsc{PSLR}: Linear Interpolation}

Let the observed grid be \(\mathcal G=\{0=t_0<\dots<t_m=1\}\) with mesh
\(\Delta = \max_k (t_{k+1}-t_k)\). Let \(\widehat{\mathbf X}_{\text{interp}}\)
denote the piecewise linear interpolant of \(\mathbf X\) on this grid.

\begin{lemma}[Convergence of linear interpolation in \(1\)-variation]
\label{lem:interp}
Let \(\mathbf X\in AC([0,1])\cap BV([0,1])\). Then
\[
\varepsilon_{\text{PSLR}}(\Delta)
:=
\|\mathbf X-\widehat{\mathbf X}_{\text{interp}}\|_{1\text{-var}}
\longrightarrow 0
\quad\text{as } \Delta \to 0.
\]
\end{lemma}

\begin{proof}
For absolutely continuous paths (which include our simulation setting,
as shown below), let \(\mathbf X'\in L^1\) denote the derivative. The
derivative of the linear interpolant on each interval
\(I_k=[t_k,t_{k+1}]\) is constant:
\[
\widehat{\mathbf X}'_{\text{interp}}(t)
=
\frac{\mathbf X(t_{k+1})-\mathbf X(t_k)}{t_{k+1}-t_k}
=
\frac{1}{|I_k|}\int_{I_k} \mathbf X'(s)\,ds
\quad\text{for }t\in I_k.
\]
As the mesh size \(\Delta\to0\), the piecewise averages
\(\widehat{\mathbf X}'_{\mathrm{interp}}\) form the local average
approximation of \(\mathbf X'\). By the Lebesgue differentiation theorem,
together with the standard \(L^1\)-convergence result for local averages, \(\widehat{\mathbf X}'_{\mathrm{interp}}\longrightarrow\mathbf X'\ \text{in }L^1([0,1]).\)

Since
\[
\|\mathbf X-\widehat{\mathbf X}_{\text{interp}}\|_{1\text{-var}}
=
\int_0^1
\left|
\mathbf X'(t)-\widehat{\mathbf X}'_{\text{interp}}(t)
\right|
dt,
\]
it follows that \(\varepsilon_{\text{PSLR}}(\Delta) \to 0\).

For general bounded variation paths, piecewise affine approximation results in the strict \(BV\) topology provide a related approximation theory; see \cite{kristensen2016piecewise}.
\end{proof}

Crucially, this error depends only on the sampling grid and vanishes as
the grid refines. It involves no tuning parameters (knots, bandwidths,
basis dimension) and imposes no global structural assumption beyond
continuity and bounded variation.

\subsubsection{Reconstruction Error for Basis-Expansion Methods}

Classical methods reconstruct the path by projecting onto a
finite-dimensional basis \(\{\phi_k\}_{k=1}^K\). Let \(\Pi_K\) denote
the projection (or smoothing operator) onto this basis. The reconstructed
path is \(\widehat{\mathbf X}_{\text{smooth}} = \sum_{k=1}^K \hat{c}_k \phi_k\).
The reconstruction error decomposes as
\begin{align*}
\varepsilon_{\text{smooth}}(K)
&=
\|\mathbf X-\widehat{\mathbf X}_{\text{smooth}}\|_{1\text{-var}}
\\
&\le
\underbrace{\|\mathbf X-\Pi_K\mathbf X\|_{1\text{-var}}}_{\text{Approximation Bias}}
+
\underbrace{\|\Pi_K\mathbf X-\widehat{\Pi_K\mathbf X}\|_{1\text{-var}}}_{\text{Estimation Error}}.
\end{align*}

It is important to note that this decomposition assumes the basis functions
\(\{\phi_k\}\) are fixed a priori. For data-adaptive bases such as FPCA,
where both the eigenfunctions and the scores are estimated from the same
data, the reconstructed path takes the form
\(\widehat{\mathbf X}_{\text{smooth}} = \sum_{k=1}^K \hat{c}_k \hat{\phi}_k\).
In that case, the true reconstruction error includes an additional term
due to the estimation of the basis itself:
\[
\|\mathbf X-\widehat{\mathbf X}_{\text{smooth}}\|_{1\text{-var}}
\le
\|\mathbf X-\Pi_K\mathbf X\|_{1\text{-var}}
+
\|\Pi_K\mathbf X-\tilde{\Pi}_K\mathbf X\|_{1\text{-var}}
+
\|\tilde{\Pi}_K\mathbf X-\widehat{\mathbf X}_{\text{smooth}}\|_{1\text{-var}},
\]
where \(\tilde{\Pi}_K\) denotes projection onto the estimated basis.
The middle term captures the error from estimating the eigenfunctions,
which can be substantial under irregular sampling and may not vanish even
as the grid refines, since it depends on the quality of the covariance
estimation. Consequently, the decomposition used above for fixed bases
underestimates the total reconstruction error for adaptive methods.

Unlike \(\varepsilon_{\text{PSLR}}\), this error \emph{does not vanish}
as \(\Delta \to 0\) for a fixed \(K\). The approximation bias
\(\|\mathbf X-\Pi_K\mathbf X\|_{1\text{-var}}\) is strictly positive
unless \(\mathbf X\) lies exactly in the chosen finite-dimensional
subspace. Even if \(\mathbf X\) is smooth, this bias is of order
\(\mc O(K^{-\alpha})\) for some \(\alpha>0\), but it never reaches zero
unless \(K\to\infty\), which is infeasible with finite samples.
Furthermore, the estimation error may be amplified by the ill-conditioning
of the basis (e.g., for FPCA, variance is inflated by \(1/\sqrt{\lambda_k}\)
for small eigenvalues), making it highly sensitive to the sampling design
and smoothing parameters.

A further subtlety is that the approximation bias measured in the
\(1\)-variation norm does not necessarily decrease monotonically with \(K\)
for all bases. For Fourier expansions, truncating high frequencies can
introduce Gibbs phenomena that increase the total variation, so
\(\|\mathbf X-\Pi_K\mathbf X\|_{1\text{-var}}\) may not be small even for
smooth \(\mathbf X\) unless the basis is specifically designed to control
variation (e.g., B-splines with bounded derivatives). Thus, the common
\(\mc O(K^{-\alpha})\) rate typically holds in \(\ell_2\) or supremum norms, but
its extension to the variation norm requires additional assumptions on the
basis and the path. This stands in contrast to the piecewise linear
interpolation used by \textsc{PSLR}, whose convergence in \(1\)-variation
is guaranteed for any absolutely continuous path without any structural
assumptions on the basis.

\subsection{Comparison and Connection to Simulations}
\label{app:cc}

Combining the feature stability analysis with the reconstruction error
analysis, the risk perturbation for each method is:
\begin{align*}
\text{PSLR:} & \quad rC_p \cdot \varepsilon_{\text{PSLR}}(\Delta),
\quad \text{where } \varepsilon_{\text{PSLR}}(\Delta) \to 0 \text{ as } \Delta \to 0; \\
\text{Fixed basis methods:} & \quad rL_{\text{basis}} \cdot \varepsilon_{\text{smooth}}(K),
\quad \text{where } \varepsilon_{\text{smooth}}(K) \not\to 0 \text{ as } \Delta \to 0.\\
 \text{FPCA:} & \quad rL_{\text{FPCA}} \cdot \bigl( \varepsilon_{\text{smooth}}(K) + \varepsilon_{\text{basis-est}} \bigr), \quad \text{where both } \varepsilon \not\to 0 \text{ as } \Delta \to 0.
\end{align*}

Thus, \textsc{PSLR} enjoys two simultaneous advantages:
\begin{enumerate}
\item Its feature stability constant \(C_p\) is guaranteed by rough path
      theory and is independent of sampling design;
\item Its reconstruction error \(\varepsilon_{\text{PSLR}}\) vanishes under
      grid refinement, imposing no persistent structural bias.
\end{enumerate}

In contrast, fixed basis methods incur an irreducible approximation bias in the \(1\)-variation norm (which may not decrease monotonically with \(K\)), while adaptive methods like FPCA add eigenfunction estimation errors, making them more sensitive to sampling. In our simulation studies (Section~\ref{sec:expe}, Scenario~3), the functional data are generated as
\[
X_i^j(t) = f_j(t) + N_{i,j}^{(\varepsilon)}(t),
\]
where \(N_{i,j}^{(\varepsilon)}(t)\) is defined by convolution with a
Gaussian kernel:
\[
N_{i,j}^{(\varepsilon)}(t) = \int_0^1 N_{i,j}^{\text{raw}}(s)\,\phi_\varepsilon(t-s)\,ds.
\]
Since convolution with a \(C^\infty\) Gaussian kernel yields
\(C^\infty([0,1])\) sample paths, each trajectory is absolutely
continuous and hence belongs to \(AC([0,1])\cap BV([0,1])\), precisely
satisfying the conditions of Lemma~\ref{lem:interp}. Consequently, for
\textsc{PSLR}, the interpolation error \(\varepsilon_{\text{PSLR}}\)
decays to zero as the sampling grid refines.

For classical basis-expansion methods, however, the same data are projected
onto a fixed finite-dimensional basis (with \(K\) chosen by cross-validation
or heuristic rules). The approximation bias
\(\|\mathbf X-\Pi_K\mathbf X\|_{1\text{-var}}\) remains strictly
positive regardless of the sampling density, for any fixed
\(K\) that is feasible with finite samples. (If FPCA were used instead, the additional eigenfunction estimation error would further degrade performance, as discussed above.) This explains the empirical finding in Section~\ref{sec:expe} that \textsc{PSLR} maintains stable classification performance under missing observations and uneven sampling, whereas classical methods exhibit significant performance degradation: \textsc{PSLR}'s error is driven solely by the sampling grid and vanishes with refinement, while basis methods suffer from an irreducible structural bias that does not diminish with increased sampling frequency.

%\clearpage
\section{More figures in Experiment}
\label{app:expe}
\noindent
%{\bf More figures in Experiment}. 
\begin{figure}[htpb]
    \centering
    \includegraphics[width=1\linewidth]{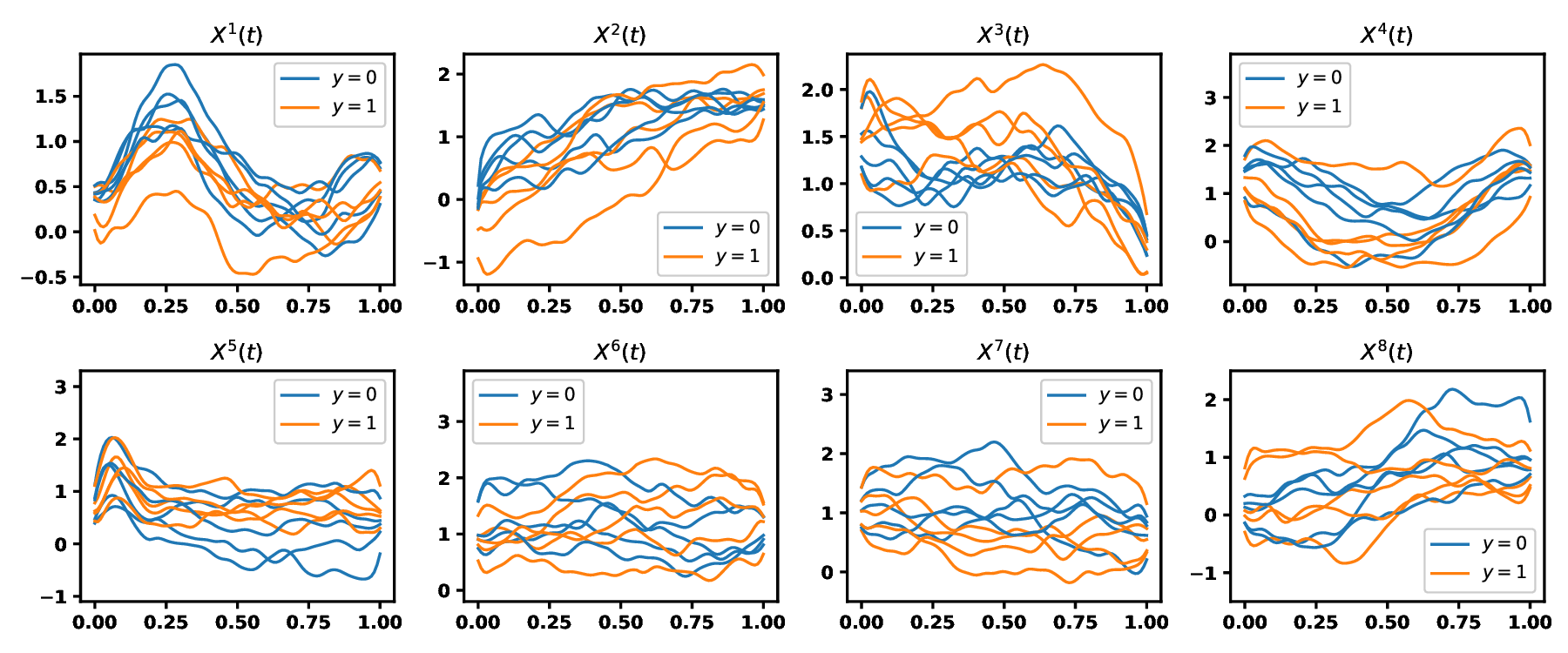}
    \caption{Simulated 8-dimensional functional data: five representative curves per class}
    \label{fig:data_sim}
    %\vspace{-0in}
\end{figure}

\begin{figure}[htpb]
    \centering
    \includegraphics[width=1\linewidth]{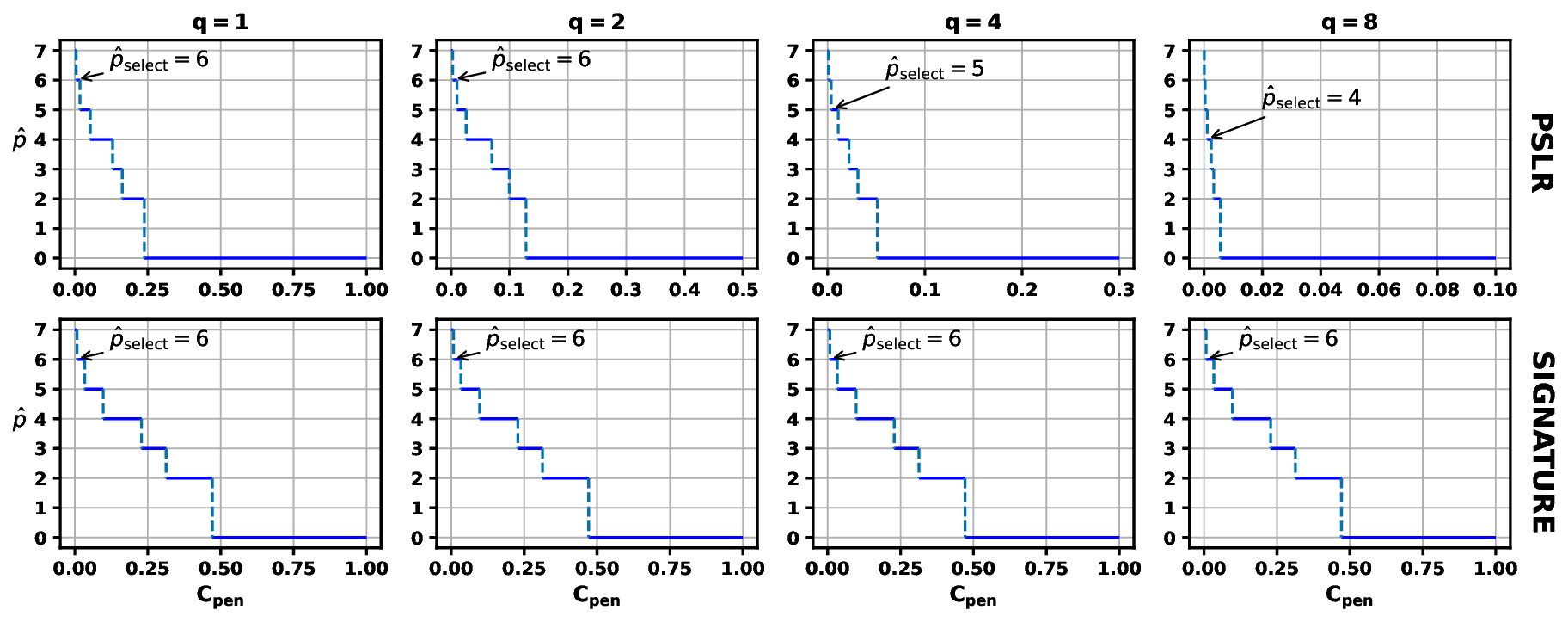}
    \caption{Truncation order selection for the PSLR and Signature methods across numbers of scalar covariates $q \in \{1,2,4,8\}$ with fixed dimension $d = 3$ (Scenario 2). Results are shown for one representative dataset per type (out of 50 instances).}
    \label{fig:order_varq}
    %\vspace{-0.3in}
\end{figure}

\begin{figure}[htpb]
    %\vspace{-0.3in}
    \centering
    \includegraphics[width=1\linewidth]{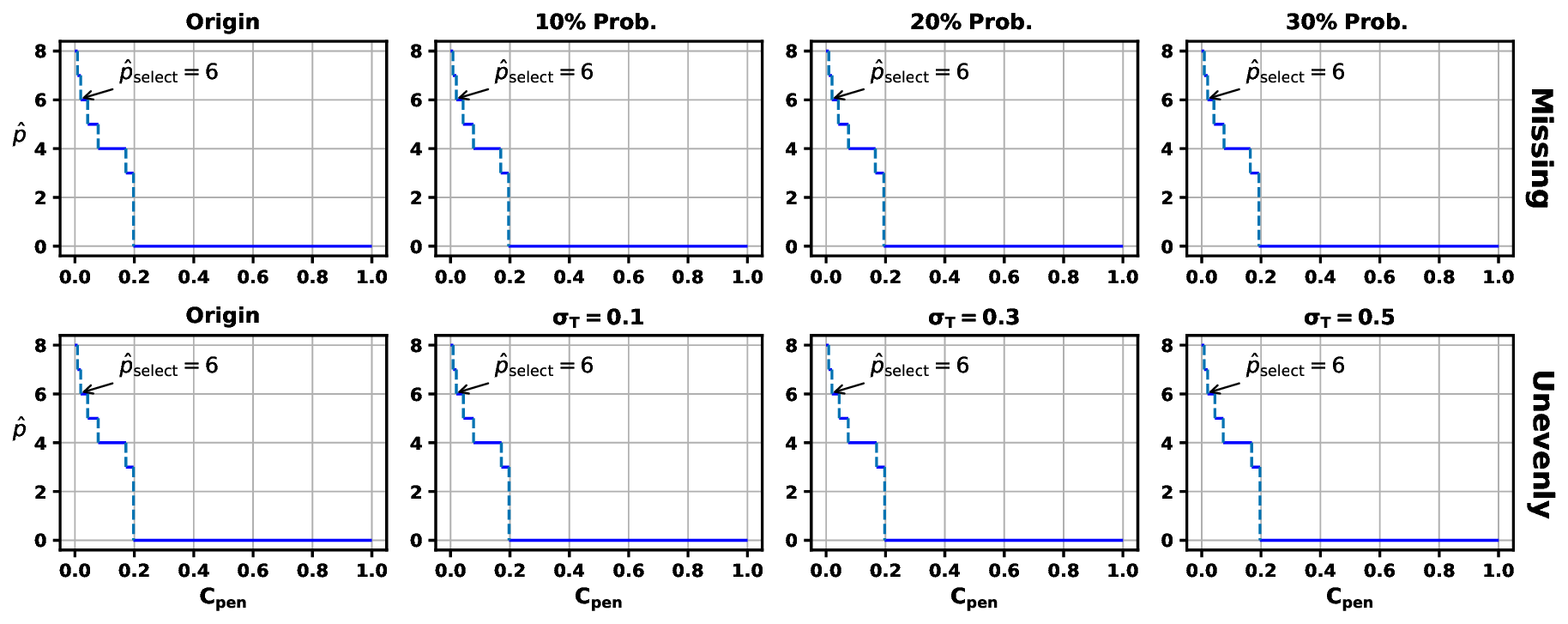}
    %\vspace{-0.3in}
    \caption{Truncated order selection for the PSLR model across irregularly sampled simulated dataset (Scenario 3).}
    \label{fig:result_timep}
    %\vspace{-0.3in}
\end{figure}

\begin{figure}[htpb]
    %\vspace{-0.5in}
    \centering
    \includegraphics[width=1\linewidth]{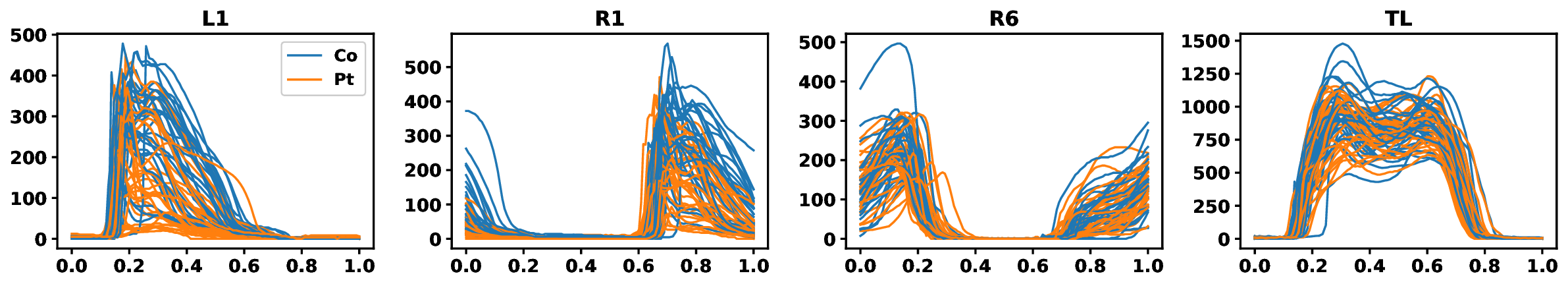}
    %\vspace{-0.25in}
    \caption{The processed functional observations from Gait in Parkinson's Disease Database across 4 signals (L1, R1, R6 and TL) with 2 classes (Co and Pt).}
    \label{fig:data_park}
    %\vspace{-0.3in}
\end{figure}

\begin{figure}[htpb]
    %\vspace{-0.5in}
    \centering
    \includegraphics[width=0.9\linewidth]{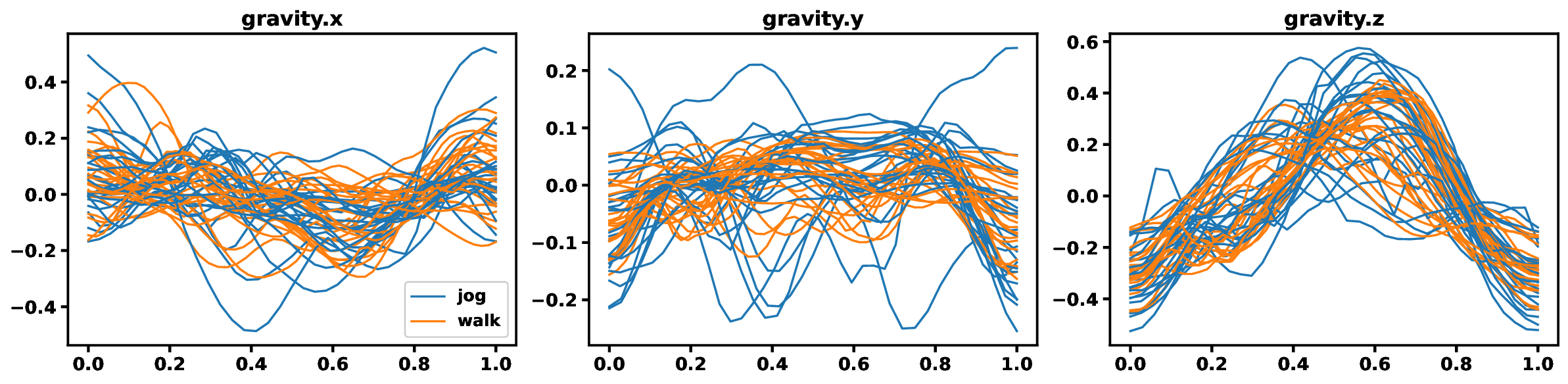}
    %\vspace{-0.05in}
    \caption{The processed functional observations from Motion Sense Dataset across 3 signals (Gx, Gy and Gz) with 2 classes (walking and jogging). }
    \label{fig:data_sensor}
    %\vspace{-0.3in}
\end{figure}

\begin{figure}[htpb]
    \centering
    \includegraphics[width=1\linewidth]{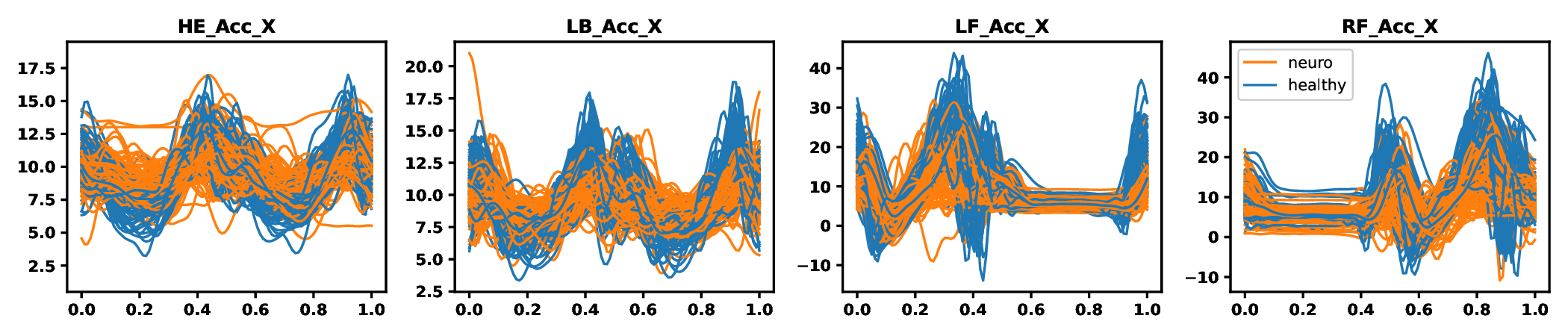}
    %\vspace{-0.25in}
    \caption{The processed functional observations from Multi-Cohort Clinical Gait Dataset across 4 signals (Acc\_X from HE, LB, LF and RF) with 2 classes (healthy and neurologic).}
    \label{fig:data_gait}
    %\vspace{-0.3in}
\end{figure}

\begin{figure}[ht]
    %\vspace{-0.5in}
    \centering
    \begin{subfigure}[b]{0.48\textwidth}
        \centering
        \includegraphics[width=\textwidth]{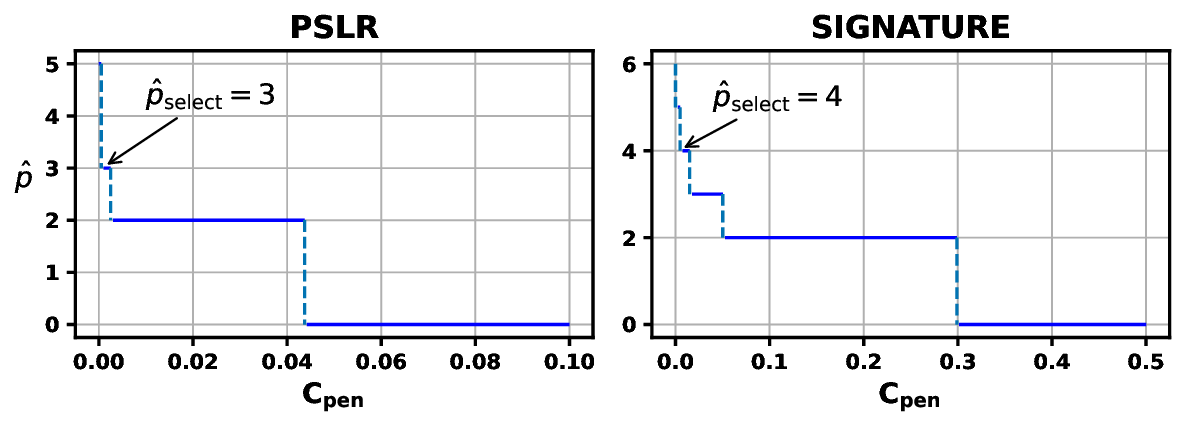}
        %\vspace{-0.28in}
        \caption{}
        %\label{fig:result_park}
    \end{subfigure}
    \hfill
    \begin{subfigure}[b]{0.48\textwidth}
        \centering
        \includegraphics[width=\textwidth]{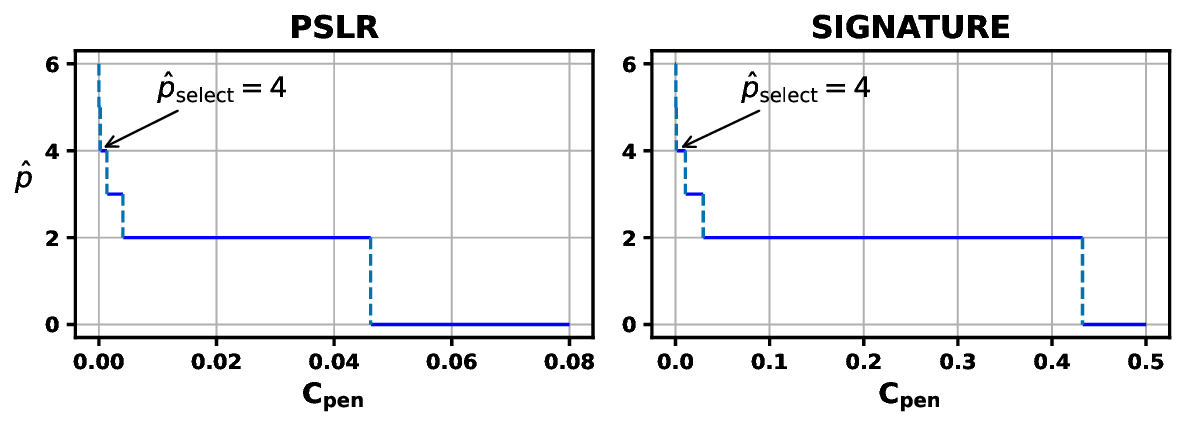}
        %\vspace{-0.28in}
        \caption{}
        %\label{fig:order_park}
    \end{subfigure}
    %\vspace{-0.1in}
    \caption{Truncated order selection for the PSLR and Signature model on one representative random split dataset from Parkinson's data \textbf{(a)} and Motion Sense data \textbf{(b)}, respectively.}
    \label{fig:order_parkmotion}
    %\vspace{-0.3in}
\end{figure}

\begin{figure}[htpb]
    \centering
    \includegraphics[width=0.5\linewidth]{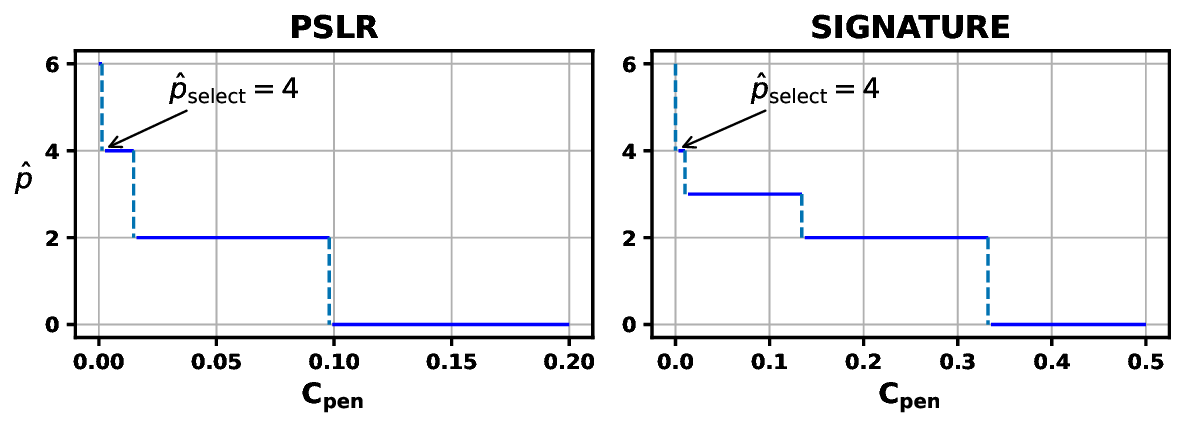}
    %\vspace{-0.05in}
    \caption{Truncated order selection for the PSLR and Signature model on one representative random split dataset from Multi-Cohort Clinical Gait Dataset}
    \label{fig:order_gait}
\end{figure}

\begin{figure}[htpb]
    \centering
    \includegraphics[width=1\linewidth]{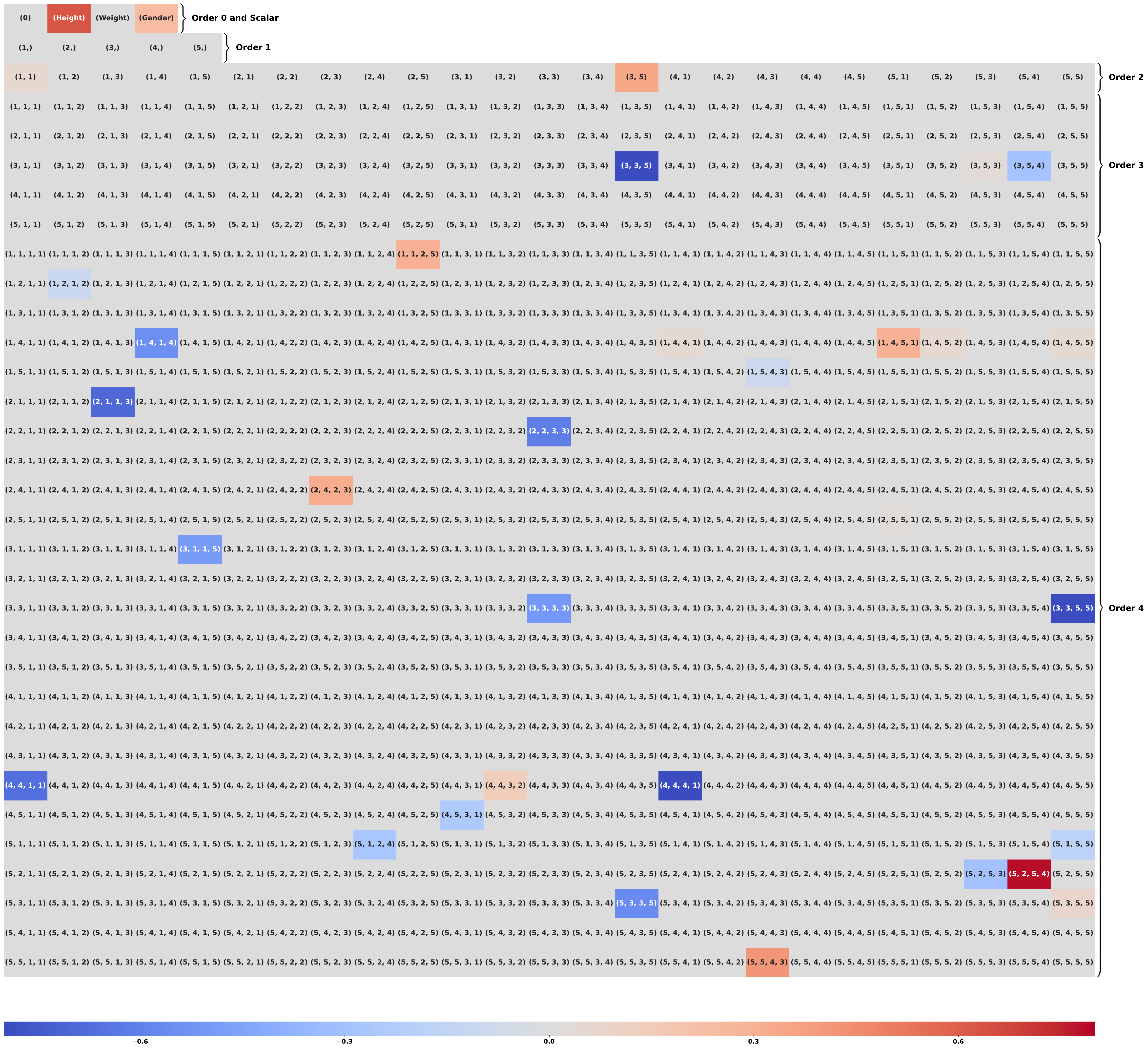}
\caption{Coefficient magnitudes from order-4 PSLR applied to Multi-Cohort Clinical Gait sensor data (HE\_Acc\_X, LB\_Acc\_X, LF\_Acc\_X, RF\_Acc\_X and time [channels 1–5]). Coefficients are organized hierarchically by signature order (vertical axis: Order 0 [intercept] = 1, Orders 1–4 = 5/25/125/625) with 3 scalar covariates aligned top-left.}
    \label{fig:Gait_coef}
\end{figure}

%\clearpage
%\newpage
%%%%%%%%%%%%%%%%%%%%%%%%%%%%%%%%%%%%%%%%%%%%%%%%%%%%%%%%%%%%
\end{document}